%% file: main.tex
\documentclass[11pt]{article} 
\input{dependencies/packages}
\input{dependencies/shortcuts}

\title{Efficient Active Auditing of Multi-Group Fairness with Bias Probes}

\author{
  \textbf{Ayoub Ajarra, Debabrota Basu}\\
  \'Equipe Scool, Univ. Lille, Inria, CNRS, Centrale Lille, UMR 9189- CRIStAL\\
  \small{\texttt{ayoub.ajarra@inria.fr}}
}

\date{}

\begin{document}

\maketitle

\begin{abstract}

Over the past decade, Machine Learning (ML) has been trained under dual objectives: minimizing prediction error via Empirical Risk Minimization (ERM) while controlling unfairness bias. In practice, however, fairness-aware training often yields limited improvements over standard ERM, making reliable post hoc auditing essential. Existing auditing approaches for black-box models either rely on model reconstruction --exposing systems to extraction attacks-- or directly estimate fairness metrics, offering limited insight into which regions of the data distribution drive bias. More fundamentally, property-specific auditing --aimed at extracting only targeted fairness information without reconstructing the model-- remains poorly understood. In this work, we introduce the bias probe framework, which enables targeted and adaptive querying to reveal bias structure while preserving model confidentiality. Building on this framework, we propose ALeBi, an active auditor that learns such probes to efficiently estimate multi-group fairness metrics. We establish novel sample complexity guarantees governed by a property-specific complexity measure, resolving a previously posed open question, and extend our analysis to adversarial settings where the model owner may strategically obscure bias. Our results uncover a fundamental trade-off between model confidentiality and reliable auditing, and show that property-specific probing enables both accurate estimation and interpretable identification of high and low-bias regions. Extensive experiments support our theoretical findings and demonstrate the practical effectiveness of our approach.
\end{abstract}


\tableofcontents

\newpage

\input{sections/introduction}

\input{sections/probdef}

\input{sections/manipfree}

\input{sections/manip}

\input{sections/experiments}

\input{sections/conclusion}

\acks{This work is supported by the Regalia project of Inria and French Ministry, ANR JCJC project REPUBLIC (ANR-22-CE23-0003-01), and
the PEPR project FOUNDRY (ANR23-PEIA-0003). D. Basu would also like to acknowledge the Inria-
ISI Kolkata associate team SeRAI for relevant motivation and discussions. }

\bibliographystyle{apalike}
\bibliography{refs}

\newpage
\appendix
\part*{Appendix}
\input{appendix/frameworkbis}
\input{appendix/manipfree}

\input{appendix/manip}
\newpage
\input{appendix/malliscious}

\input{appendix/experiments}

\end{document}

%% file: dependencies/packages.tex
\usepackage[margin=1in]{geometry}
\usepackage{hyperref}
\hypersetup{colorlinks=true,linkcolor=blue,citecolor=blue}

\usepackage{amsthm}
\usepackage{thmtools}
\usepackage{thm-restate}
\usepackage[mathscr]{eucal}
\usepackage{bbm}

\usepackage[most]{tcolorbox}
\usepackage{xcolor}

\definecolor{boxgray}{RGB}{245,245,245}
\definecolor{titlegray}{RGB}{70,70,70}

\usepackage[most]{tcolorbox}

\definecolor{theorybg}{RGB}{247,247,247}
\definecolor{theorytitle}{RGB}{75,75,75}

\newtcolorbox{theorybox}[1]{
    enhanced,
    breakable,
    colback=theorybg,
    colframe=theorytitle,
    boxrule=0.5pt,
    sharp corners,
    title={#1},
    coltitle=white,
    fonttitle=\bfseries,
    colbacktitle=theorytitle,
    left=3mm,
    right=3mm,
    top=2mm,
    bottom=2mm,
    before skip=10pt,
    after skip=10pt,
}
\newcommand{\acks}[1]{\section*{Acknowledgments}#1}
\newcommand{\nonl}
{\renewcommand{\nl}{\let\nl\oldnl}}

\usepackage{latexsym,tikz,graphicx}
\usetikzlibrary{calc}
\usetikzlibrary{decorations.pathreplacing}
\usetikzlibrary{patterns}

\usepackage{amsfonts}
\usepackage{xcolor}
\usepackage{amsthm}

\usepackage{amsmath}
\usepackage{amssymb}
\usepackage{amsthm}
\usepackage{hyperref}
\usepackage[capitalize]{cleveref}
\hypersetup{colorlinks=true,linkcolor=blue,citecolor=blue}
\usepackage{natbib}

\usepackage{multirow}
\usepackage{mathtools}
\usepackage{graphicx}
\usepackage{url}

\usepackage{xcolor}

\newtheorem{theorem}{Theorem}[section]
\newtheorem{lemma}{Lemma}[section]
\newtheorem{example}{Example}[section]
\newtheorem{definition}{Definition}[section]

\newtheorem{remark}{Remark}[section]
\newtheorem{corollary}{Corollary}[section]

\newtheorem{claim}{Claim}[section]
\newtheorem{assumption}{Assumption}[section]

\usepackage{thm-restate}

\usepackage{dsfont}
\usepackage{bm}
\usepackage{float,lipsum,tcolorbox}
\usepackage{xcolor}

\usepackage{subcaption}
\usepackage{graphicx}

\usepackage{multirow}
\usepackage{mathtools}
\usepackage{graphicx}
\usepackage{siunitx}
\usepackage{url}
\usepackage{algorithm}
\usepackage{algpseudocode}

\usepackage{bbm}

\usepackage{latexsym,tikz,graphicx}
\usetikzlibrary{calc}
\usetikzlibrary{decorations.pathreplacing}
\usetikzlibrary{patterns}

\renewenvironment{proof}[1][]{\par\noindent{\bf Proof #1\ }}{\hfill$\blacksquare$\\[2mm]}

\usepackage{adjustbox}
\usepackage{caption}
\usepackage{array}
\usepackage{threeparttable}
\usepackage{makecell}

\usepackage[utf8]{inputenc}
\usepackage{booktabs}
\usepackage{stmaryrd}
\usepackage{pifont}
\usepackage{xltabular}

\newtcbox{\resultbox}{
  colback=gray!10,
  colframe=black,
  boxrule=0.5pt,
  arc=4pt,
  boxsep=4pt
}

\usepackage[mathscr]{eucal}

%% file: dependencies/shortcuts.tex
\newcommand{\auditor}{\mathbf{A}}

\newcommand{\modelowner}{\mathbf{M}}
\newcommand{\shwartz}{\mathrm{DS}_k}

\newcommand{\E}{\mathbb{E}}
\newcommand \prob {\mathop{{{\mathbb{P}}}}\nolimits}
\newcommand{\cA}{\mathcal{A}}

\newcommand{\cD}{\mathcal{D}}
\newcommand{\cE}{\mathcal{E}}
\newcommand{\cF}{\mathcal{F}}

\newcommand{\cH}{\mathcal{H}}
\newcommand{\cI}{\mathcal{I}}
\newcommand{\cJ}{\mathcal{J}}

\newcommand{\cO}{\mathcal{O}}
\newcommand{\cP}{\mathcal{P}}
\newcommand{\cQ}{\mathcal{Q}}

\newcommand{\cT}{\mathcal{T}}

\newcommand{\cV}{\mathcal{V}}

\newcommand{\cX}{\mathcal{X}}
\newcommand{\cY}{\mathcal{Y}}
\newcommand{\cZ}{\mathcal{Z}}

\newcommand{\bX}{\mathbf{X}}
\newcommand{\bY}{\mathbf{Y}}

\newcommand{\bx}{\mathbf{x}}
\newcommand{\bb}{\mathbf{b}}
\newcommand{\by}{\mathbf{y}}

\newcommand{\bbH}{\mathbb{H}}

\newcommand{\R}{\mathbb{R}}


%% file: sections/introduction.tex
\section{Introduction}

Assessing ML models has become increasingly critical, as mounting evidence shows that deployed systems may exhibit harmful biases across a wide range of domains, including medical imaging \citep{seyyed2021underdiagnosis,daneshjou2022disparities,vrudhula2024machine,banerjee2021reading}, hiring and job screening \citep{harwell2022face,wilson2024gender}, credit lending \citep{garcia2024algorithmic}, and online platforms \citep{biswas2020machine}. Among these concerns, discriminatory bias has received particular attention, giving rise to the field of algorithmic fairness \citep{barocas2016big,hardt2016equality}. A central line of work focuses on designing fair learning procedures, where fairness constraints are incorporated directly into the training process, typically as optimization constraints trading off predictive accuracy and fairness guarantees \citep{zafar2017fairness,agarwal2018reductions}. However, empirical evidence suggests that such approaches may yield limited improvements over standard ERM in practice \citep{zong2022medfair}. 

These limitations, combined with increasing regulatory pressure, have shifted attention toward post-deployment auditing. In the United States, regulations such as New York City’s Local Law 144 \cite{ny_bill_s8612_2024} mandate bias audits of automated decision systems. In Europe, the Digital Services Act \cite{EC_DSA_2024} and the AI Act \cite{eu_ai_act_2024,ebers2025truly} impose transparency, risk management, and auditing requirements for large-scale and high-risk ML systems in online platforms. These developments reflect a broader shift, where auditing is becoming a central requirement for deploying ML systems in high-stakes settings. In this post-deployment regime, the model is typically treated as a black box. An external auditor seeks to assess the model’s behavior without access to its internal structure, interacting with the model owner through queries. This setting introduces a fundamental tension: the auditor aims to extract sufficient information to reliably evaluate fairness, while the model owner may wish to limit disclosure in order to protect model's sensitive information.

A large body of work has studied fairness auditing under this paradigm. Existing approaches typically fall into two categories. The first relies on \emph{model reconstruction}, using tools from active learning \citep{angluin1988queries,balcan2012active} to learn the model --or an accurate surrogate-- before estimating fairness metrics. The second relies on \emph{direct estimation} of a difference of means, using labeled samples. However, both approaches exceed the auditing objective: reconstruction attempts to recover the full model, while direct estimation is as costly as replicating the learning process and do not produce explanations, making it impractical in realistic auditing scenarios. Moreover, reconstruction-based approaches may expose the model to extraction attacks, compromising confidentiality. This tension raises the following question:
\begin{quote}
\hypertarget{Q1}{\textcolor{blue}{\textbf{\textit{Q1}}}: Can we obtain distribution-free guaranties for fairness auditing that avoid both full model reconstruction and large-scale labeled data requirements?}
\end{quote}

Beyond estimation, interpretability is also a critical component of auditing. In the context of machine learning regulation, Article 86 of the AI Act \cite{eu_ai_act_2024} underscores the right to explanation in automated decision-making. For example, an individual denied a bank loan by an ML system is entitled to receive an explanation for that decision. While standard audit methods provide quantitative summaries of bias \citep{goldreich1998property,yan2022active,chugg2023auditing}, they do not yield interpretable representations of the model’s discriminatory behavior. In parallel, the literature on representation learning has introduced \emph{probes} as simple functions used to extract specific properties from learned representations \citep{alain2016understanding}. Inspired by this perspective, we introduce \emph{bias probes} -- structured comparison functionals that capture relational disparities between protected groups. Unlike feature attribution methods, which operate at the level of individual inputs, bias probes focus on intergroup comparisons and provide an interpretable summary of discriminatory behavior. To the best of our knowledge, this is the first work to formalize fairness auditing with explicit interpretability guarantees in this sense.

Finally, we must account for the security of the auditing process. Existing auditing methods, particularly those that rely on membership queries \citep{yan2022active}, may render the model vulnerable to extraction attacks, effectively reducing auditing to model reconstruction. This highlights a second fundamental challenge: ensuring that auditing procedures do not compromise model confidentiality. This leads to our second question:
\begin{quote}
\hypertarget{Q2}{\textcolor{blue}{\textbf{\textit{Q2}}}:  Can we design sample-efficient and interpretable auditing procedures while preserving model confidentiality against extraction attacks?}
\end{quote}

Beyond the accuracy and interpretability of the audit report, we further extend our setting by considering a model owner who may behave adversarially, with the goal of concealing unfairness across groups. The adversarial behavior we consider differs from the adversarial models commonly studied in the literature, such as Huber contamination models~\cite{chen2016general}, adversarial distribution shifts~\cite{croce2020robustbench}, or adversarial input perturbations~\cite{montasser2021adversarially,goodfellow2018making}. In these settings, the adversary typically manipulates the data or the input distribution, whereas our adversary is \emph{fairness-aware}: the model owner may strategically manipulate the model's behavior to make unfairness harder to detect by the auditor.

This leads to our second research question:

\begin{quote}
\hypertarget{Q3}{\textcolor{blue}{\textbf{\textit{Q3}}}:  How can we design fairness auditors that are accurate and interpretable while remaining robust to fairness-aware adversaries?}
\end{quote}


\begin{figure}
    \centering
    \includegraphics[width=0.4\linewidth]{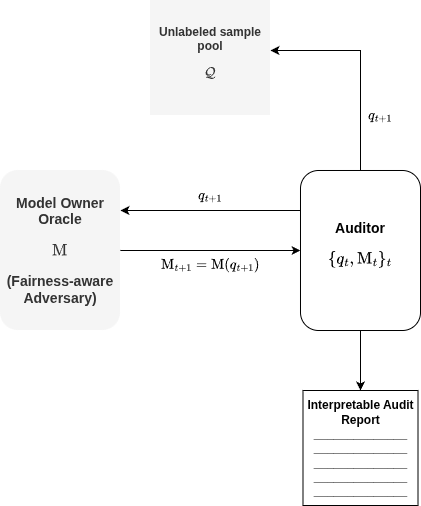}
    \caption{Interactive fairness auditing against a fairness-aware adversary.}
    \label{fig:protocol}
\end{figure}

In Figure \ref{fig:protocol}, the auditor is given the task to audit a property of interest $\mu$, has access to a pool of unlabeled instances $\cQ$ sampled from the marginal input distribution, providing information about the regions of the input space relevant to the audit. The auditor then selectively queries the model owner with samples that are informative with respect to $\mu$. At round $t+1$, the auditor constructs a query $q_{t+1}$ and sends it to the model owner $\modelowner$, who returns a response $\modelowner_{t+1}=\modelowner(q_{t+1})$.

We assume a powerful, fairness-aware adversary that may strategically alter its responses to conceal disparities across groups. Based on the history $\{(q_1,\modelowner_1),\ldots,(q_{t+1},\modelowner_{t+1})\}$, the auditor produces an estimate $\hat{\mu}_{t+1}$ of the property under audit, together with regions of the input space in which the model exhibits larger fairness disparities.

In this work, we address these challenges in the setting of multi-group fairness. We consider an interactive auditing framework in which an auditor queries a black-box model through carefully designed inputs drawn from a pool of unlabeled data. Our approach focuses on learning comparison functionals that capture the model’s discriminative behavior across protected groups, rather than reconstructing the model itself. This perspective enables both statistical efficiency and interpretability while inherently limiting the information exposed to the auditor, thereby protecting the model owner from reconstruction attacks.

\subsection{Related Work}

\paragraph{Fairness Auditing.} Algorithmic fairness has been widely studied from both the learning and auditing perspectives. A large body of work focuses on learning fair models via constrained optimization or reduction-based approaches \citep{zafar2017fairness,agarwal2018reductions}. In addition, post hoc audit methods aim to assess the fairness properties of a model deployed under limited access \citep{kearns2018preventing,yan2022active,chugg2023auditing}. Recent works have proposed active auditing strategies for estimating fairness metrics under distributional assumptions \citep{yan2022active}. These approaches often reduce auditing to model reconstruction within a hypothesis class, leading to guarantees that scale with the complexity of learning.  Other related works propose frameworks, where bias audit reduces to a binary objective via testing approaches \citep{chugg2023auditing,goldreich1998property}.  In contrast, our work treats fairness auditing as a property estimation problem, focusing on the intrinsic combinatorial structure of intergroup comparisons rather than recovering the full model.

\paragraph{Interpretability and Probing Methods.}
Post hoc interpretability methods, including feature attribution techniques \citep{ribeiro2016should,lundberg2017unified} and counterfactual explanations \citep{kusner2017counterfactual}, aim to provide descriptive summaries of model behavior. In parallel, probing methods have been developed in representation learning to extract specific properties from learned representations using simple models \citep{vafa2025has,alain2016understanding} or alternatively explaining biases in Large Language Models \cite{immer2022probing,morehouse2025rethinking,guo2022auto,manerba2024social}. Our notion of bias probes is conceptually related but differs in both objective and scope. Rather than analyzing internal representations, bias probes capture relational disparities between protected groups and are directly tied to fairness properties. Moreover, our probes are equipped with statistical guarantees, bridging interpretability and rigorous auditing.

\paragraph{Active Learning and Disagreement-Based Complexity.}
Disagreement-based active learning characterizes label complexity through geometric and combinatorial quantities such as the disagreement coefficient in distribution-dependent settings \citep{balcan2006agnostic, dasgupta2007general} and the star number in distribution-free regimes \citep{hanneke2014theory, hanneke2015minimax}. These parameters capture the ability of a hypothesis class to isolate individual instances via localized perturbations around a reference hypothesis. Our analysis is inspired by this framework but departs from it in a fundamental way. In classical active learning, disagreement regions are subsets of the instance space. In contrast, the auditing setting we consider induces disagreement over cross-group tuples, leading to a multipartite geometry rather than a pointwise one. To capture this structure, we introduce the intergroup star number, which extends the classical notion of star number to relational settings between sub-populations. This quantity measures the capacity of a comparison class to isolate cross-group configurations, and serves as the key complexity parameter governing sample efficiency in our setting.

\paragraph{Preference Learning.}
Preference learning is closely related to our setting, as it also relies on pairwise comparisons. In its classical formulation, a collection of users provides pairwise preferences over a finite set of items, and the goal is to predict unseen preferences. Many approaches aim to estimate a score matrix of size (number of users) $\times$ (number of items), often under low-rank assumptions and solved via convex optimization \cite{park2015preference}. Several works consider active querying in this setting. For instance, \cite{ailon2012active} studies adaptive pairwise comparisons over a finite set of items and returns a total order whose disagreement is competitive with the optimal ranking. However, their guarantees concern approximation to the best ranking with bounded query complexity, rather than uniform learning over a general function class. Similarly, methods based on convex relaxations, such as SVM-rank, operate within a fixed parametric class and focus on ranking performance rather than learning comparison functionals. \cite{mao2023h} provides finite-sample guarantees linking surrogate losses to target ranking losses. However, this line of work focuses on bipartite ranking and misranking loss, rather than recovering structure from pairwise comparisons in a general interactive setting. In contrast, our approach does not aim to learn a ranking over a finite set of items. Instead, we learn a class of comparison functionals that capture the discriminative behavior of a black-box model over the input distribution. This shifts the problem from ranking to distributional estimation, where the goal is to approximate the induced comparison structure with uniform guarantees over a function class.

\paragraph{Model Extraction Attacks.} Model extraction attacks have emerged as a major threat to ML systems deployed through black-box interfaces. In these attacks, an adversary interacts with a model via query access and trains a surrogate model that replicates the functionality of the original model \citep{tramer2016stealing,orekondy2019knockoff, yuan2024data,orekondy2019prediction}. Early work demonstrated that even simple query strategies can recover models such as logistic regression, decision trees, and neural networks with high fidelity \citep{tramer2016stealing}. Subsequent work has shown that extraction is possible even without access to real data, by generating synthetic queries or by using knowledge distillation techniques \citep{orekondy2019knockoff,yuan2022attack}. From a broader perspective, model extraction raises both intellectual property and security concerns, as stolen models can be reused or further exploited for downstream attacks \citep{jagielski2020high}. In the context of large language models (LLMs), model owners with API access are highly vulnerable to model extraction attacks, posing significant security risks. In particular, \cite{liu2025model} show that any low-rank language model can be learned in polynomial time using conditional queries. We design auditing mechanisms that intentionally restrict the information revealed through queries, and show that recovering the underlying model from our bias probes is computationally hard (under standard assumptions).

\subsection{ Summary of contributions}

Our main contributions are as follows:

\begin{itemize}

\item \textbf{A comparison-based framework for fairness auditing.}
We introduce a novel framework based on \emph{bias probes}, which capture the discriminative behavior of a model through structured inter-group comparisons. We formalize fairness auditing as the problem of learning comparison functionals and establish distribution-free guarantees for estimating multi-group fairness. 

\item \textbf{A new complexity measure and optimal query-complexity bounds.}
We introduce the $k$-\emph{inter-group star number} $\mathfrak{s}_k$ (Definition~\ref{def:dualstar}), extending the classical star number~\citep{hanneke2015minimax} to the multi-group setting and addressing an open question raised by~\citep{yan2022active} in the case $k=2$, that we generalize for arbitrary to arbitrary $k$. We characterize the query complexity of auditing in terms of this parameter, proving the lower bound $\Omega \left(
\max \left\{ \frac{1-2\delta}{1-\delta} \min\left\{ \mathfrak s_k, \left\lfloor\frac1{2\epsilon}\right\rfloor \right\}, \; \left[ 1-2\epsilon-2\delta(1-\epsilon) \right]_+ \shwartz \right\} \right)$. We further propose $k$-\textsc{ALeBi}, an active auditing algorithm achieving the upper bound
$$\resizebox{\linewidth}{!}{$ \mathcal{O}\left(
\min \left\{ n,\; (e-1)\min\{\mathfrak s_k,n\} \log\frac{en}{\min\{\mathfrak s_k,n\}} + \log\frac{2}{\delta} + 1,\; \frac{\min\{\mathfrak s_k,n\}}{1\vee\log \min\{\mathfrak s_k,n\}}\, \shwartz \log\left( \frac{ en(2^k-2) }{ \shwartz } \right) \right\}
\right)$}$$ where $n= \left\lceil \frac{\shwartz +\log(4/\delta) }{ \epsilon } \right\rceil$, and $\shwartz$ is the Daniely-Shwartz dimension \citep{daniely2015multiclass} (Definition \ref{def:shwartz}).
\item \textbf{Hardness of model extraction from bias probes.}
We distinguish accurate fairness auditing from recovery of the underlying classifier and show that selecting a bounded set of relational queries sufficient to identify all candidate models is NP-complete. Empirically, accurate fairness estimates can be obtained while multiple underlying models remain observationally compatible.

\item \textbf{Auditing under a fairness-aware manipulation model.}
We introduce a fairness-aware adversary that activates based on the auditor's current disparity estimate and may adaptively corrupt individual relational coordinates in CGQ responses. Under conditionally bounded corruption with $p<1/2$, we propose \textsc{Robust $k$-ALeBi}, which achieves the same asymptotic bounds as \textsc{$k$-ALeBi}, up to additional terms that depend on the corruption level.

\item \textbf{Empirical evaluation.}
We empirically evaluate \textsc{ALeBi} along four dimensions: estimation accuracy, interpretability of the resulting audit reports, protection against model extraction attacks, and robustness to fairness-aware manipulation. We compare our approach against two baselines, \textsc{ReconAudit}, which reconstructs the underlying model before performing the audit, and \textsc{DirectAudit}, which directly estimates the fairness property. Our experiments illustrate the trade-offs between these auditing paradigms and evaluate their behavior under both honest and adversarial model-owner interactions.

\end{itemize}

%% file: sections/probdef.tex
\section{Problem Formulation: Auditing Multi-group Fairness}

Let $\cH$ denote a hypothesis class. 

The model owner $\modelowner$ trains a model $h : \cZ \subseteq \R^{p+1} \to \Tilde{\cY}$ from $\cH$ on a dataset which then deploys. We assume that the first $p$ coordinates correspond to unprotected features, while the $(p+1)$-th coordinate represents a protected attribute $A$. This attribute partitions the input space $\cZ$ into $k \geq 2$ disjoint subpopulations $\cZ_1, \ldots, \cZ_k$ (e.g., when $A$ represents gender, $k=2$). 

We denote by $\mathds{1}_{[\cdot]}$ the indicator function. Finally for $a \in \R$, $[a]_+ := \max(a,0)$.

We study group-level properties of $h$. In particular, we consider the maximum disparity of $h$ across protected groups, captured by multi-group statistical parity.

\begin{definition}[Multi-Group Fairness]\label{def:kSP}
Let $h : \cZ \to \Tilde{\cY}$ and let $\cD$ be a distribution over $\cZ$. For a pair of groups $(i,j) \in [k]^2$, define
$$\mu_{i,j}(\cD,h)
=  \underset{X \sim \cD}{\E}[h(X) \mid X \in \cZ_i]
- \underset{X \sim \cD}{\E}[h(X) \mid X \in \cZ_j] .$$
The multi-group fairness violation of $h$ is
$$\mu(\cD,h) = \max_{i,j \in [k]} \mu_{i,j}(\cD,h).$$
\end{definition}

The functional  $\mu(\cD,h)$ depends jointly on the model and the data distribution. In particular, disparities in regions of $\cZ$ with negligible probability mass under $\cD$ do not contribute substantially to $\mu$. Therefore, as illustrated in Figure \ref{fig:protocol}, estimating $\mu$ requires access to the marginal distribution over $\cZ$, which motivates an active audit setting in which information about $\cD$ is obtained through a pool of unlabeled points $\cP$.

\begin{remark}
Unlike previous auditing studies \citep{yan2022active, ajarra2026on}, which consider properties defined in terms of an absolute difference, our fairness property is signed and therefore preserves the \emph{direction} of the disparity. In particular, identifying which group is favored and which group is disadvantaged is an essential part of the audit for interpretability. This directional information must be preserved for every pair of groups, as reflected in our definition, which takes the maximum over group pairs.
\end{remark}

More generally, we consider an auditing setting in which an auditor $\auditor$ has access to a pool of unlabeled samples $\cQ$ and interacts with a model owner $\modelowner$ by issuing queries $q_1, \ldots, q_m$ in order to estimate a property $\mu$ of a model $h$. Formally, $\mu : \cP(\cZ) \times \cH \to \R$ is a functional that quantifies a distributional property of models in $\cH$. An auditing problem is defined by the tuple
$\langle \cZ, \Tilde{\cY}, \cP, \mathrm{I}_{\modelowner}(h), \mu, \ell_\mu \rangle$, where $\cZ$ and $\Tilde{\cY}$ denote the input and output spaces, $\cP$ is a class of probability measures over $\cZ$, and $\mathrm{I}_{\modelowner}(h)$ represents the information about $h$ available to the auditor. In the white-box setting, $\mathrm{I}_{\modelowner}(h) = h$, while in the black-box setting it corresponds to oracle access to $h$, possibly augmented with side information such as the hypothesis class $\cH$. The functional $\mu$ captures a property of $h$ under distributions in $\cP$, and $\ell_\mu : \R \times \R \to \R$ is a bounded loss function. Additional details about the general auditing framework are provided in Appendix~\ref{app:generalfram}. We now introduce the notion of PAC auditing in the interactive setting. We focus on auditing multi-group properties, which introduce additional challenges as reliable estimation requires uniform information across subpopulations. In particular, we study the auditing of multi-group fairness (Definition~\ref{def:kSP}). To date, several approaches have been proposed to audit statistical parity. Broadly, these methods fall into two categories:

\subsection{Audit via  Model Reconstruction}
 In this setting, the auditor has access to the hypothesis class $\cH$ and query access to the model $h$, i.e., $\mathrm{I}_{\modelowner}(h) = \{\cH, \modelowner(h)\}$. The goal is to reconstruct a surrogate model by emulating the training process within $\cH$. For instance, \cite{yan2022active} use disagreement queries to identify a subclass of $\cH$ whose elements share the same statistical parity on a given unlabeled pool. 

\newpage
\begin{theorybox}{Interaction Protocol for Reconstruction-based Audit.}
The interaction between the auditor \textsc{ReconAudit} and the model owner $\modelowner(h)$ proceeds as follows:

\begin{itemize}
\item The auditor is given the hypothesis class $\cH$ and access to the unlabeled pool $\cQ$, and a query budget $T$.

\item \textbf{Round $t=1$.} The auditor selects a query $X_1 \in\cQ\cap\mathrm{DIS}(\cH)$ and sends it to the model owner. The model owner returns the response $h(X_1)$.

\item \textbf{Round $t=2, \cdots, T$.} Based on the history of previous queries and responses, the auditor updates its current hypothesis class $V_t$ and selects a query $X_t\in\cP\cap\mathrm{DIS}(V_t)$. It sends $X_t$ to the model owner and receives the response $h(X_t)$.

\item \textbf{Reconstruction.} After $T$ rounds, the auditor selects any $\hat{h} \in V_T$ to reconstruct a surrogate model of $h$.

\item \textbf{Audit.} The auditor applies a plug-in estimator to $\hat h$ to estimate the multi-group fairness property using the history of queries collected during the interaction.

\end{itemize}
\end{theorybox}

However, such approaches rely on queries designed for learning rather than for estimating the target property. As a result, their sample complexity scales with that of learning, typically $\cO(d/\varepsilon^2 \log \frac{1}{\delta})$, where $d$ denotes a complexity measure of $\cH$ (e.g., VC dimension). This dependence not only leads to inefficiency but also raises model confidentiality concerns, as it may enable model extraction. We design our first baseline, \textsc{ReconAudit}, inspired by the reconstruction-based auditing approach.\footnote{We note that the algorithm of \cite{yan2022active} is restricted to the class of linear separators. We therefore develop a general reconstruction-based algorithm that does not rely on a specific model class.}

\begin{algorithm}[H]
\caption{\textcolor{blue}{\textsc{ReconAudit}}}
\label{alg:reconstruction-cal}
\begin{algorithmic}[1]
\Require $\mathrm{I}_{\modelowner}(h) = \{\cH, \modelowner(h)\}$, query pool $\mathcal{Q}$,
budget $T$
\State Initialize
\[
V_0=\mathcal{H}.
\]
\For{$t=1,\ldots,T$}
    \State Select
    $$X_t \in  \mathrm{DIS}(V_t) $$
    \State Query $$\modelowner_t= h(X_t).$$
    \State Update
    \[
    V_t
    =
    \left\{
    h' \in V_{t-1}:
    h'(X_t)= \modelowner_t
    \right\}.
    \]
\EndFor
    \State Reconstruct
    \[
    \widehat h \in V_T
    \]
\State \Return $\widehat h$ and
\[
\widehat \mu_T=\mu_{\cQ_{1:T}}(\widehat h).
\]
\end{algorithmic}
\end{algorithm}

\begin{remark}
We first observe that model learning does not depend on the protected groups represented in the data. In particular, \textsc{ReconAudit} does not account for group membership during the learning procedure, which may lead to higher reconstruction errors for some groups than for others.

Second, once the model has been reconstructed, the auditor effectively has white-box access to the model, and the subsequent audit can therefore employ any white-box procedure, including procedures for producing interpretable audit reports. Whether such a reconstruction-based strategy remains robust in the presence of a fairness-aware adversary, however, is not clear \emph{a priori}. We investigate this question in Section~\ref{sec:exp}.
\end{remark}

\subsection{Audit via Direct Estimation}
In this setting, the auditor samples $\cO(1/ \varepsilon^2 \log \frac{1}{\mu})$ from each group in the pool $\cP$, and perform membership queries \footnote{ or alternatively has access to a random query oracle}. The objective is to estimate the target property $\mu$ directly from these samples, without reconstructing the model. While the sample complexity still scales as $\cO(k/\varepsilon^2 \log \frac{1}{\mu})$, it no longer depends on the complexity of $\cH$. This avoids the need for uniform approximation over the hypothesis class and mitigates risks associated with model reconstruction.

\begin{theorybox}{Interaction Protocol for Direct-Estimation Audit}
The interaction between the auditor \textsc{DirectAudit} and the model owner $\modelowner(h)$ proceeds as follows:

\begin{itemize}
\item The auditor is given access to an unlabeled pool $\cQ$ and a query budget $T$.

\item \textbf{Rounds $t=1,\ldots,T$.} At each round, for each group $i\in[k]$, the auditor samples points from $\cQ$ and queries the model owner. It uses the responses to estimate the expected positive-label rate for each group.

\item \textbf{Estimation.} After $T$ rounds, the auditor obtains estimates of the pairwise disparities
\[
\left\{
\hat{\mu}_{ij}: i,j\in[k],\, i<j
\right\},
\]
forming a vector of size $\frac{k(k-1)}{2} = \cO(k^2)$.

\item \textbf{Audit.} The auditor outputs the maximum estimated pairwise disparity as its estimate of the multi-group fairness property.

\end{itemize}
\end{theorybox}

\begin{remark}
The reconstruction-based audit can be viewed as a \emph{search-then-estimate} procedure: the auditor first searches for a surrogate model consistent with its observations and only then extracts the audit report from the reconstructed model. In contrast, the direct-estimation approach estimates the fairness property directly from queried responses, without first reconstructing the model.

This distinction is particularly relevant when the auditor is used to verify whether a model update changes the property under audit. In the reconstruction-based approach, the learned surrogate can be used to check whether the updated model agrees with the pre-audit model on the set of points used for the plug-in estimation. However, this set is also used during model reconstruction and has size scaling as $\cO\!\left(\frac{d}{\varepsilon^2}\log\frac{1}{\mu}\right).$
Consequently, requiring the updated model to agree with the original model on this large set of points can impose a substantially stronger constraint than requiring the fairness property itself to remain unchanged. In this sense, the reconstruction-based audit may effectively constrain the model update over a large portion of the input space, rather than merely verifying the stability of the property under audit.
\end{remark}

\begin{algorithm}[H]
\caption{\textcolor{blue}{\textsc{DirectAudit}}}
\label{alg:direct}
\begin{algorithmic}[1]
\Require $\mathrm{I}_{\modelowner}(h) = \modelowner(h)$, Query distribution $\cD$, budget $T$
\For{$i=1,\ldots,k$}
\For{$t=1,\ldots,T$}
    \State Sample $X_t^{(i)} \sim \cD_{\cdot \mid \cZ_i}$ 
\Comment{$\cD_{\cdot\mid \cZ_i}$: marginal distribution conditioned on group $i$.}
    \State Membership Query
    $\modelowner_t^{(i)}=h(X_t^{(i)})$ \Comment{Alternatively, use random queries from the joint distribution.}
\EndFor
\State Estimate model behavior on every group $$\widehat{\mu}_T^{(i)} = \frac{1}{T}\sum_{t=1}^T \modelowner_t^{(i)}.$$
\EndFor
\State Compute every pairwise component $$ \mathrm{\mu} = \Big(\widehat{\mu}_T^{(i)} - \widehat{\mu}_T^{(j)}   \Big)_{1 \le i <j \le k} $$

\State Return $$ \widehat\mu = || \mu ||_{\infty} $$
\end{algorithmic}
\end{algorithm}

\paragraph{Organization.}
In Section~\ref{sec:mainresults}, we introduce our setting for auditing multi-group fairness and formalize the intermediate step in which our proposed algorithm processes the given pool of unlabeled samples before issuing queries to the model owner. In Section~\ref{sec:manipfree}, we derive upper and lower bounds for the auditing problem in the manipulation-free regime. In Section~\ref{sec:manip}, we extend our analysis to the presence of a fairness-aware adversary and study the auditing problem under model manipulation. Finally, in Section~\ref{sec:exp}, we provide an empirical evaluation of our proposed approach and compare its performance with the two baselines discussed above, \textsc{ReconAudit}, \textsc{DirectAudit}, as well as the auditing procedure proposed by \citet{yan2022active} \footnote{The interaction protocol for \citep{yan2022active} algorithm is given in Appendix \ref{app:baselines}.}.

%% file: sections/manipfree.tex
\section{Interpretable, Manipulation-Proof and Reconstruction-Free Audits with Bias Probes}\label{sec:mainresults}

In this section, we introduce the auditing setting that will be used throughout the remainder of the paper. Rather than extracting audit-relevant information directly from the input space $\cZ$, we consider a conceptual framework in which such information is extracted over the product space of protected groups. This perspective allows us to design an auditing framework that focuses exclusively on the fairness property under audit, while providing interpretable audit information and mitigating the risk of model extraction.

\paragraph{Lifting the Sample Pool to the Product Space.}
Let $\cX \triangleq \prod_{i=1}^k \cZ_i$ denote the product space induced by the $k$ protected subpopulations. An element $\bx=(x_1,\ldots,x_k)\in\cX$ consists of one sample from each protected group, and $\cX$ serves as the query space for auditing multi-group fairness (Definition~\ref{def:kSP}). This perspective motivates a new auditing formulation in which the relevant information is represented by a class of functions defined directly on the query space, as introduced below:

\begin{definition}[Multi-Context Class]\label{def:multicontextclass}
For a hypothesis class $\cH$. The relative multi-context class induced by $\cH$ is defined as:
\[
\mathcal{F}_k(\cH)
= \left\{ f_h : \cX \to \R \;\middle|\;
f_h(\bx)
= \modelowner_k\big(h(x_1),\dots,h(x_k)\big), \; h \in \cH \right\},
\]
where $\modelowner_k$ is a fixed operator that depends on the property under audit.
\end{definition}

\begin{remark}
In fact, the operator $\modelowner_k$ performs a preliminary processing of the unlabeled pool, reorganizing and structuring the available information according to group membership for the auditing task. In our setting, this preprocessing is specific to the multi-group fairness property under audit, as illustrated in Figure~\ref{fig:setting}, hence $\modelowner_k$ is a comparison operator.  That is for every $\bx \in \cX$, $$f_h(\bx) = \left(h(x_i) - h(x_j)  \right)_{1 \le i<j \le k} $$ The resulting partition of available information is further depicted in Figure~\ref{fig:overall}, which provides an overview of our auditing setting. $h$ being a fixed model under audit, we will relax the notation and use $f$ instead of $f_h$.
\end{remark}

This perspective naturally induces an interaction model between the auditor $\auditor$ and the model owner $\modelowner$. Unlike standard membership-query access, the auditor cannot query the model on individual points or observe its absolute predictions. Instead, the auditor submits a structured query $\bX\in\cX$, consisting of one sample from each protected group. The auditor's access to the model is mediated by the oracle $\mathrm{I}_{\modelowner}(\bX) = \modelowner_k(\bX)$,  which, for our multi-group fairness property, returns only comparative evaluations across groups: $\modelowner_k(\bX) = \bigl(f_{i,j}(\bX)\bigr)_{i\neq j}$.  Thus, the auditor observes relative disparities between groups rather than the model's absolute predictions.

\begin{definition}[Cross-Group Queries (CGQ)]
Given access to a pool of unlabeled samples, a Cross-Group Query consists of selecting a tuple $\bX \in \cX$ formed by choosing one sample from each protected group. The response to $\bX$ is the collection of pairwise comparisons $\big(f_{i,j}(\bX)\big)_{i \neq j}$.
\end{definition}

\cite{kane2017active} proposed a related type of query in the context of active classification with linear separators. Their comparison queries consist of presenting a pair of points to an oracle, which indicates which point lies closer to the decision boundary. In our setting, instead, the auditor presents a pair of points from different groups to the model owner, who identifies the pair exhibiting the largest fairness disparity.

\begin{figure}
    \centering
    \includegraphics[width=0.5\linewidth]{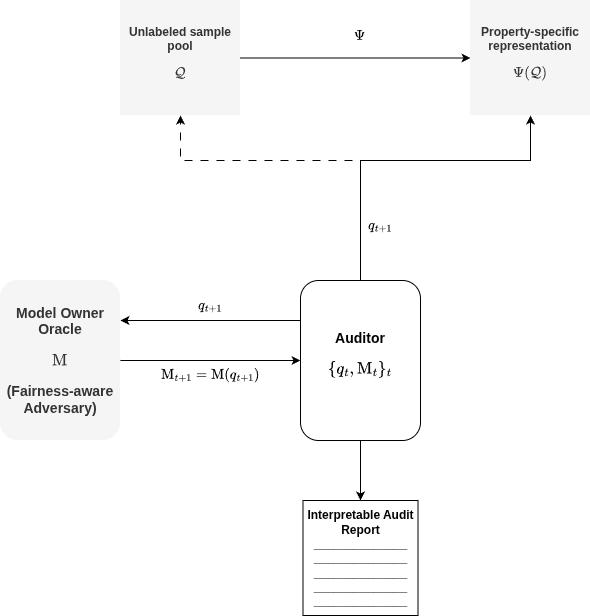}
    \caption{Interactive Fairness Auditing via Group-Based Pool Partitioning.}
    \label{fig:setting}
\end{figure}

\paragraph{Coupling and Sampling Invariance.}
We now formalize the role of the operator $\modelowner_k$ in Definition \ref{def:multicontextclass}, in other words, the role of sampling strategies over the product space from the pool $\cQ$. Let $\nu$ be any probability measure on $\cX$ whose marginals satisfy $\nu_i(\cdot) = \cD(\cdot \mid \cZ_i)$ for every $ i \in [k]$. Let $\nu_{i,j}$ denote the corresponding pairwise marginal on $\cZ_i \times \cZ_j$. The following result shows that the choice of coupling does not affect the value of $\mu$.

\begin{restatable}[Coupling Invariance]{proposition}{propcoupling}\label{prop:coupling}
Let $\mu$ denotes multi-group fairness defined in \ref{def:kSP}. 

Let $\nu\in\Pi(\nu_1,\ldots,\nu_k)$ be any coupling measure on $\cX$ whose $i$th marginal satisfies $\nu_i(\cdot)= \cD(\cdot\mid\cZ_i),$ for all $i \in [k]$. The following holds:
$$\mu(h,\mathcal D) = \left\| \mathbb E_{\mathbf X\sim\nu} [f_h(\mathbf X)] \right\|_\infty$$
\end{restatable}

This invariance implies that auditing multi-group fairness can be carried out without modeling the joint distribution across groups, and justifies the use of arbitrary coupling strategies when constructing queries from the sample pool. 

\begin{proof}
Let $\nu\in\Pi(\nu_1,\ldots,\nu_k)$ be any coupling measure on $\cX$ whose $i$th marginal satisfies 

$\nu_i(\cdot)=
\cD(\cdot\mid\cZ_i),$ for all $i \in [k]$
For $i\neq j$, let $\nu_{ij}$ denote the $(i,j)$ marginal of $\nu$
on $\mathcal Z_i\times\mathcal Z_j$.

By definition, $f_h(\mathbf X) = \bigl( h(X_i)-h(X_j) \bigr)_{1\le i<j\le k}.$

Since expectation of a finite-dimensional random vector is taken
componentwise,
\begin{align*}
\mathbb E_{\mathbf X\sim\nu}
\left[
f_h(\mathbf X)
\right]
&=
\left(
\mathbb E_{\mathbf X\sim\nu}
\left[
h(X_i)-h(X_j)
\right]
\right)_{1\le i<j\le k}
\\
&=
\left(
\mathbb E_{(X_i,X_j)\sim\nu_{ij}}
\left[
h(X_i)-h(X_j)
\right]
\right)_{1\le i<j\le k}
\\
&=
\left(
\int_{\mathcal Z_i\times\mathcal Z_j}
\bigl(
h(x_i)-h(x_j)
\bigr)
\,d\nu_{ij}(x_i,x_j)
\right)_{1\le i<j\le k}
\\
&=
\left(
\int_{\mathcal Z_i\times\mathcal Z_j}
h(x_i)\,d\nu_{ij}(x_i,x_j)
-
\int_{\mathcal Z_i\times\mathcal Z_j}
h(x_j)\,d\nu_{ij}(x_i,x_j)
\right)_{1\le i<j\le k}.
\end{align*}

Because the marginals of $\nu_{ij}$ are $\nu_i$ and $\nu_j$,
respectively,
\begin{align*}
\mathbb E_{\mathbf X\sim\nu}
\left[
f_h(\mathbf X)
\right]
&=
\left(
\int_{\mathcal Z_i}
h(x)\,d\nu_i(x)
-
\int_{\mathcal Z_j}
h(x)\,d\nu_j(x)
\right)_{1\le i<j\le k}
\\
&=
\left(
\mathbb E_{X\sim\mathcal D(\cdot\mid\mathcal Z_i)}
[h(X)]
-
\mathbb E_{X\sim\mathcal D(\cdot\mid\mathcal Z_j)}
[h(X)]
\right)_{1\le i<j\le k}.
\end{align*}

And because $\mu_{ij}(h,\mathcal D) := \mathbb E_{X\sim\mathcal D(\cdot\mid\mathcal Z_i)} [h(X)] - \mathbb E_{X\sim\mathcal D(\cdot\mid\mathcal Z_j)} [h(X)],$
we therefore obtain
$$\mathbb E_{\mathbf X\sim\nu}
[f_h(\mathbf X)]
=
\bigl(
\mu_{ij}(h,\mathcal D)
\bigr)_{1\le i<j\le k}$$

Consequently, $\left\| \mathbb E_{\mathbf X\sim\nu} [f_h(\mathbf X)] \right\|_\infty = \max_{1\le i<j\le k} |\mu_{ij}(h,\mathcal D)|$

Therefore
$$\mu(h,\mathcal D) = \left\| \mathbb E_{\mathbf X\sim\nu} [f_h(\mathbf X)] \right\|_\infty$$

\end{proof}

This reformulation shows that auditing multi-group fairness can be viewed as estimating expectations of pairwise comparison functionals under the product distribution. In particular, it shifts the objective from learning the model $h$ itself to estimating the structured family of functional defined on $\cX$. The resulting interaction model differs fundamentally from reconstruction-based approaches: rather than approximating $h$ over the hypothesis class $\cH$, the auditor directly probes the relational structure induced by $h$ across protected groups.

\begin{remark}
Although the value of the audited property is invariant to the choice of coupling, the coupling determines the distribution of the relational queries observed by the auditor and therefore the information available during the audit. We therefore treat the coupling as part of the audit design: a coupling \(\nu\in\Pi(\cD_1,\ldots,\cD_k)\) is selected before interaction and remains fixed throughout the auditing procedure, with CGQs sampled i.i.d. from \(\nu\). The auditor may adaptively decide which sampled CGQs to query, but may not adapt the coupling itself in response to previous oracle outputs. This restriction does not constrain the value of the fairness functional, which depends only on the protected-group marginals, but it fixes the statistical experiment through which that functional is learned. Allowing the coupling to change adaptively would define a strictly stronger query model in which the auditor can modify the distribution of future CGQs based on past responses, rather than merely perform active selection from a fixed data-generating distribution.

\end{remark}

\subsection{Interpretable Audits via Bias Probes}
To obtain interpretable insights into the model's behavior, we introduce bias probes, which localize the comparative behavior of $h$ across protected groups.

\begin{definition}[Bias Probes]\label{def:bias_probes}
Let $\bx = (x_1,\dots,x_k) \in \cX$. 
\begin{itemize}
    \item \textbf{Global bias probe.} The global bias probe is defined as $  f \colon \bx  \mapsto \Big(f_{i,j}(\bx)\Big)_{i \neq j}.$ 
  \item \textbf{Feature-wise bias profile.}
Let $q\in[p]$ denote an unprotected feature, 

and $P_q(\bX)=\bigl(X_1^{(q)},\ldots,X_K^{(q)}\bigr)$ the the projection of the CGQ query on the $q^{\text{th}}$ feature. For each pair $(i,j)$, we define the conditional feature-wise probe
$$\mathfrak e_{q,i,j}(z) := \mathbb{E}_{\bX\sim\mathcal D^K} \left[ f_{i,j}(\bX) \mid P_q(\bX)=z \right]$$
The vector $\mathfrak e_q(z)=\bigl(\mathfrak e_{q,i,j}(z)\bigr)_{i<j}$
describes how the expected relational response of the deployed model
varies with feature $q$.

\item \textbf{Feature-wise bias magnitude.}
We summarize the conditional relational disparity by
$$\widetilde f_q(z) := \max_{i<j} \left| \mathfrak e_{q,i,j}(z) \right| = \max_{i<j} \left| \mathbb{E} \left[ f_{i,j}(\bX) \mid P_q(\bX)=z \right] \right|$$
\end{itemize}
\end{definition}

\begin{example}[Interpretable Audit Report via Unprotected Feature-wise Bias Probes]
Consider a credit-lending setting in which a bank deploys a binary classifier
$h:\mathcal X\to\{0,1\}$ to decide whether to grant a loan. Suppose that each
client is represented by three unprotected input features, with monthly income
and age corresponding to the first and second coordinates, respectively, and
let gender be the protected attribute. For simplicity, assume two protected
groups. A cross-group query is $Q=(X_1,X_2)\in\mathcal X_1\times\mathcal X_2,$
where $X_1$ and $X_2$ are sampled from the two protected groups. The associated
relational bias probe is $    f_{1,2}(Q)=h(X_1)-h(X_2)\in\{-1,0,+1\}$.
The auditor observes this relational response rather than the two individual
predictions separately.

For an unprotected feature $q$, let $P_q(Q):=\bigl(X_1^{(q)},X_2^{(q)}\bigr).$
For example, $P_1(Q)=\bigl(X_1^{(1)},X_2^{(1)}\bigr)$
contains the monthly incomes of the two queried clients, while $P_2(Q)=\bigl(X_1^{(2)},X_2^{(2)}\bigr)$
contains their ages. The feature-wise bias profile is $\mathfrak e_q(z_1,z_2) := \mathbb E\!\left[ f_{1,2}(Q) \mid P_q(Q)=(z_1,z_2) \right],$
and its unsigned magnitude is $\widetilde f_q(z_1,z_2) := \left|\mathfrak e_q(z_1,z_2)\right|.$
Thus, $\mathfrak e_q$ records the direction and magnitude of the conditional group
prediction disparity, while $\widetilde f_q$ records only its magnitude.

The resulting audit report can combine the following complementary summaries:
\begin{itemize}

    \item \textbf{Global bias probe.}
    The global statistical-parity disparity is obtained from $\mu(h) = \left| \mathbb E[f_{1,2}(Q)] \right|. $
    This provides an overall measure of the prediction-rate difference between
    the two protected groups.

    \item \textbf{Feature-wise bias profile.}
    For a feature such as age or monthly income, the function $m_q$ describes
    how the signed relational disparity varies across values of that feature.
    For example, $\mathfrak e_1(z_1,z_2)$ gives the expected relational response among
    cross-group queries whose two clients have monthly incomes near
    $(z_1,z_2)$.

    \item \textbf{Expected feature-wise bias magnitude.}
    The quantity $ \widetilde f_q(z_1,z_2) = |\mathfrak e_q(z_1,z_2)|$
    identifies regions of the projected feature space in which the conditional
    group disparity is large. For monthly income, it can therefore reveal
    income combinations for which the model exhibits particularly pronounced
    group-relative prediction differences; the same construction applies to
    age.

\end{itemize}

Importantly, these quantities provide \emph{probe-level interpretability}:
they are estimable from projected CGQ feature values and relational probe
responses. They should not be interpreted as ordinary predictive feature
importance or as causal effects of the corresponding features on the
classifier.
\end{example}

Together, these three quantities provide complementary levels of evidence: the global probe establishes whether the model exhibits discrimination overall, the feature-wise probes indicate \emph{where} disparities vary across the input space, and the expected feature-wise probes identify regions in which the disparities are most pronounced. The resulting audit report therefore provides the model owner not only with evidence of unfairness, but also with interpretable information about where and along which features that unfairness manifests itself.

Bias probes thus provide a localized view of the model's discriminative behavior, isolating how differences across groups manifest at specific feature configurations. While, the expected feature-wise bias probe represents the average intergroup comparison induced by $h$ among all inputs sharing the same configuration on the unprotected feature $q$, providing a low-dimensional summary of the model’s discriminative behavior.

We now formalize interpretable active auditing for general properties $\mu$ through the lens of comparison function learning.

\begin{definition}[Interpretable Active Auditing]\label{def:pacactiveprobe}
Let $\cP$ be a class of distributions over $\cZ$, and let $\cF$ be an intergroup context class defined over $\cX$. We say that $\cF$ is actively learnable under $\cP$ if there exists a sample complexity function $m_{\cF} : (0,1)^2 \to \mathbb{N}$, such that for all $(\varepsilon,\delta) \in (0,1)^2$, for all $\cD \in \cP$, and for all $m \geq m_{\cF}(\varepsilon,\delta)$, there exists an algorithm $\cA_m$ which, based on $m$ queries $\cQ_1,\dots,\cQ_m$, outputs an estimate $\hat f_m \in \cF$ satisfying
$$\underset{\cQ_1, \cdots \cQ_m}{\prob} \Big[
\cE_{\cD}(\hat f_m)
- \inf_{f \in \cF} \cE_{\cD}(f)
> \varepsilon
\Big] < \delta,$$

where the risk $\cE_{\cD}(f)$ is defined as $\cE_{\cD}(f)
:= \underset{\bX \sim \cD^k}{\prob}\big[f(\bX) \neq f_h(\bX)\big]$, and $f_h$ denotes the comparison functional induced by the model $h$. The minimax sample complexity of the auditing problem is defined as $\inf_{m_{\cF}} \sup_{\cD \in \cP} m_{\cF}(\varepsilon,\delta)$, 
where the infimum is taken over all valid active auditing strategies.
\end{definition}

This formulation reduces the auditing problem for $\mu$ to the problem of learning the comparison functional $f_h$, from which $\mu(\cD,h)$ can be estimated as a functional of $f_h$.

\subsection{Information Complexities For Learning Probes}

In the following section, we define the information-complexities relevant for characterizing the auditing problem of multi-group fairness. We begin with Natarajan dimension \cite{daniely2015multiclass}, that will serve to characterize global probe learning in the passive setting:

\begin{definition}[$k$-Daniely-Shwartz Dimension  \small{\citep{daniely2015multiclass}}]\label{def:shwartz}\leavevmode\\[-0.5em]
\begin{itemize}
    \item \textbf{Pseudo-cube.} Let $B\subseteq\mathcal Y^m$ be finite and non-empty. Two vectors $b,b'\in B$ are called $i$-neighbors if $b_i\neq b'_i$ and for every $j \neq i$, $b_j=b'_j$.
The set $B$ is an $m$-dimensional pseudo-cube if for every
$b\in B$ and every $i\in[m]$, $b$ has an $i$-neighbor in $B$.
\item \textbf{Daniely-Shwartz dimension.} A sequence $S=(\mathbf x_1,\ldots,\mathbf x_m)\in\mathcal X_k^m$ is DS-shattered by the $k$-context class $\cF_k$ if the set $\cF_k|_S = \{ (f(\bx_1),\ldots,f(\bx_m)): f\in\cF_k \}$
contains an $m$-dimensional pseudo-cube.

Define Daniely-Shwartz dimension $\shwartz(\cF)$
as the largest DS-shattered dimension, or $\infty$ if no finite largest
dimension exists.
\end{itemize}

\end{definition}

 For a fixed multi-context class $\cF$, the disagreement set is defined as follows:
 \newline
$\mathrm{DIS}(\cF) \triangleq \{ \bx \in \cX: \exists f, f' \in \cF: f(\bx) \neq f'(\bx) \}$. The disagreement sets arising in learning and auditing are objects of a
different nature. In standard learning, the disagreement region is a subset of
the sample space $\cZ$, reflecting uncertainty about the label of individual
instances. In contrast, for multi-group fairness auditing, the disagreement set
consists of \emph{pairs} of elements drawn from distinct protected groups,
encoding comparative uncertainty rather than pointwise ambiguity. This distinction originates from two structural differences: First, the input
space in the auditing setting carries additional organization through the
presence of a protected attribute, which induces a partition into groups.
Second, the underlying model classes differ in purpose: in learning, the hypothesis class represents candidate models themselves, whereas in auditing it represents surrogate preference functionals capturing only the
discriminative behavior of the fixed black-box model.

\begin{definition}[Multi-star Number]\label{def:dualstar}
Let $\mathcal{F}$ be a class of bias probes defined on $\cX$. 

The \emph{multi-star number} $\mathfrak{s}_k$ is the largest value of $m$ for which there exist distinct points $\{\bX_i: i \in [m]\}$, and functions $f_0 \in \cF, \quad f_i \in \cF$, such that for all $i \in [m]$, $\mathrm{DIS}(\{f_0,f_i\}) = \{\bX_i\}$, \emph{and each component $p$ of the multi-context vector $\bX$ belongs to the protected group $\cZ_p$, for $p \in [k]$}. If no finite maximum exists, we set $\mathfrak{s}_k = \infty$.
\end{definition}

Examples of multi-star number for multi-context classes is given in Appendix \ref{app:examplestar}. We prove the following result for the linear classifiers induced multi-context class.

\paragraph{Relationship between multi-star number and Daniely-Shwartz dimesnion.} The following claim establishes that the multi-star number is upper bounded by the DS dimension. Consequently, when the DS dimension is infinite, passive learning of the probe class is impossible, and hence active learning is impossible as well. Later, we characterize active learnability of probes in the interactive setting induced by CGQs in terms of both these information complexities. In particular, we show that the finiteness of the multi-star number provides a necessary and sufficient condition for learning probes in this interactive setting.

\begin{claim}\label{claim:danielystar}
    For an arbitrary multiclass $\cF$, the following holds:
    $$\mathfrak{s}_k(\cF) \le \shwartz(\cF)$$
\end{claim}

\begin{proof}
Let $S = \{ \bx_1,\bx_2, \cdots, \bx_m\}$ denotes a DS-shattered set, $B \subseteq \cF[S]$ an $m$-dimensionnal cube, and fix $b_0 \in B$ realizable by $f_0 \in \cF$ in $S$. 
For every $i$ in $[m]$, by the pseudo-cube definition, there exists an $i$-neighbor $b_i$ in $B$ of $b$. 
We have for every $j \in [m] \setminus \{i\}$, $b^i_i\neq b^0_i$ and $b^i_j=b^0_j$. Let $f_i$ in $\cF$ a probe that realizes $b_i$. That is $f_i|_S=b^i$. Since $f_0|_S=b^0$, we obtain
$$
f_i(\mathbf x_i)=b^i_i\neq b^0_i=f_0(\mathbf x_i),
$$
and for every $j \in [m] \setminus \{i\}$,
$$
f_i(\mathbf x_j)=b^i_j=b^0_j=f_0(\mathbf x_j).
$$
Then for every $i$ in $[m]$, $\mathrm{DIS}(\{f_0,f_i\}) \cap S = \{\bx_i\}$.
Hence $S$ is also a star set of size $m$. 
By taking the supremum over DS-shattered $S$, we obtain the desired property.
\end{proof}

\subsection{Information-Theoretic Separation of Reconstruction and Probe Auditing}

It is known that linear classifiers, even in low dimension input space (e.g., $d=2$), the star number is unbounded, implying that active reconstruction-based auditing cannot improve over passive reconstruction and therefore requires a large number of queries, increasing the risk of model extraction. Proposition~\ref{prop:starlinear} establishes a separation between auditing via active reconstruction and bias probing, the latter avoiding this intrinsic complexity barrier by focusing on property-specific queries.

\begin{restatable}[Information-Theoretic Separation of Reconstruction and Probe Auditing]{proposition}{starlinearclass}\label{prop:starlinear}
     There exists an audit instance problem, where auditing via model reconstruction is strictly harder than auditing via probe learning.
\end{restatable}

This establishes an information-theoretic separation between the two approaches: while reconstruction has unbounded combinatorial complexity for the underlying model class, the corresponding probe-learning problem admits a finite \(k\)-multi-star complexity. This aligns with the intuition that multi-group fairness is a global property driven primarily by the protected attribute, rather than the ambient feature space. The proof is given in Appendix~\ref{app:multistarlinear}.

\subsection{Manipulation-free Regime}\label{sec:manipfree}

In this section, we focus on the manipulation-free regime, in which the model owner is honest and does not attempt to conceal unfairness.

\begin{figure}[t]
    \centering
    \includegraphics[width=0.5\linewidth]{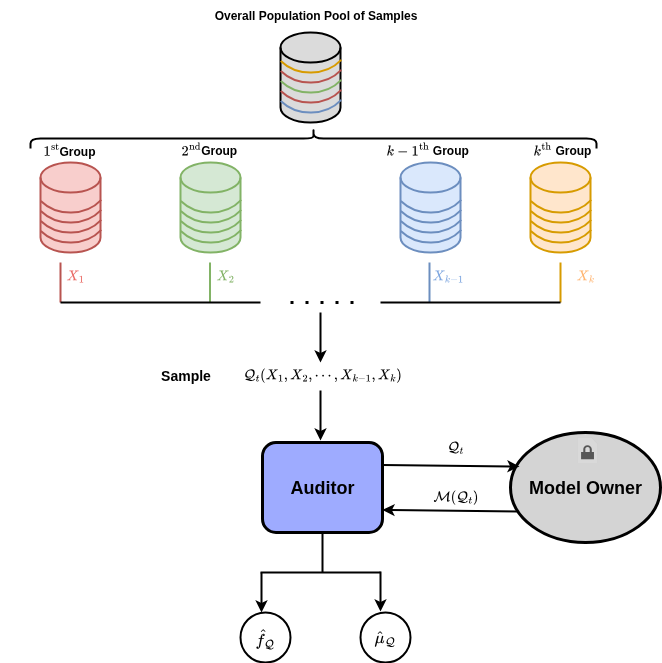}
    \caption{Interactive Fairness Auditing: Setting and Interaction Protocol}
    \label{fig:overall}
\end{figure}

\subsubsection{Hardness Results on Sample Complexity and Model Extraction}

We begin by establishing a lower bound on the sample complexity required to learn bias probes.

\begin{restatable}[Sample Complexity Lower Bound]{theorem}{scLB}\label{th:SCLB}
Fix $\varepsilon \in (0,1/4), \delta \in (0, 1/2)$. For any (possibly randomized) active auditor making at most $$m \le \Omega \left(
\max \left\{ \frac{1-2\delta}{1-\delta} \min\left\{ \mathfrak s_k, \left\lfloor\frac1{2\epsilon}\right\rfloor \right\}, \; \left[ 1-2\epsilon-2\delta(1-\epsilon) \right]_+ \shwartz \right\} \right)$$
CGQ queries, there exists a distribution $\cD$ over $\cX$ and a target comparison functional $f_{h^*} \in \cF$ such that, with probability at least $\delta$, $$\underset{\cQ_{1:m} \sim \cD^{mk}}{\prob} \left[\underset{\bX \sim \cD^k}{\prob}\big[ \hat f_{\cQ_{1:m}}(\bX) \neq f_{h}(\bX) \big] > \varepsilon \right] > \delta.$$
\end{restatable}

Theorem~\ref{th:SCLB} shows that learning bias probes is fundamentally limited by the complexity of the comparison class $\cF$. In particular, even in the active setting, the sample complexity must scale with intrinsic structural parameters such as $\mathfrak{s}_k$. This highlights that auditing through comparison functionals does not circumvent statistical hardness, but rather shifts it to the geometry of the induced comparison class. The proof is given in Appendix~\ref{app:SClb}.

\subsubsection{Hardness of Model Extraction Attacks}

In this section, we characterize the hardness of extracting the model underlying $\modelowner$ from both computational and information-theoretic perspectives.

\paragraph{Computational Hardness.}
We now show that, beyond the statistical limitations established in Proposition \ref{prop:starlinear}, extracting the underlying model from bias probes is computationally hard. Our probe-based auditing procedure induces an interactive protocol in which the auditor interacts with the model owner through at most $B$ cross-group queries (CGQs). We formalize the corresponding model-extraction problem through the following notion of \emph{$B$-Bounded Relational Extraction}, that consists of:

\begin{itemize}
    \item Finite hypothesis class $\cH_0=\{h_1,\ldots,h_m\},$
    \item Finite collection of admissible relational queries $ \cQ=\{\bx_1,\ldots,\bx_r\}
    \subseteq \cX,$
    \item The query budget of \textsc{$k$-ALeBi} $B$.
\end{itemize}

Model owner's model extraction corresponds to whether there exists a subset of queries $\cQ'\subseteq \cQ$, such that $|\cQ'|\leq B $ and such that the answers to the queries in $\cQ'$ uniquely identify the unknown hypothesis in $\cH_0$. Equivalently, the question is whether for any two different models $h$ and $g$ in $\cH_0$ there exists CGQ queries $\cQ'$ that distinguish $f_h$ from $f_g$. Thus, the goal is to make the map $ h\longmapsto
\bigl(f_h(q)\bigr)_{q\in Q'}$ injective on $\cH_0$.

\begin{restatable}[Hardness of Model Extraction via Bias Probes]{theorem}{computeLPN}\label{th:lpn}
Given a learned probe $f_h \in \cF_{\cH}$ from a (black-box) interaction with a model owner $\modelowner$ of $h \in \cH$, recovering $h$ is NP hard. 
\end{restatable}

Theorem~\ref{th:lpn} establishes a computational barrier to model extraction from bias probes. Recovering the underlying model $h$ from CGQ-based queries is intractable. More precisely, finding a small set of CGQ sufficient for universal model extraction is NP-complete. This formalizes a separation between auditing and reconstruction: while bias probes suffice to estimate properties such as multi-group fairness, they do not reveal enough information to reconstruct the model efficiently. The proofn is given in Appendix \ref{app:protextraction}.

\paragraph{Information-Theoretic Hardness.}
In this section, we quantify the information revealed about the model owner's identity through interactions with the auditor. We assume that $\cH$ is finite and that the induced multi-context class relative to $\cH$, denoted by $\cF=\cF_{\cH}$, is therefore also finite. We further assume a uniform prior over $\cH$.

Before the audit procedure begins, the uncertainty about the model identity is $\widetilde{\cI}_0 := \log_2 |\cH|$ bits. After $t$ interactions with $\modelowner$, the observed responses induce a version space $\cV_t\subseteq\cH$ containing the hypotheses consistent with all interactions observed so far. Under the uniform prior, the remaining model identity uncertainty is therefore $\widetilde{\cI}_t := \log_2 |\cV_t|$. This motivates the following definition of identity leakage.

\begin{definition}[Identity Leakage]\label{def:leake}
For a hypothesis class $\cH$, the identity leakage after $t$ interactions between the auditor and the model owner is defined as

$$
\cI_t
:= \widetilde{\cI}_0-\widetilde{\cI}_t
= \log_2\frac{|\cH|}{|\cV_t|}.
$$

\end{definition}

For example, suppose the auditor knows that the model belongs to a hypothesis class of size $|\cH|=1000$. If, after $t$ interactions with the model owner, only $|\cV_t|=20$ hypotheses remain consistent with the observed responses, then the identity leakage is

$$
\cI_t=\log_2\frac{1000}{20}\approx 5.64~\text{bits}.
$$
Thus, the audit interactions have reduced the uncertainty about the model identity by approximately $5.64$ bits.

In Section~\ref{sec:exp}, we design an experimental protocol to evaluate model extraction through identity leakage. We compare our proposed auditor with three baselines: \textsc{ReconAudit}, \textsc{DirectAudit}, and the auditing method proposed by \citep{yan2022active}.

\subsubsection{Algorithm Design: \textcolor{blue}{$k$-\textsc{ALeBi}} and Upper Bounds}

\begin{algorithm}[H]
\caption{\protect\hypertarget{kalebi}{\textcolor{blue}{\textsc{$k$-ALeBi}}}}\label{alg:alebi_k}
\begin{algorithmic}[1]

\Require
Access to $\mathrm{I}_{\modelowner}(h)=\modelowner_k(h)$,
CGQ fixed coupling distribution $\nu$,
multi-context probe class $\cF$,
pool budget $n$

\State Initialize version space $\cV_0 \gets \cF$

\For{$t=1$ to $n$}

    \State Draw independently
    \[
    \bX^{(t)}
    =
    (X_1^{(t)},\ldots,X_k^{(t)})
    \sim \nu
    \]

    \State $\cV_t \gets \cV_{t-1}$

        \If{$\bX^{(t)}\in\mathrm{DIS}(\cV_{t-1})$}

            \State Query the relational oracle
            \[
            \bY^{(t)}
            \gets
            \modelowner_k(h;\bX^{(t)})
            =
            f_h(\bX^{(t)})
            \]
            \Comment{$\bY^{(t)}\in\{-1,0,1\}^{\binom{k}{2}}$}

            \State Update the version space
            \[
            \cV_t
            \gets
            \left\{
            f\in\cV_{t-1}:
            f(\bX^{(t)})=\bY^{(t)}
            \right\}
            \]

        \EndIf

\EndFor

\State Select a global relational probe
\[
\hat f\in\cV_n
\]

\State Compute
\[
\{(f_q,\widetilde f_q)\}_{q\in[p]}
\]
for all unprotected features $q$

\State \Return
$$
\mathsf{AuditReport}
=
\left(
\underbrace{\hat f}_{\text{global bias probe}},
\underbrace{\{f_q\}_{q\in[p]}}_{\text{feature-wise bias probes}},
\underbrace{\{\widetilde f_q\}_{q\in[p]}}_
{\text{expected feature-wise bias probes}}
\right)
$$
\end{algorithmic}
\end{algorithm}

\begin{restatable}[Upper bounds of \textsc{ALeBi}]{theorem}{alebiUB}\label{prop:ManipfreeUB}
\hyperlink{kalebi}{\textcolor{blue}{{\textsc{$k$-ALeBi}}}} is a PAC active auditor: with probability at least $1-\delta$, it outputs $\hat f$ satisfying
$\prob\big\{ \hat f(Z) \neq f_{h^*}(Z) \big\}\le \epsilon$
using at most
$$\resizebox{\linewidth}{!}{$ \mathcal{O}\left(
\min \left\{ n,\; (e-1)\min\{\mathfrak s_k,n\} \log\frac{en}{\min\{\mathfrak s_k,n\}} + \log\frac{2}{\delta} + 1,\; \frac{\min\{\mathfrak s_k,n\}}{1\vee\log \min\{\mathfrak s_k,n\}}\, \shwartz \log\left( \frac{ en(2^k-2) }{ \shwartz } \right) \right\}
\right)$}$$ queries. where $n= \left\lceil \frac{\shwartz +\log(4/\delta) }{ \epsilon } \right\rceil .$
\end{restatable}

\noindent For fixed sufficiently small $\delta$, the lower bound in Theorem \ref{th:SCLB} becomes $\Omega\left(\max\{\shwartz,\min\{\mathfrak s_k,1/\epsilon\}\}\right)$,
while the upper bound becomes $\widetilde{O}\left(\min\lbrace\mathfrak s_k,\frac{\shwartz}{\epsilon}\rbrace\right)$. 
If $m$ denotes the sample complexity, and if $\epsilon\lesssim 1/{\mathfrak s_k}$, then $\mathfrak s_k \lesssim m \lesssim \mathfrak s_k \log\frac{e \shwartz}{\epsilon \mathfrak s_k},$
and the bounds match up to logarithmic factors. If
$1/\mathfrak s_k\lesssim\epsilon\lesssim1/\shwartz$, then $\frac1\epsilon\lesssim m \lesssim\frac{\shwartz}{\epsilon},
$
with a smaller logarithmic gap whenever $\epsilon \mathfrak s_k=\cO(1)$. Finally, if
$\epsilon\gtrsim1/\shwartz$, then $\shwartz \lesssim m \lesssim \shwartz/\epsilon,
$. In particular $m$ reduces to $\Theta(\shwartz)$ for constant accuracy. Thus the only potentially substantial gap occurs in the intermediate regime and is controlled by the separation between $s_k$ and $\shwartz$; in particular, if $\shwartz=O(1)$ or $\mathfrak s_k=\Theta(\shwartz)$, the upper and lower bounds match up to logarithmic factors throughout the accuracy range. Theorem~\ref{prop:ManipfreeUB} further reveals two distinct regimes. When $\mathfrak{s}_k=\infty$, the active learning component of the bound vanishes, and the complexity is governed by the passive learning term $\shwartz(\cF)$. In this regime, the auditing problem effectively reduces to passive learning of the probe class, as established by Claim~\ref{claim:danielystar}. The proof of Theorem \ref{prop:ManipfreeUB} is given in Appendix \ref{app:ubmanipfree}. The following corollary characterizes the learnability of the probe class in the interactive setting in terms of the multi-star number.

\begin{corollary}\label{corr:starchar}
    A multi-class $\cF$ is PAC actively learnable (Definition \ref{def:pacactiveprobe})
if and only if $\mathfrak{s}_{k}(\cF) < \infty$.
\end{corollary}

Corollary \ref{corr:starchar} derives from Theorems \ref{th:SCLB} and \ref{prop:ManipfreeUB}.

In Section~\ref{sec:exp}, we evaluate the accuracy of $k$-\textsc{ALeBi} in estimating multi-group fairness (Definition~\ref{def:kSP}) and compare its performance with three baselines: \textsc{ReconAudit}, \textsc{DirectAudit}, and the auditing method proposed by \citep{yan2022active}.

%% file: sections/manip.tex
\subsection{Fairness-aware Adversarial Regime}\label{sec:manip}

We now consider an adaptive adversarial model owner $\modelowner$ whose\emph{ decision to attack depends on the auditor's current estimate of multi-group disparity}. The robust version of \textsc{$k$-ALeBi} requires that each CGQ admits a sufficiently homogeneous neighborhood of nonnegligible probability mass. For any CGQ $\bX$ in $\cQ$ \footnote{Hereafter, $\cQ$ denotes the pool of unlabeled CGQs.}, $\cJ(\bX)$ denotes a cell from the input space containing $\bX$. 

\begin{assumption}[Local $\cF$ replicability]
\label{ass:local-replicability} 
There exists a measurable neighborhood map $ \bX \mapsto J(\bX)\subseteq\mathcal Q$
chosen independently of the unknown target $f^{\star}$ and before observing
any $\modelowner$ oracle response, such that every requested CGQ $\bX$  by $\auditor$ satisfies a probability mass of: $\prob(\cJ(\bX))\ge\gamma$
and $\sup_{f\in\cF}
\prob\!\left(
f(\bX')\neq f(\bX)
\,\middle|\,
\bX'\in \cJ(\bX)
\right)
\le\rho$
for some $\gamma>0$ and $0\le\rho<\frac12.$
\end{assumption}

Assumption \ref{ass:local-replicability} means that a draw from $\cJ(\bX)$ has the same clean query response as $\bX$ uniformly over $\cF$ with
probability at least $1-\rho$. The cells do not need to form a partition and may overlap. We assume that the learner can sample $\bX'\sim \cD_Q(\cdot\mid \cJ(\bX))$ from the i.i.d unlabeled pool \footnote{this can be implemented via rejection sampling.}. Since $\prob(\cJ(\bX))\ge\gamma$, each accepted local CGQ requires at most $1/\gamma$ unlabeled draws in expectation.

\begin{algorithm}[t!]
\caption{\protect\hypertarget{robustkalebi}{\textcolor{blue}{\textsc{Robust $k$-ALeBi}}}}
\label{alg:robustify}
\begin{algorithmic}[1]

\Require \textsc{$k$-ALeBi}, attack bound $p$, fixed cell map $\bX \mapsto \cJ(\bX)$

\If{$p=0$}
    \State Run subroutine \textsc{$k$-ALeBi}
    \State Whenever \textsc{$k$-ALeBi} requests CGQ $\bX_t$, query $\bX_t$
    directly and return the clean response to \textsc{$k$-ALeBi}
    \State \Return the output of \textsc{$k$-ALeBi}
\EndIf

\State Configure \textsc{$k$-ALeBi} with confidence $\delta/2$
\State $N\gets N_0(\epsilon,\delta/2)$
\State
\[
R
\gets
\left\lceil
\frac{2}{(1-2\beta)^2}
\log\frac{2N}{\delta}
\right\rceil
\]

\While{\textsc{$k$-ALeBi} has not terminated}
    \State Let $\bX_t$ be the next CGQ requested by \textsc{$k$-ALeBi}
    \State
    \[
    \widehat Y_t
    \gets
    \textsc{LocalLabel}(\bX_t,\cJ(\bX_t),R)
    \]
    \State Feed $\bY_t$ to \textsc{$k$-ALeBi} as the simulated clean
    response $f^{\star}(\bX_t)$
\EndWhile

\State \Return the output of \textsc{$k$-ALeBi}

\Statex
\Statex \rule{\linewidth}{0.4pt}
\Statex \hspace{0.5em}\textbf{Subroutine \textsc{LocalLabel}}

\Statex \rule{\linewidth}{0.4pt}

\Statex \textbf{Input:} CGQ $q$, fixed cell grid $J(q)$, local budget $R$

\For{$r=1,\ldots,R$}
    \State Sample independently
    \[
    \bX'_r\sim \mathcal D_Q(\cdot\mid \cJ(\bX))
    \]
    \State Query the attacked oracle at $\bX'_r$ and receive
    $\bY_r\in\mathcal Y_k$
\EndFor

\State \Return
\[
\widehat Y(q)
\in
\arg\max_{y\in\mathcal Y_k}
\sum_{r=1}^{R}\mathbf 1\{\bY_r=y\}
\]

\end{algorithmic}
\end{algorithm}

\subsubsection{Adaptive Fairness-aware Adversary: Algorithm and Upper Bounds}

For the $r$-th local query $\bX'_r$, let $A_r\in\{0,1\}$ denote whether its response is attacked. Let $\bbH_{r-1}$ denote the complete history
before that response.  If $A_r=0$, then $\bY_r=f^{\star}(\bX'_r).$
If $A_r=1$, the adversary may return any
$\bY_r\in\mathcal Y_k$, possibly as an adaptive function of the
history, the queried CGQ $\bX'_r$ and its clean response.

\begin{definition}[Conditional probabilistic attack]
\label{ass:attack}
There exists $p<1/2$ such that
\[
\prob\!\left(
A_r=1
\,\middle|\,
\bbH_{r-1},\bX'_r,f^{\star}(\bX'_r)
\right)
\le p
\]
for every local query. 
\end{definition}

In particular, the conditional bound is stronger than the nonadaptive assumption $\prob(A_r=1)\le p$.
We introduce the quantity $\beta
\triangleq
\rho+(1-\rho)p
=
p+\rho-p\rho$, that accounts for both sources of error: the local clean
response may differ from the response at $q$, or a locally correct response
may be corrupted. By Definition~\ref{ass:attack}, $\beta<\frac12.$

We derive upper bounds for Algorithm \ref{alg:robustify}.

\begin{restatable}[Neighborhood robustification]{theorem}{robustalebiub}
\label{thm:generic-robustification}
Let $N_0(\varepsilon,\delta_0)$ denote the upper bounds derived in Theorem \ref{prop:ManipfreeUB}. Under assumptions~\ref{ass:local-replicability} and the assumption that $p<1/2$ (Definition~\ref{ass:attack}), \hyperlink{robustkalebi}{\textcolor{blue}{{\textsc{Robust $k$-ALeBi}}}} is a (Robust) PAC-active auditor
For $p>0$, a sample complexity of $$\cO\left(N_0(\varepsilon,\delta/2)
\left\lceil
\frac{2}{(1-2\beta)^2}
\log
\frac{
2N_0(\varepsilon,\delta/2)
}{\delta}
\right\rceil\right)$$ suffices in the presence of fairness-aware adversarial model owner.

If the neighborhood samples are obtained by rejection sampling from an
i.i.d unlabeled stream, the expected number of unlabeled draws is at
most $
\left\lceil
\frac{2N_0(\epsilon,\delta/2)}{\gamma(1-2\beta)^2}
\log
\frac{
2N_0(\varepsilon,\delta/2)
}{\delta}
\right\rceil$
\end{restatable}

Theorem~\ref{thm:generic-robustification} characterizes the additional query cost incurred when the model owner is adversarial and attempts to conceal the unfairness of the model. This cost is governed by the corruption parameters, in particular the corruption level $\beta$ and the probability mass $\gamma$ of the cells that reveal the corruption mechanism. Together, these quantities determine how difficult it is for the auditor to obtain sufficiently informative observations to detect the model owner's concealment of unfairness. The proof is given in Appendix~\ref{app:robust-proofs}. 

The dependence of the query complexity on these parameters is intuitive. In particular, the sample complexity increases as the corruption parameter $\beta$ decreases: when the model owner has a greater ability to manipulate the oracle responses, the auditor requires additional queries to dilute the effect of these corruptions and recover reliable information about the property under audit. Similarly, when the probability mass $\gamma$ of the informative cells is small, the auditor encounters such cells less frequently. Consequently, the rejection-sampling procedure requires more queries to obtain a sufficient number of samples from these low-probability regions. Thus, the query overhead captures two complementary difficulties faced by the auditor: the extent to which the model owner can corrupt the observations and the rarity of the regions in which such corruption can be revealed.

For every unlabeled CGQ $\bx \in \cQ$, we associate a cell $\cJ(\bx)$, where the collection of distinct cells satisfies
$\left|{\cJ(\bx)}_{\bx \in \cQ}\right| \ll |\cQ|$.
This collection of cells is constructed before any interaction takes place between the auditor and the model owner; in particular, its construction does not depend on the current active version space or on any target responses observed during the audit. \hyperlink{robustkalebi}{\textcolor{blue}{\textsc{Robust $k$-ALeBi}}} proceeds by first running \hyperlink{kalebi}{\textcolor{blue}{\textsc{$k$-ALeBi}}} to select a query $\bx$ that induces disagreement over the probe space. It then invokes an additional sampling subroutine that independently draws multiple queries from the same cell $\cJ(\bx)$ and submits them to the model owner's oracle. The auditor subsequently aggregates the oracle responses within this cell, and uses it as the effective label associated with $\cJ(\bx)$. The rationale behind this construction is that it exploits regularities of the model class (e.g., finite VC dimension, linearity) that allow the model owner to generalize from observed data in the first place. 

Consequently, a sufficient number of CGQs within the same cell are expected to exhibit consistent model behavior.
This dilutes the effect of adversarial corruption during the interaction, and further exposes the model owner's unfairness allowing the auditor to obtain information about the target property.

\subsubsection{Global Information-theoretic Hardness}

The star number in Definition~\ref{def:dualstar} alone does not necessarily yield a property-specific lower bound for multi-group fairness. In particular, distinct probe functions may disagree substantially while inducing the same statistical-parity functional across all $\binom{k}{2}$ group pairs. To capture the information complexity that is specific to multi-group fairness, we introduce the following notion.

\begin{definition}[Fairness-effective star number]
\label{def:fair-star}
The fairness-effective star number $\mathfrak s_\mu(\cF)$ is the
supremum of all integers $m$, for which there exist distinct CGQs
$\bx_1,\ldots,\bx_m$ and functions $f_0,\ldots,f_m\in\cF$ such
that for every $j\in[m]$, $f_0(\bx_j)=\mathbf 0$ and, for every $i\in[m]$, $f_i(\bx_i)\neq\mathbf 0$, and $j\neq i$, $f_i(\bx_j)=\mathbf 0$.
\end{definition}

Definition~\ref{def:fair-star} introduces a property-specific notion of disagreement. Rather than distinguishing probes according to whether they favor or disadvantage particular group pairs, as in the standard disagreement-based formulation, we distinguish only between probes that are fair and those that are unfair on the star set. This formulation captures the information that is directly relevant to the fairness property, while abstracting away disagreements that do not affect the fairness functional of interest.

Now, we state an assumption required to derive the lower bound for robust fairness estimation with CGQs under conditional probabilistic attack. 
\begin{assumption}[Cell-replicability on the star set]
\label{ass:star-cell-replicability}
The fairness-effective star set has pairwise disjoint measurable
cells $\cJ_1,\ldots,\cJ_m$
such that, for every witness $f_i$ and every cell $\cJ_j$, $f_i$ is constant
on $\cJ_j$ and equals its value at $\bx_j$.
\end{assumption}

Assumption~\ref{ass:star-cell-replicability} captures a regularity condition on the probe induced by the black-box model $h^*$. Specifically, it assumes that the model exhibits sufficiently consistent behavior within each cell: conditional on a cell, the model's responses do not vary substantially across queries. Thus, queries belonging to the same cell tend to preserve the relevant regularities of the underlying model behavior, which enables the auditor to replicate the probe's response from multiple samples within that cell.

\begin{restatable}[Fairness lower bound under probabilistic attack]{theorem}{robustLB}
\label{thm:attacked-fair-lower}
Under Assumption~\ref{ass:star-cell-replicability} and
$p<1/2$. A set of CGQs of size $\Omega\!\left(
\min\left\{
\mathfrak s_\mu(\mathcal F),
\frac1\varepsilon
\right\}
\right)$ is necessary to learn the undelrying probe.
\end{restatable}

In Section~\ref{sec:exp}, we evaluate \textsc{Robust $k$-ALeBi} and compare its performance with \textsc{ $k$-ALeBi}, \textsc{ReconAudit} and the auditing method proposed by \cite{yan2022active} in the fairness-aware adversarial regime.

%% file: sections/experiments.tex
\section{Experimental Analysis}\label{sec:exp}

We empirically evaluate our framework along three axes:  the accuracy of learned  bias probe, multi-group fairness estimation via bias probes and the sample and computation efficiency of the proposed interactive procedure. Given a dataset $D$, we first train three classifiers $h$ ---linear model, random forest, or multi-layer perceptron (MLP). We then generate a labeled dataset by querying $h$ on unlabeled samples. Using this dataset, we perform interactive learning of bias probes associated with multi-group fairness. From the learned probes, we estimate the fairness metric $\mu(\cD,h)$ and report the estimation error along with explanations for the audit report.

 \paragraph{Datasets and experimental details.} 
We evaluate our theoretical results on three well-known benchmark datasets: \textsc{COMPAS}, \textsc{German Credit} and \textsc{Student}. 
For the \textsc{COMPAS} dataset,  introduced in the context of recidivism risk assessment \cite{angwin2022machine},  the number of groups is fixed to $k=2$, while for the \textsc{German Credit} dataset, originating from the Statlog German Credit benchmark dataset \cite{hofmann1994statlog}., we consider $k=4$ groups. 
In Appendix \ref{secapp:exp}, we further extend our analysis to synthetic datasets with multiple values of $k$. 

\paragraph{Evaluation details and reproducibility.}
As a benchmark, we compare \hyperref[alg:alebi_k]{\textsc{$k$-ALeBi}} against the two baselines, \hyperref[alg:direct]{\textsc{DirectAudit}} and \hyperref[alg:reconstruction-cal]{\textsc{ReconAudit}}. For each binary classifier, we construct a dataset and use it as a pool of unlabeled samples, which we partition into $k$ group-specific pools. \hyperref[alg:alebi_k]{\textsc{$k$-ALeBi}} uses these pools to selectively construct informative CGQ queries. Each experiment is run with 10 random seeds to ensure reproducibility.

\begin{table*}[H]
\centering
\scriptsize
\setlength{\tabcolsep}{2.6pt}
\renewcommand{\arraystretch}{1.08}
\caption{Clean-audit summary at budget $B=60$ and $H=500$. Error entries are mean $\pm$ 95\% CI over five seeds. Runtime is mean audit-only wall-clock time in milliseconds, excluding model fitting, data loading, and plotting. \textbf{Bold values indicate the best (lowest) value among directly comparable baselines within each dataset--model row; ties are all bolded.} Probe error is reported only for Probe-CAL and is therefore not ranked across methods.}
\label{tab:clean_audit_summary}
\resizebox{\textwidth}{!}{%
\begin{tabular}{llc ccc cc cc cc cc}
\toprule
& & & \multicolumn{3}{c}{Probe-CAL} & \multicolumn{2}{c}{Reconstruction-CAL} & \multicolumn{2}{c}{Direct} & \multicolumn{2}{c}{Audit-CAL} & \multicolumn{2}{c}{Finite-H CAL} \\
\cmidrule(lr){4-6}\cmidrule(lr){7-8}\cmidrule(lr){9-10}\cmidrule(lr){11-12}\cmidrule(lr){13-14}
Dataset & Groups & Model & Probe error & Fairness error & ms & Fairness error & ms & Fairness error & ms & Fairness error & ms & Fairness error & ms \\
\midrule
COMPAS & 2 & Linear & $0.0004\pm0.0008$ & $0.0165\pm0.0108$ & 101.4 & $0.0003\pm0.0006$ & 6.1 & $0.0984\pm0.0865$ & \textbf{0.5} & $\mathbf{0.0000\pm0.0000}$ & 237.2 & $\mathbf{0.0000\pm0.0000}$ & 187.4 \\
 &  & MLP & $0.0004\pm0.0008$ & $0.0178\pm0.0101$ & 240.3 & $\mathbf{0.0000\pm0.0000}$ & 6.8 & $0.1126\pm0.0949$ & \textbf{0.5} & -- & -- & -- & -- \\
 &  & RF & $0.0000\pm0.0000$ & $0.0217\pm0.0144$ & 77.6 & $\mathbf{0.0000\pm0.0000}$ & 5.6 & $0.1282\pm0.0759$ & \textbf{0.5} & -- & -- & -- & -- \\
\midrule
German Credit & 4 & Linear & $0.0000\pm0.0000$ & $0.0277\pm0.0235$ & 260.2 & $\mathbf{0.0000\pm0.0000}$ & 2.1 & $0.2258\pm0.1148$ & \textbf{0.5} & $\mathbf{0.0000\pm0.0000}$ & 39.1 & $\mathbf{0.0000\pm0.0000}$ & 26.9 \\
 &  & MLP & $0.0000\pm0.0000$ & $0.0098\pm0.0028$ & 661.7 & $\mathbf{0.0000\pm0.0000}$ & 2.0 & $0.1326\pm0.0571$ & \textbf{0.4} & -- & -- & -- & -- \\
 &  & RF & $0.0000\pm0.0000$ & $0.0285\pm0.0091$ & 37.5 & $\mathbf{0.0000\pm0.0000}$ & 1.9 & $0.1790\pm0.0570$ & \textbf{0.4} & -- & -- & -- & -- \\
\midrule
Student & 2 & Linear & $0.0000\pm0.0000$ & $0.0162\pm0.0116$ & 38.5 & $\mathbf{0.0000\pm0.0000}$ & 0.8 & $0.0414\pm0.0281$ & \textbf{0.4} & $\mathbf{0.0000\pm0.0000}$ & 12.3 & $\mathbf{0.0000\pm0.0000}$ & 10.7 \\
 &  & MLP & $0.0000\pm0.0000$ & $0.0172\pm0.0060$ & 179.7 & $\mathbf{0.0000\pm0.0000}$ & 1.0 & $0.1110\pm0.0754$ & \textbf{0.5} & -- & -- & -- & -- \\
 &  & RF & $0.0000\pm0.0000$ & $0.0227\pm0.0101$ & 14.3 & $\mathbf{0.0000\pm0.0000}$ & 0.8 & $0.0954\pm0.0499$ & \textbf{0.4} & -- & -- & -- & -- \\
\bottomrule
\end{tabular}
}
\vspace{2pt}
\begin{minipage}{0.99\textwidth}
\footnotesize\textit{Notes.} Probe error is defined only for Probe-CAL because the other baselines do not directly learn the relational probe. Audit-CAL is the finite-$H$ fairness-targeted Yan--Zhang adaptation; it and the corresponding Finite-H CAL baseline were evaluated only for linear models. The shared continuous-linear experiment is not merged into this table because it uses a different hypothesis family and budget ($B=500$). Runtime was benchmarked after the original experiments using the saved prediction matrices, so it measures auditing cost rather than training cost. Boldface compares fairness-estimation error and runtime separately within each row; it does not imply an overall ranking across accuracy and computational cost.
\end{minipage}
\end{table*}

\subsection{Accuracy}

Across the three datasets, \textsc{$k$-ALeBi} is competitive with the baselines of \textsc{DirectAudit} and \textsc{ReconAudit} in the accuracy of multi-group fairness estimation while requiring only few CGQs. As Figure \ref{fig:directbaselinesesterror} shows, after $1000$ queries \textsc{$k$-ALeBi} generally achieved substantially smaller estimation error than the the baselines, although reconstruction could achieve essentially zero error when it successfully recovered the target within the finite candidate class. 

\begin{figure}[H]
    \centering
    \includegraphics[width=0.9\linewidth]{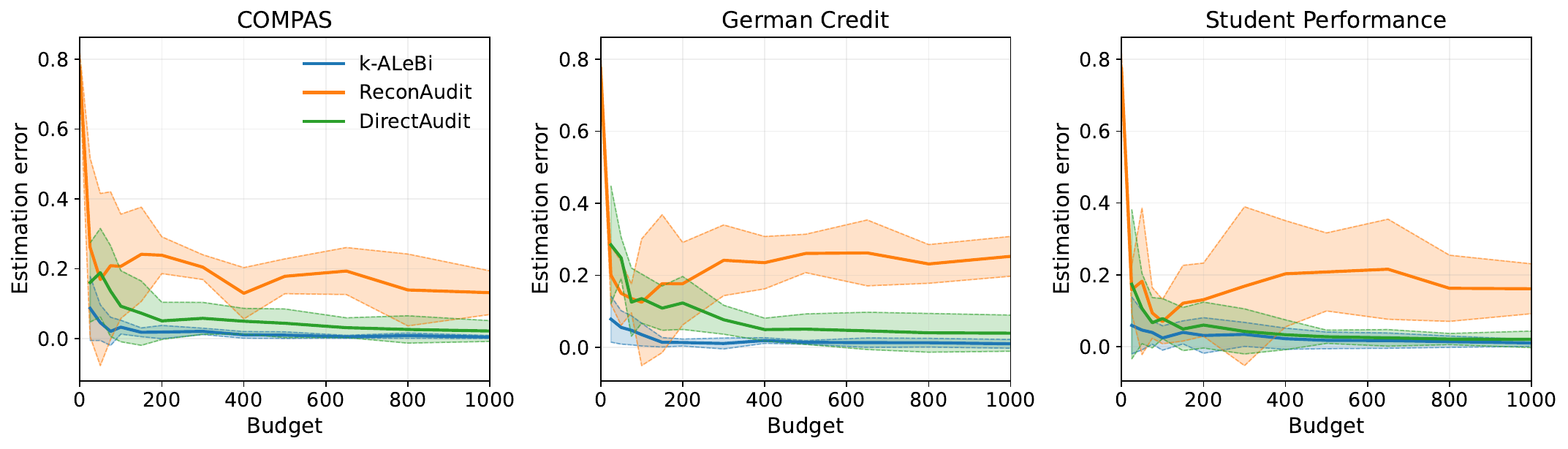}
    \caption{Estimation Error of Multi-group fairness under limitted budget.}
    \label{fig:directbaselinesesterror}
\end{figure}

We further evaluate \textsc{$k$-ALeBi} against \textsc{DiamAudit}, the auditing method proposed by \cite{yan2022active}, in its applicable setting of $k=2$ protected groups and a linear hypothesis class, corresponding to statistical parity. Figure~\ref{fig:yanbaselineesterror} shows that \textsc{$k$-ALeBi} achieves lower statistical parity estimation error than \textsc{DiamAudit} on both datasets. On COMPAS, \textsc{$k$-ALeBi} achieves an estimation error of $0.0000213$, compared with $0.000046$ for \textsc{DiamAudit}. On the Student dataset, \textsc{$k$-ALeBi} achieves an estimation error of $0.006947$.

\begin{figure}[h!]
    \centering
    \includegraphics[width=0.95\linewidth]{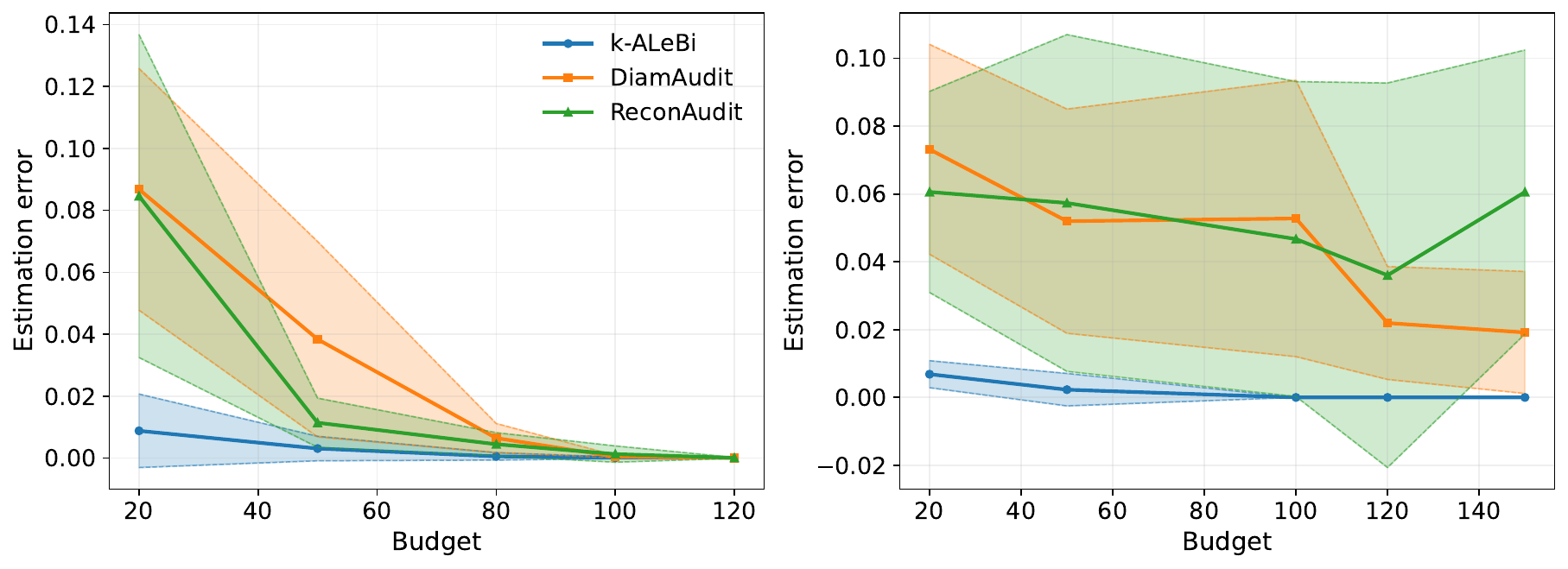}
    \caption{Estimation Error of \textsc{DiamAudit} under Its Applicable Setting}
    \label{fig:yanbaselineesterror}
\end{figure}

\subsection{Interpretability}

In addition to estimating multi-group fairness, \textsc{$k$-ALeBi} provides an interpretation of the measured unfairness that can't be obtained from a scalar fairness estimate alone. Its interpretability outputs decompose the learned probe across non-protected features and characterize how different feature configurations affect the CGQs response. For a feature $q$, let
$Z_q(Q)=(X_1^{(q)},\ldots,X_K^{(q)})$ and recall $\mathfrak e_{q,ij}(z)=\mathbb{E}[f_{ij}(Q)\mid Z_q(Q)=z]$, $\widetilde f_q(z)=\max_{i<j}|m_{q,ij}(z)|.$ We estimate these conditional means using independent i.i.d CGQs and $K$-nearest-neighbor regression at the observed $Z_q$ configurations. For $K=2$, we additionally visualize $\widetilde f_q(z_1,z_2)$ using Gaussian-kernel surfaces. Importantly, \textsc{$k$-ALeBi} observes CGQs responses rather than individual model predictions.

Figure~\ref{fig:interpretability1} ranks features according to the mean across used seeds of the $90^{\text{th}}$ percentile of the estimated $\widetilde f_q(Z_q)$. The corresponding feature-pair visualizations retain information about the direction and configuration of group disparities with respect to $\modelowner$ that would otherwise be lost by reducing them to a single unsigned score. Thus the analysis identifies not only the features associated with larger measured disparities, but also the group pairs and feature configurations under which the learned model exhibits different CGQd responses.  For the COMPAS dataset, Figure~\ref{fig:interpretability2} indicates larger estimated disparities for individuals of non-Caucasian race who had been arrested multiple times, with the disparity varying across the number of prior arrests ($n\in\{1,2,5\}$). For the Student dataset, past-class failures and the amount of time spent going out are associated with larger gender-related disparities, with the estimated disparity being larger for females than for males. The feature "father's education" also exhibits a disparity pattern: higher levels of paternal education are associated with larger estimated differences in the model's group disparities for males.

\begin{figure}[h!]
    \centering
    \includegraphics[width=0.95\linewidth]{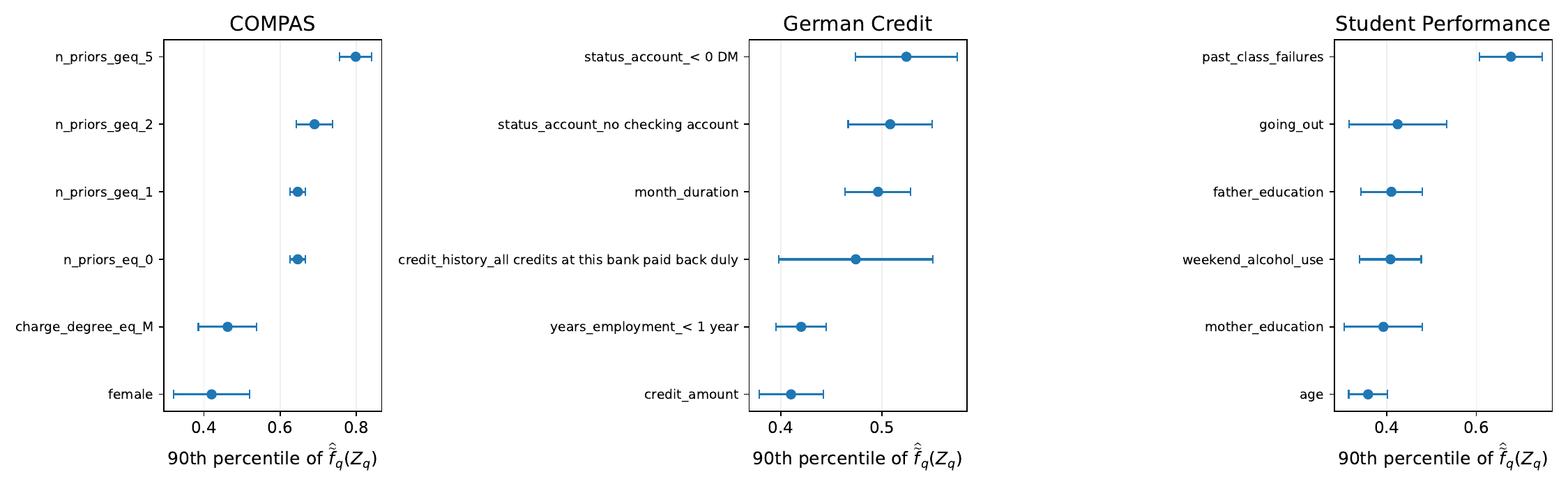}
    \caption{Feature-wise ranking of probe-derived disparity scores.}\label{fig:interpretability1}
\end{figure}

\begin{figure}[h!]
    \centering
    \includegraphics[width=0.95\linewidth]{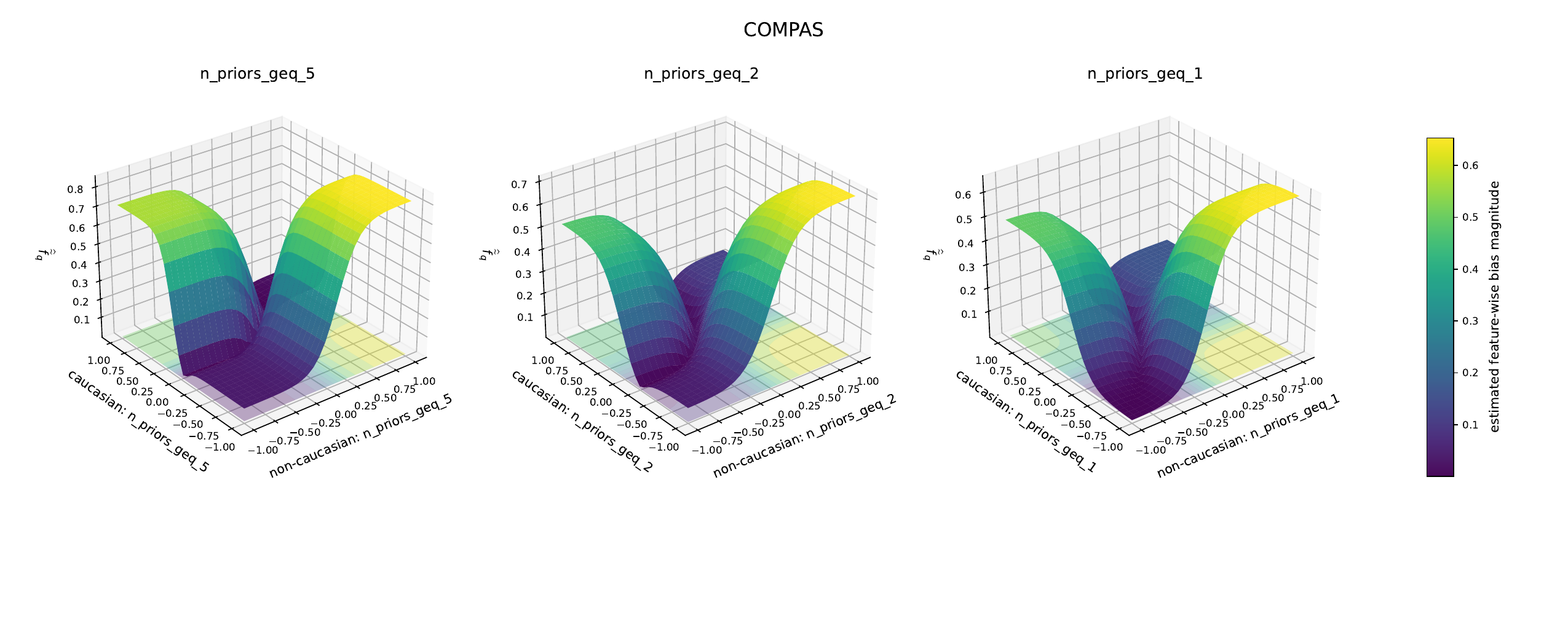}
    \caption{Feature-wise visualization of learned probes on the COMPAS dataset.}
    \label{fig:interpretability2}
\end{figure}

\begin{figure}[h!]
    \centering
    \includegraphics[width=0.95\linewidth]{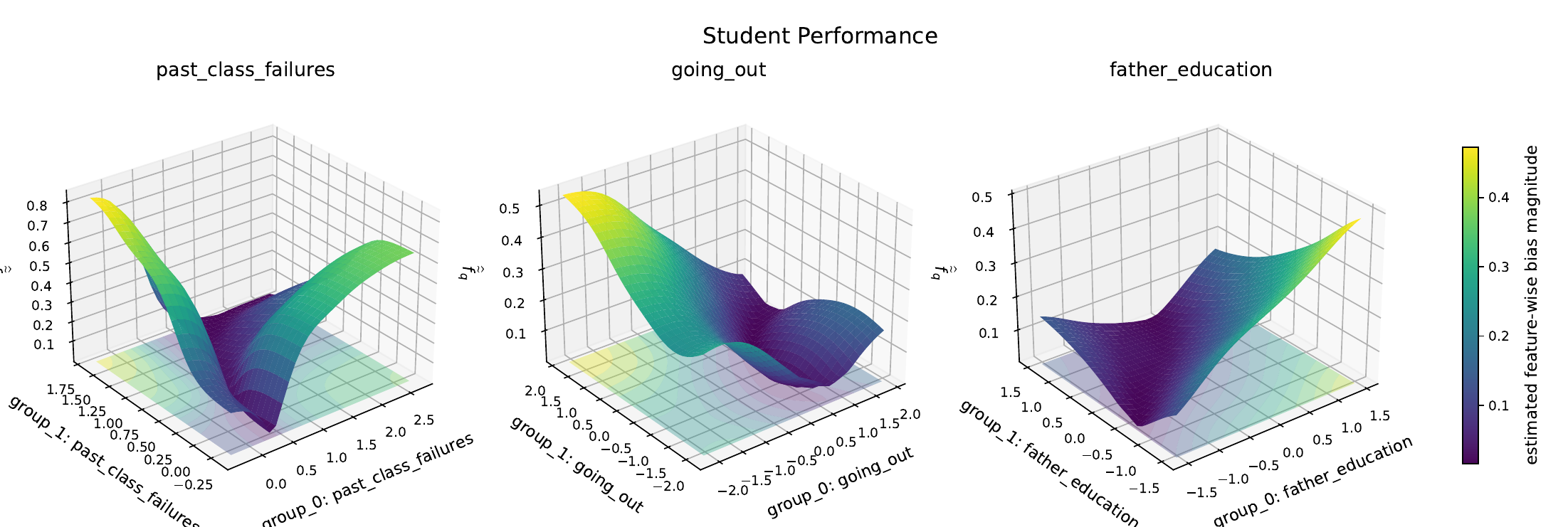}
    \caption{Feature-wise visualization of learned probes on the Student dataset.}
    \label{fig:interpretability4}
\end{figure}

\subsection{Protection Against Model Extraction Attacks}

We next examine the distinction between accurately auditing fairness and revealing sufficient information to reconstruct the underlying model. At limited query budgets, Cross Group Queries can preserve a large version space of models with similar discriminative behavior while still providing accurate fairness estimates. This demonstrates an empirical separation between audit utility and immediate model identification: useful fairness information can be obtained without necessarily exposing the individual predictions required for direct model reconstruction.

\begin{figure}[h!]
    \centering
    \includegraphics[width=0.95\linewidth]{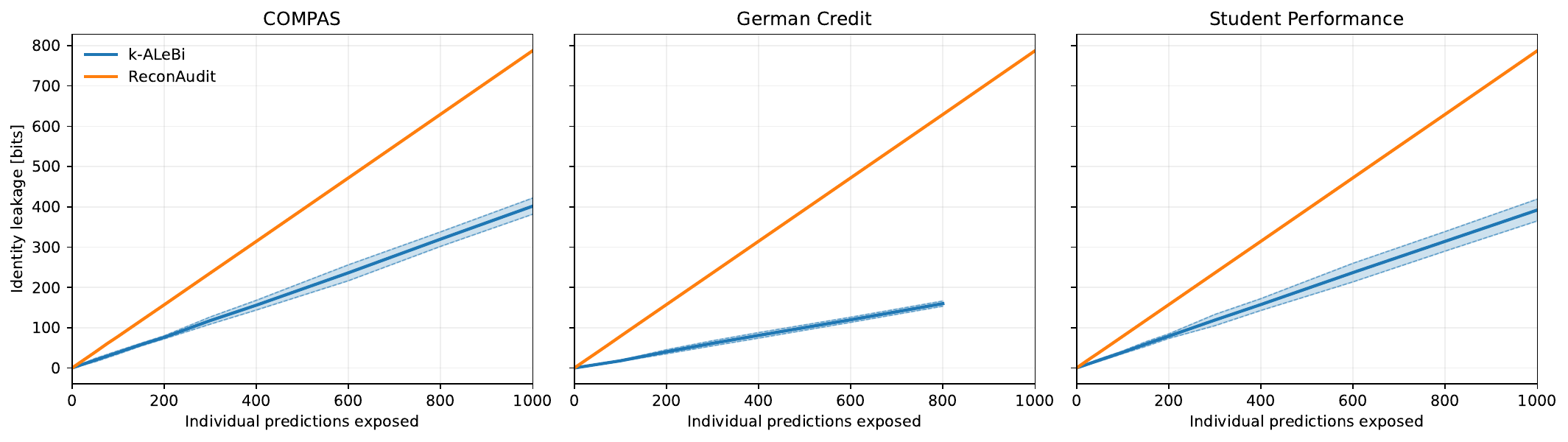}
    \caption{Identity leakage, measured by the number of individual model responses revealed to the auditor, comparing \textsc{$k$-ALeBi} with \textsc{ReconAudit}.}
    \label{fig:leak}
\end{figure}

Figure~\ref{fig:leak} compares the identity leakage induced by \textsc{$k$-ALeBi} and \textsc{ReconAudit}. \textsc{ReconAudit} uses membership queries by default, which directly expose individual model labels to the auditor. In contrast, \textsc{$k$-ALeBi} operates through relational CGQs and therefore does not directly reveal the corresponding individual labels. We quantify this information exposure using the identity-leakage measure defined in Definition~\ref{def:leake}. Across the evaluated query budgets, the results show substantially lower identity leakage for \textsc{$k$-ALeBi} than for \textsc{ReconAudit}, while retaining the ability to perform fairness auditing.

\begin{figure}[t!]
    \centering
    \includegraphics[width=0.95\linewidth]{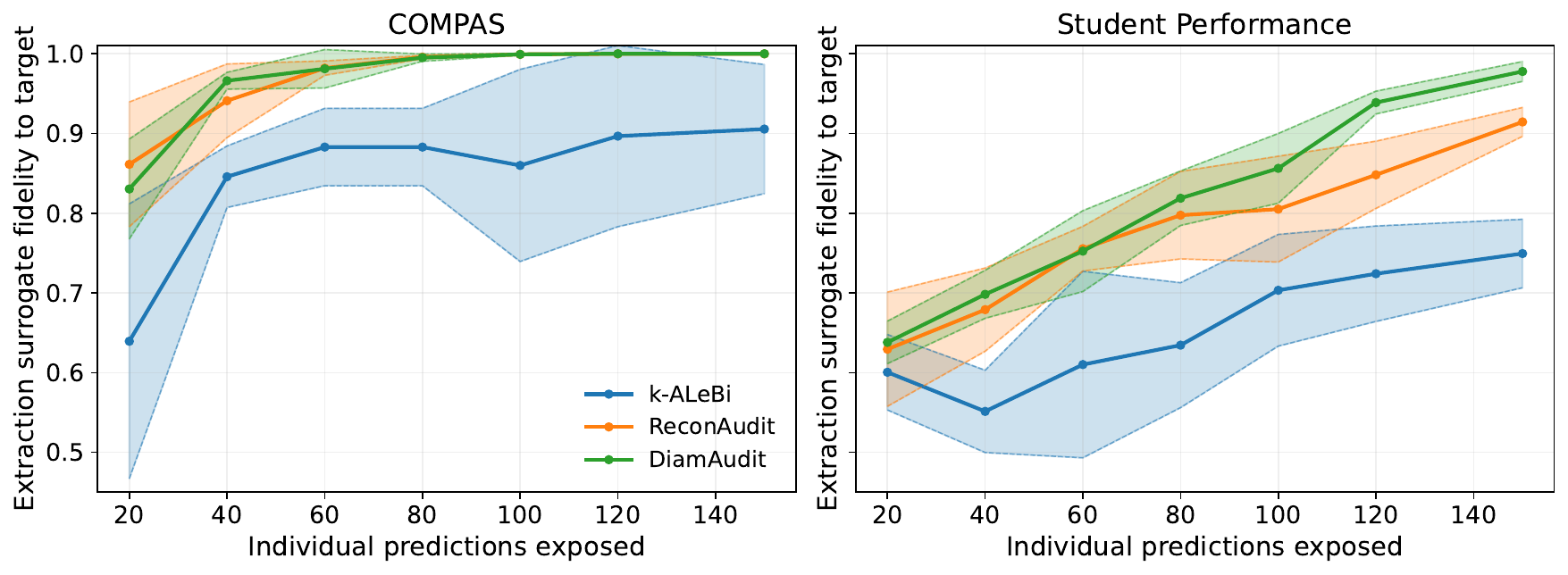}
    \vspace{1em}
    \includegraphics[width=0.95\linewidth]{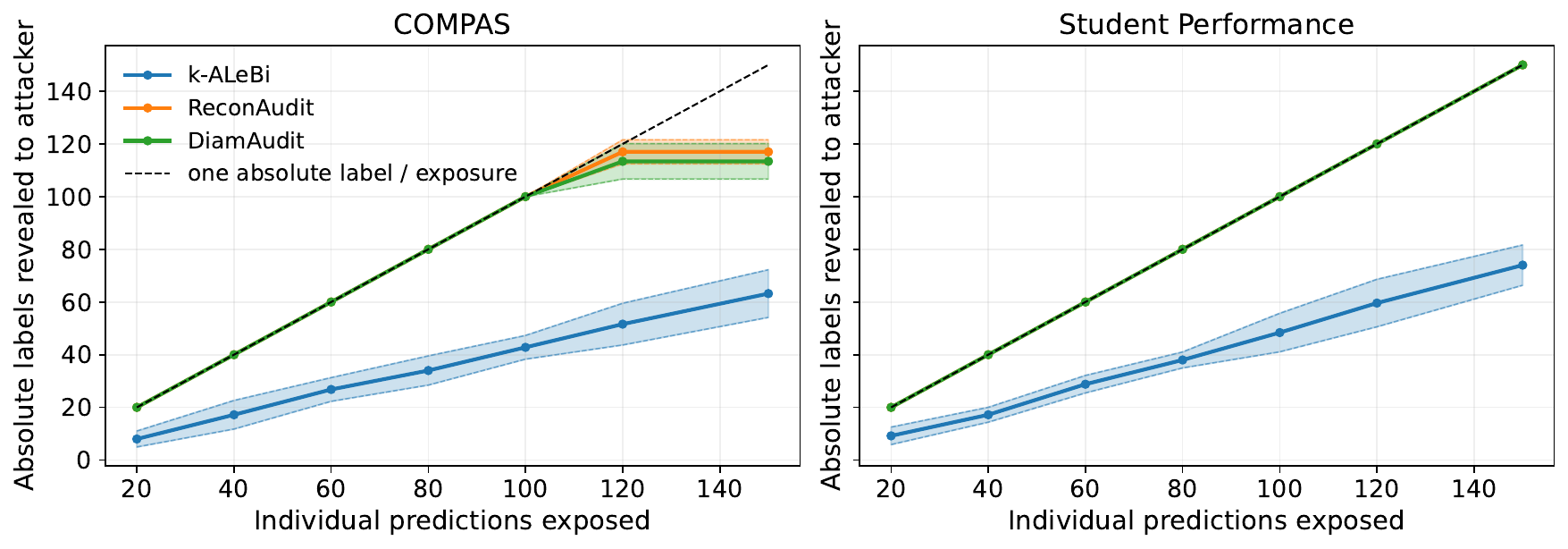}
    \caption{Model-extraction behavior under \textsc{$k$-ALeBi} and \textsc{DiamAudit}: extraction surrogate fidelity to the target model and the number of absolute labels revealed to the attacker.}
    \label{fig:extract_yan}
\end{figure}

Figure~\ref{fig:extract_yan} further compares \textsc{$k$-ALeBi} with \textsc{DiamAudit}. The first panel measures the fidelity of an extraction surrogate to the target model as individual predictions are exposed, while the second reports the number of absolute labels revealed to the attacker. These results illustrate the information advantage of relational auditing: \textsc{$k$-ALeBi} can provide useful fairness information through CGQs without requiring the same level of exposure of individual model predictions as membership-query-based extraction.

Overall, the results indicate that CGQs can reduce the amount of model-identifying information exposed during auditing. This allows fairness properties to be investigated while limiting direct access to individual model responses, thereby separating the utility of fairness auditing from the immediate reconstruction of the underlying model.
\subsection{Robustness against Fairness-Aware Adversaries}

Finally, we evaluate the robustness of \textsc{Robust $k$-ALeBi} against a fairness-aware adversary. On COMPAS-linear, for example, under a high corruption regime of $p=0.40$, the adversary induces an average of 41.4 raw coordinate corruptions, of which only 6.2 result in corrupted queries after filtering. Despite these corruptions, the fairness estimation error increases only marginally, from 0.019 to 0.021, with no observed concealment across the five runs. In comparison, the reconstruction error of \textsc{ReconAudit} increases from 0, corresponding to exact model identification in the absence of corruption, to 0.074 under attack, with a concealment rate of 0.20. 

\begin{figure}[h!]
\centering
\includegraphics[width=0.95\linewidth]{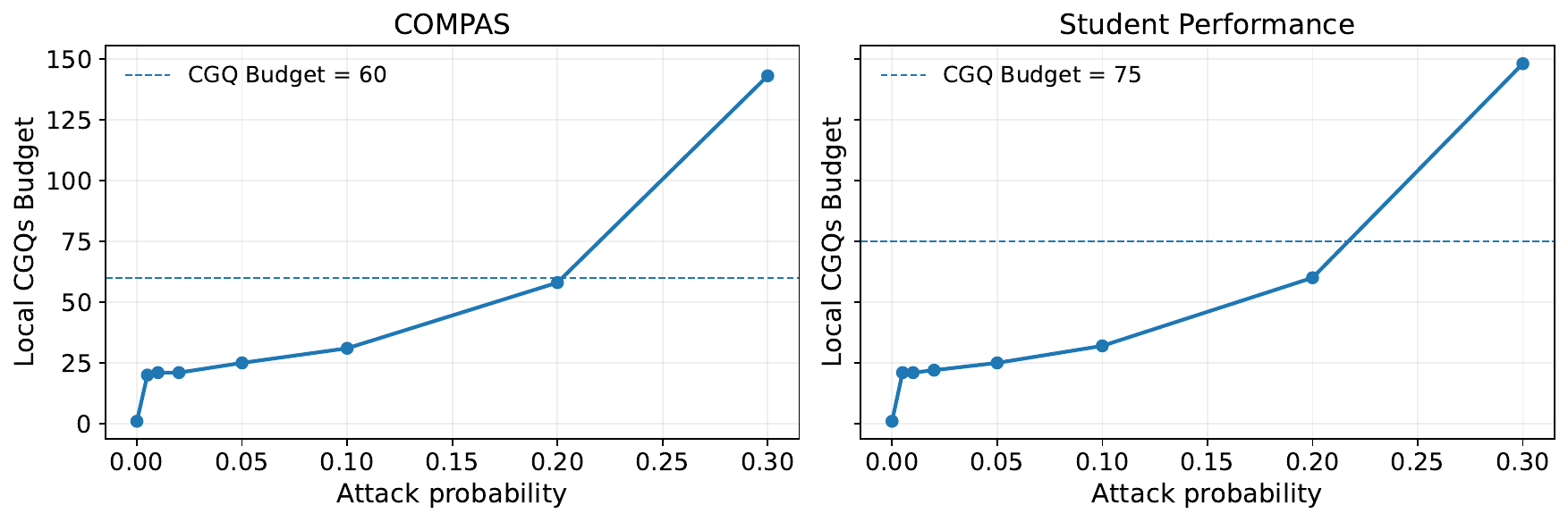}
\caption{Local CGQ budget of \textsc{Robust $k$-ALeBi} across cells as a function of the attack probability.}
\label{fig:robust-cgq-budget}
\end{figure}

The \cite{yan2022active} baseline exhibits an estimation error of 0.03 in the absence of attack, which increases to 0.05 under the strongest feasible attack probability, capped at 0.12. The same qualitative stability of \textsc{Robust $k$-ALeBi} is observed on the German-linear and Student-linear models.

\begin{figure}[h!]
\centering
\includegraphics[width=0.95\linewidth]{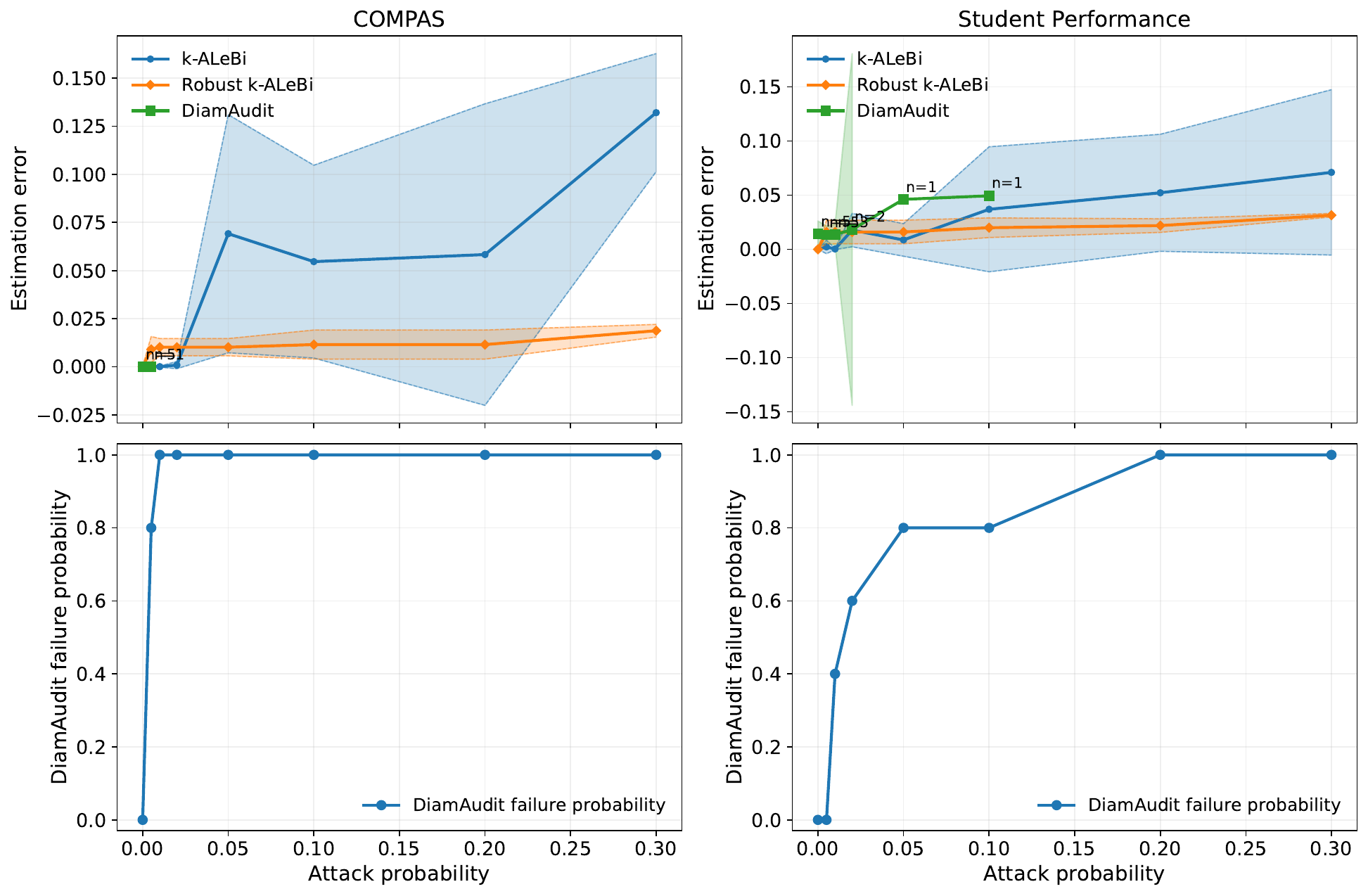}
\caption{Estimation error for statistical parity as a function of the attack probability.}\label{fig:robust-est-error-yan}
\end{figure}%

%% file: sections/conclusion.tex
\section{Conclusion}

In this work, we established an explicit separation between reconstruction-based and probe-based approaches to fairness auditing. We characterized the probe-learning problem through the multi-context star number and showed that, even when reconstruction-based auditing is infeasible for linear classifiers, the corresponding probe-based approach remains learnable, addressing \hyperlink{Q1}{\textbf{\textcolor{blue}{Q1}}}. Building on this formulation, we designed an auditing algorithm based on learning interpretable bias probes, providing audit information directly in terms of the fairness property of interest and addressing \hyperlink{Q2}{\textbf{\textcolor{blue}{Q2}}}. Finally, we analyzed the auditing problem in the presence of a fairness-aware adversary and established both upper and lower bounds on the resulting audit complexity, addressing \hyperlink{Q3}{\textbf{\textcolor{blue}{Q3}}}.

%% file: appendix/frameworkbis.tex
\section{Extended Framework Description}
\subsection{General Framework for Property Specific Audits}\label{app:generalfram}

Auditing ML models concerns a broad range of properties, including algorithmic fairness, robustness, etc. In a black-box setting, auditing a given property requires carefully designing the queries submitted to the model. Ideally, these queries should be informative about the property under audit while minimizing the amount of information they reveal about the model itself, thereby reducing the risk of model extraction. This principle can be applied to properties involving multiple groups, but we argue that it extends more generally to arbitrary classes of properties. For example, robustness can be interpreted as a two-group property, where an original input and its perturbed counterpart can be viewed as samples from two related groups (perturbed samples come from a ghost group). 

\begin{figure}[H]
    \centering
    \includegraphics[width=0.5\linewidth]{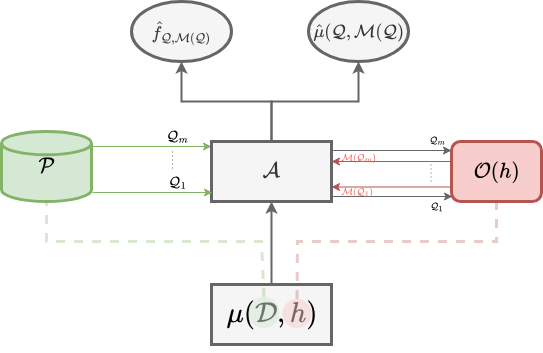}
    \caption{General Property-specific Audit Framework. }
    \label{fig:generalprop}
\end{figure}

More generally, as illustrated in Figure~\ref{fig:generalprop}, the model owner can construct a pool of candidate samples that is designed to avoid leaking sensitive information about the model while remaining informative with respect to the property being audited. The auditor $\auditor$ can then select, from this pool, the samples that are most informative for the property of interest and submit the corresponding queries to the model owner $\modelowner$. This separation between sample pool construction and query selection allows the auditor to retain flexibility in designing the audit while giving the model owner control over the information exposed through the available queries.

\subsection{Information-theoretic and combinatorial separation between active learning and active auditing in the presence of protected groups} \label{app:linearstar}

\paragraph{Active learning of linear classifiers.} It is known that the linear classifier learnability in the active setting fails even in 2 dimensions. Sangupta came up with a simple counterexample of why this happens; as Figure \ref{fig:alcountexample} shows, when considering a circular representation of data points there exists a strategy to construct an infinite star set. To see this, let $h_0$ be the center of the star set, one can observe that for all $i \in [q]$, $\text{DIS}(h_0,h_i) = \{x_i\}$. It is also easy to see that this occurs while $q$ can go to infinity. This construction shows that the star number of linear classifiers class explodes to infinity just in two dimensions.

\begin{figure}[H]
    \centering
    \includegraphics[width=0.6\textwidth]{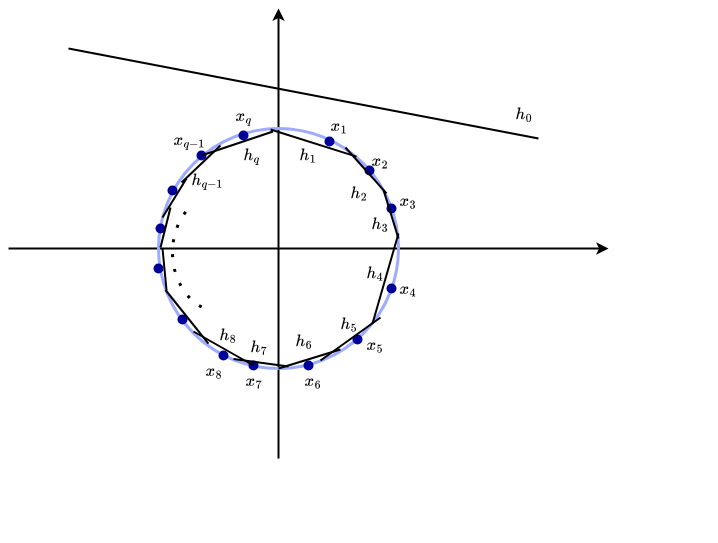}
    \caption{2D of linear classifier counterexample in the active leanring setting: the blue dots $x_i$ represent covariates, the dark lines represent classifiers $h_i.$}
    \label{fig:alcountexample}
\end{figure}

\paragraph{Active auditing in the presence of protected groups.} In the presence of protected groups, an additional feature (protected feature) needs to be added to the setting. Therefore to reproduce this counter-example for the auditing problem, we add an additional feature, therefore $\cX \subseteq \R^3$ as shown in Figure \ref{fig:3dcountexample}. We consider the scenario where $\cX = \cX^0 \times \cX^1 \times \cX^2 $, $\cX^0$ denotes the protected feature while $\cX^1$ and $\cX^2$ denote the unprotected features. $\cX^0_1 \subseteq \cX$ (resp.  $\cX^0_1$) denotes the first (resp. second) protected group.

\begin{figure}[H]
    \centering
    \includegraphics[width=0.8\textwidth]{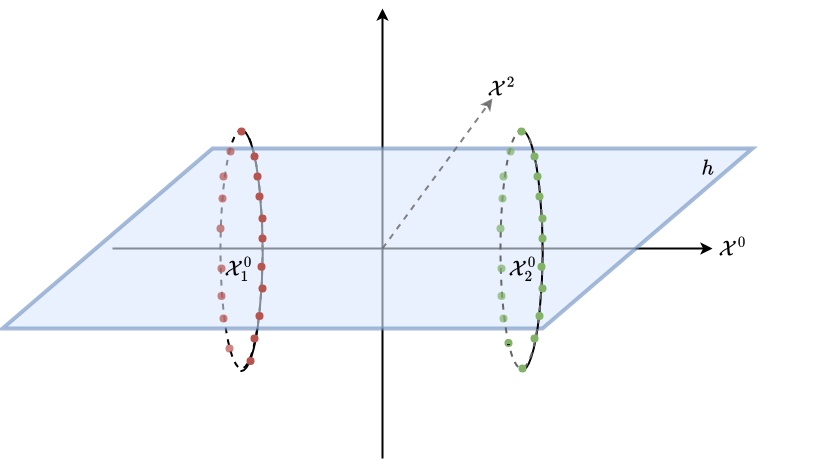}
    \caption{Illustration of \cite{dasgupta2004analysis} counterexample in the context of fairness auditing.  Red dots represent samples from the first protected group while green dots represent sampless from the second protected group. }
    \label{fig:3dcountexample}
\end{figure}

It is easy to see that, in this setting the counterexample of \cite{dasgupta2004analysis} still applies in 3D. and therefore the star number is infinite. However, when we focus on solely extracting the discriminative behavior of $h$ instead of the whole $h$ (through active learning approaches), we avoid unnecessary information in the subspace of $\cX$ where $h$ is fair. In other words, the information $h(x_1) -h(x_2) = 0$ is equivalent to both $h(x_0) = h(x_1) = 0$ and $h(x_0) = h(x_1) = 1$  in the context of fairness auditing but not in the context of learning, this shows that auditing problem reduces significantly the number of dichotomies used in learning.

To see why the dual-number corresponding to the audit problem illustrated in Figure \ref{fig:3dcountexample}, we reduce learning the overall bias canvas to learning canvas biases in each unprotected direction. The problem now reduces to the one explained in Example \ref{example:lin2d}, and the bias canvas in each direction is illustrated in Figure \ref{fig:3dcanvas}.

\begin{figure}[H]
    \centering
    \includegraphics[scale = 0.27]{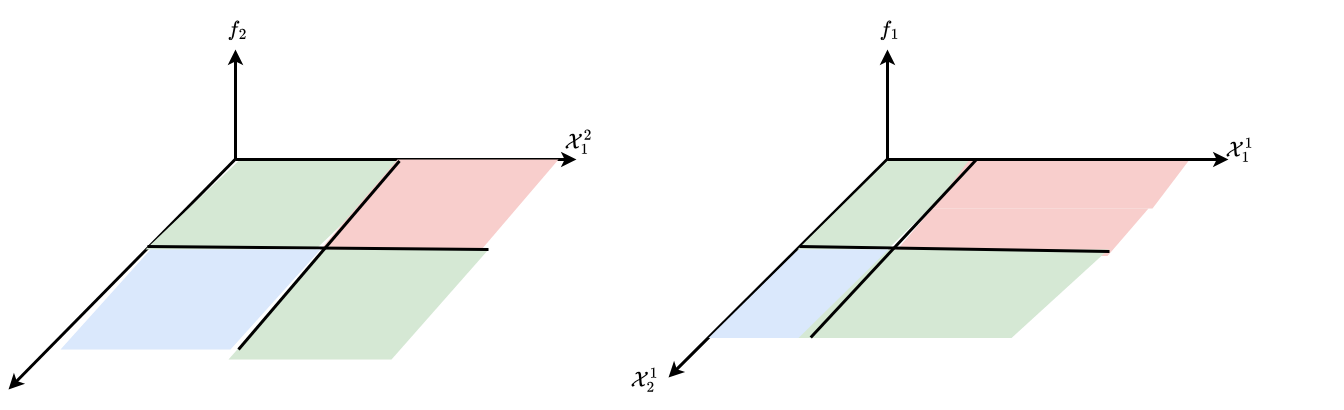}
    \caption{Bias canvases of $h$ in feature 1 (right) and in feature 2 (left).}
    \label{fig:3dcanvas}
\end{figure}

In this example the overall bias canvas is defined on the space $\cX^1_1 \times \cX^1_2 \times \cX^2_1 \times \cX^2_2 \subseteq \R^4 $.

\subsection{Examples of relative dual classes and their star numbers from known hypothesis classes}\label{app:examplestar}

When the original model class is known, the dual context class can be defined explicitly based on the model class. In this scenario, we refer to relative dual context classes, where instead of searching for arbitrary preference functions modeling the discriminative behavior of the model, we restrict our search within dual classes that depend on the original model class. We provide two examples to compute the dual star number, one for linear classifiers and the other for axis-parallel rectangles.

\begin{example}[Linear Classifiers as relative dual-context classes.]\label{example:lin2d}

For the relative dual class of linear separators in dimension 2, $\cF^{\text{lin}}_2$, the dual star number is $2$. To convince ourselves of that, consider the figure below. The preference function $f_0$ represents the center of the star set. It acts discriminatively on both the blue and red samples. For the red (resp. blue) sample, it prefers the one in $\cX_0$ (resp. the one in $\cX_1$). As the figure shows, $\{f_0,f_1,f_2\}$ consists of a star set for the dual-class $\cF(\cH^{\text{lin}}_{\R^2})$, since $\text{DIS}^{\text{rel}}_\text{dual}(\{f_0,f_i\}) = \{x^{(i)}\} $, for $i \in \{1,2\}$. It is easy to see that any new point of couples added to the partition of $\cX_0 \times \cX_1$ will increase the disagreement, therefore the star dimension is $2$.

\begin{figure}[H]
    \centering
    \includegraphics[width=0.8\textwidth]{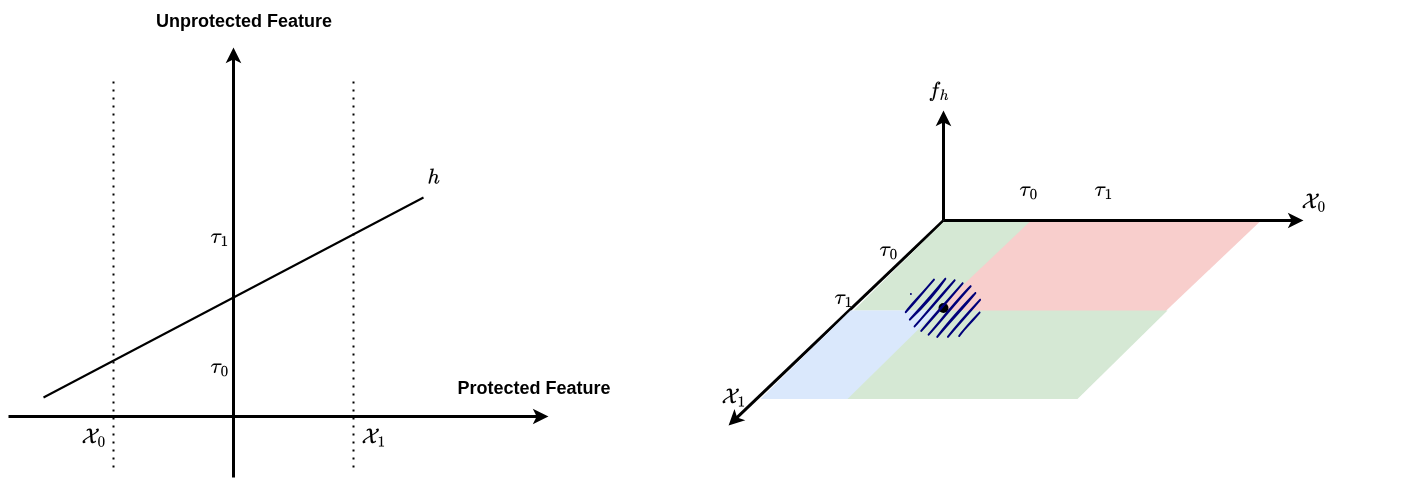}
    \caption{Illustration of the dual star number for the dual class relative to $\cH^{\text{lin}}_{\R^2}$.}
    \label{fig:illuslinear}
\end{figure}
\end{example}

By the symmetry of $\cF^{\text{lin}}_2$, the same reasoning apply for points chosen on the fair region (green area) of the center preference $f_0$.

    

\begin{example}[Axis-parallel rectangles as relative dual-context classes.]
    When the original model class consists of axis-parallel rectangles. The set of possible discriminative dichotomies increases. Consider the classification problem illustrated in Figure \ref{fig:illusrec}. One can observes that the same pattern of the previous examples repeats itself four times ( two times along 

    \begin{figure}[H]
    \centering
    \includegraphics[width=0.8\textwidth]{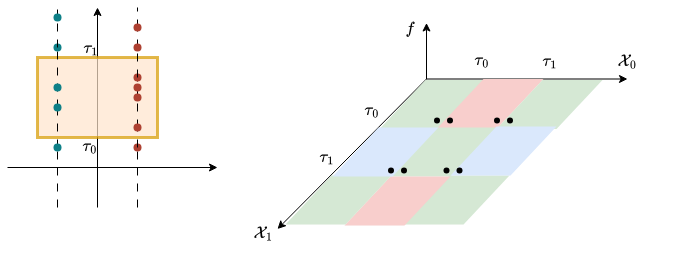}
    \caption{Illustration of the dual star number for the dual class relative to $\cH^{\text{rec}}_{\R^2}$.}
    \label{fig:illusrec}
\end{figure}

\end{example}

\subsection{Mutli-star number of multi-context class relative to linear classifiers}\label{app:multistarlinear}
\starlinearclass*

\begin{proof}

Let $\mathcal H = \left\{ h_w(X)=\mathds{1}_{w  X \ge 0}: w\in\mathbb R^{*2} \right\}$ be the class of homogeneous linear classifiers.

Let $R=\begin{pmatrix}0&-1\\1&0\end{pmatrix}$ denote rotation by $\pi/2$, and consider the following representation of protected groups domain $ \cX_{R} = \cZ_1 \times \cZ_2 = \{(x,Rx):x\in\mathbb R^{*2} \}.$

For $h\in \cH$, we have the bias probe for our audit instance problem $k=2$, $f_h(x,Rx) = h(x)-h(Rx) \in\{-1,0,+1\}, $
and let $\cF_{R, \cH} = \{f_h:h\in\mathcal H\}$. 

We show that $\mathfrak s(\cH)=\infty, \qquad \mathfrak s_2(\cF_{\cH,R})=3. $

The fact that $\mathfrak s(\mathcal H)=\infty$
for homogeneous linear separators in $\mathbb R^2$ is standard \citep{hanneke2014theory}.
Thus it remains to compute the star number of the induced relational
class $\mathcal \cF_{\cH,R}$.

\paragraph{Step 1: Angular representation of the $2$-context class.}

Since the classifiers are homogeneous, only the directions of $w$
and $x$ matter. 

Let $w_\theta=(\cos\theta,\sin\theta)$
and $x_\phi=r(\cos\phi,\sin\phi)$,
where $r>0.$

Then $ w_\theta x_\phi = r\cos(\theta-\phi).$
Since $Rx_\phi$ has angle $\phi+\pi/2$,
\[
w_\theta^\top Rx_\phi
=
r\cos\left(\theta-\phi-\frac{\pi}{2}\right)
=
r\sin(\theta-\phi).
\]
Consequently,
\begin{equation}
\label{eq:f-angular}
f_\theta(\phi)
= 
\operatorname{sign}\bigl(\cos(\theta-\phi)\bigr)
-
\operatorname{sign}\bigl(\sin(\theta-\phi)\bigr).
\end{equation}

We choose all probe and witness angles away from the transition
points, so the value assigned to angle $0$ is
irrelevant.

Let $\beta=\theta-\phi \pmod{2\pi}.$ Then, 
\begin{equation}
\label{eq:four-sector-labeling}
f_\theta(\phi)
=
\begin{cases}
+1,
&
\beta\in(-\pi/2,0),
\\[1mm]
0,
&
\beta\in(0,\pi/2),
\\[1mm]
-1,
&
\beta\in(\pi/2,\pi),
\\[1mm]
0,
&
\delta\in(\pi,3\pi/2).
\end{cases}
\end{equation}

Thus the $2$-context class becomes the $\beta$ translation class
obtained by rotating the four-sector cyclic labeling $+1,\quad 0,\quad -1,\quad 0,$ with each sector having angular width $\pi/2$.

\paragraph{Step 2: Agreement sets relative to a probe.} Fix a reference orientation $\theta_0$. By rotational invariance,
we may assume without loss of generality that $\theta_0=0.$
For every probe angle $\phi$, define its agreement set in parameter
space by
\[
A_\phi
=
\left\{
\theta\in\mathbb S^1:
f_\theta(\phi)=f_0(\phi)
\right\}.
\]
By construction, $0\in A_\phi$ for every $\phi$. Equation~\eqref{eq:four-sector-labeling} shows that there are exactly
two possible forms for $A_\phi$.

\medskip

\begin{itemize}
    \item \emph{Discriminatory Probe: $f_0(\phi)\in\{-1,+1\}$:}  In this case the reference label occurs on exactly one sector of
length $\pi/2$. Hence $A_\phi=I_\phi,$ where $I_\phi$ is an arc of length $\pi/2$ containing $0$.
\item \emph{Fair Probe: $f_0(\phi)=0$.} The label $0$ occurs on two antipodal sectors, each of length
$\pi/2$. Therefore $A_\phi=I_\phi\cup(I_\phi+\pi),$ where again $I_\phi$ is the unique component containing $0$ and has
length $\pi/2$.
\end{itemize}

Thus every probe has a \emph{local agreement interval}
$I_\phi$ of length $\pi/2$ containing the reference parameter $0$.
A Fair probe additionally has an antipodal copy of this interval on the two dimensional sphere.

\paragraph{Step 3: An interval lemma.}

We use the following elementary fact.

\begin{lemma}
\label{lem:interval-private}
Let
\[
I_i=(\ell_i,r_i),
\qquad i=1,\ldots,m,
\]
be intervals on the real line satisfying $\ell_i<0<r_i$ for every $i$. Suppose that, for an index $i$, there exists a point
$t_i$ such that
\[
t_i\in\bigcap_{j\neq i} I_j
\qquad\text{but}\qquad
t_i\notin I_i.
\]
Then $i$ must be either an index attaining the largest left endpoint
or an index attaining the smallest right endpoint. Consequently, at
most two indices can satisfy this property.

\end{lemma}

\begin{proof}
Suppose $ t_i\in\bigcap_{j\neq i}I_j \setminus I_i.$ Since every interval contains $0$, either $t_i<\ell_i<0$ or
$t_i>r_i>0$.

If $t_i<\ell_i$, then $t_i\in I_j$ for every $j\neq i$, so $\ell_j<t_i<\ell_i$ for all $j\neq i.$
Hence $\ell_i>\ell_j$ for all $j\neq i, $
so $i$ is the unique index with largest left endpoint.

Similarly, if $t_i>r_i$, then $r_i<t_i<r_j$ for all $j\neq i$, and therefore $i$ is the unique index with smallest right endpoint.

There can be at most one unique largest left endpoint and at most
one unique smallest right endpoint. Hence at most two indices can
possess such a private point.
\end{proof}

\paragraph{Step 4: Upper bound $\mathfrak s(\mathcal Fair)\leq 3$.}

Suppose that $\phi_1,\ldots,\phi_m$ form a star centered at $f_0$. By the definition of the star number, for every $i\in[m]$ there
exists an orientation $\theta_i$ such that
\begin{equation}
\label{eq:star-condition}
f_{\theta_i}(\phi_i)\neq f_0(\phi_i),
\qquad
f_{\theta_i}(\phi_j)=f_0(\phi_j)
\quad\text{for all }j\neq i.
\end{equation}
Equivalently,
\begin{equation}
\label{eq:agreement-star}
\theta_i
\in
\bigcap_{j\neq i}A_{\phi_j}
\setminus A_{\phi_i}.
\end{equation}

We distinguish cases according to the number of discriminatory probes.

\begin{itemize}
    \item \emph{Case 1: There are no discriminatory probes.}

Then every probe is fair:
\[
A_{\phi_i}
=
I_i\cup(I_i+\pi).
\]
Since this family is invariant under translation by $\pi$, we may
project the parameter circle modulo $\pi$. Under this projection,
each agreement set becomes a single interval $I_i$ of length
$\pi/2$ containing the reference point $0$.

Condition~\eqref{eq:agreement-star} therefore reduces exactly to the
private-point condition in Lemma~\ref{lem:interval-private}.
Consequently, $m\leq2.$
\item \emph{Case 2: There is exactly one discriminatory probe.}

Let this probe have index $q$. For every $i\neq q$, a witness
$\theta_i$ satisfying~\eqref{eq:agreement-star} must in particular
belong to $A_{\phi_q}=I_q.$
Since $I_q$ is the local interval of length $\pi/2$ containing $0$,
$\theta_i$ cannot lie in the antipodal component
$I_j+\pi$ of any Fair agreement set. Hence, for all $i\neq q$,
the star condition reduces locally to
\[
\theta_i
\in
\bigcap_{j\neq i} I_j
\setminus I_i.
\]
By Lemma~\ref{lem:interval-private}, at most two of the Fair probes
can possess such witnesses.

Since there is only one discriminatory probe, this yields $m\leq 1+2=3.$

\item \emph{Case 3: There are at least two Discriminatory probes.}

For every $i$, the intersection
\[
\bigcap_{j\neq i}A_{\phi_j}
\]
contains a Discriminatory agreement interval, except possibly when $i$ is
one of the Discriminatory probes and there is exactly one other Discriminatory probe.
In either case, because there is at least one Discriminatory agreement set
among the sets indexed by $j\neq i$, every witness $\theta_i$ is
confined to the local component containing $0$.

Thus all relevant agreement sets may be replaced by their local
intervals $I_j$, and condition~\eqref{eq:agreement-star} again
reduces to
\[
\theta_i
\in
\bigcap_{j\neq i}I_j
\setminus I_i.
\]
Lemma~\ref{lem:interval-private} then implies $m\leq2.$
\end{itemize}

Combining the three cases gives $\mathfrak s_2(\mathcal \cF_{\cH,R})\leq3.$

\paragraph{Step 5: Lower bound $\mathfrak s_2(\mathcal F_{\cH,R})\geq3$.}

It remains to exhibit an explicit $3$-star on the 2-context class. To this end, consider the reference orientation $\theta_0=0$ and choose the three probe angles

\[
\phi_1=\frac{\pi}{20},
\qquad
\phi_2=\frac{21\pi}{40},
\qquad
\phi_3=\frac{23\pi}{40}.
\]

Using~\eqref{eq:f-angular}, their labels under the reference
orientation are
\[
\left(
f_0(\phi_1),
f_0(\phi_2),
f_0(\phi_3)
\right)
=
(+1,0,0).
\]

Now choose the three witnessing orientations
\[
\theta_1=\frac{4\pi}{5},
\qquad
\theta_2=\frac{3\pi}{80},
\qquad
\theta_3=\frac{25\pi}{16}.
\]
Direct substitution into~\eqref{eq:f-angular} gives
\[
\left(
f_{\theta_1}(\phi_1),
f_{\theta_1}(\phi_2),
f_{\theta_1}(\phi_3)
\right)
=
(-1,0,0),
\]
\[
\left(
f_{\theta_2}(\phi_1),
f_{\theta_2}(\phi_2),
f_{\theta_2}(\phi_3)
\right)
=
(+1,+1,0),
\]
and
\[
\left(
f_{\theta_3}(\phi_1),
f_{\theta_3}(\phi_2),
f_{\theta_3}(\phi_3)
\right)
=
(+1,0,-1).
\]
Therefore,
\[
\left\{
j:
f_{\theta_1}(\phi_j)\neq f_0(\phi_j)
\right\}
=
\{1\},
\]
\[
\left\{
j:
f_{\theta_2}(\phi_j)\neq f_0(\phi_j)
\right\}
=
\{2\}
\]
and
\[
\left\{
j:
f_{\theta_3}(\phi_j)\neq f_0(\phi_j)
\right\}
=
\{3\}
\]
Hence $\mathfrak s(\mathcal \cF_{\cH,R})\geq3.$ Together with the upper bound, we get $\mathfrak s_2(\mathcal \cF_{\cH,R})=3.$

Finally, since $\mathfrak s(\mathcal H)=\infty,$
we obtain the desired information-theoreticseparation.
\end{proof}

\begin{remark}[The protected groups domain restriction is essential]
\label{rem:domain-essential}
The finite-star conclusion on the 2-context class relies essentially on restricting the
allowable cross-group probes to
\[
\mathcal Z_R=\{(x,Rx):x\neq0\}.
\]
It does not hold for the unrestricted product-space class
\[
\widetilde{\mathcal F}
=
\left\{
(x_1,x_2)
\mapsto
h(x_1)-h(x_2)
:
h\in\mathcal H
\right\}.
\]
Indeed, for every homogeneous linear classifier and every $x$ away
from its decision boundary,
\[
h(-x)=-h(x),
\]
and hence
\[
f_h(x,-x)
=
h(x)-h(-x)
=
h(x) \footnote{Here, we reorder the labels: instead of $\{-2,0,2\}$ ($h$ takes labels $\{0,1\}$) we use labels $\{-1,0,1\}$ ($h$ takes labels $\{-1,1\}$))}.
\]
Thus the restriction of $\widetilde{\mathcal F}$ to pairs of the
form $(x,-x)$ contains an exact copy of $\mathcal H$, implying
\[
\mathfrak s_2(\widetilde{\mathcal F})
\geq
\mathfrak s(\mathcal H)
=
\infty.
\]
Therefore the separation in Proposition~\ref{prop:starlinear} arises
from the combination of relational predictions and the structured
protected groups domain, rather than from taking prediction
differences alone.
\end{remark}

\newpage

\section{Protection against Model Extraction}\label{app:protextraction}

\computeLPN*

\begin{proof}
We reduce from the Minimum Test Cover problem (Definition \ref{def:testcover}), known to be NP-complete.

We begin by showing that $B$-\textsc{Bounded-Relational-Extraction} is NP-hard via a polynomial-time reduction from \textsc{Minimum-Test-Cover}, even when the multi-context class and the set of admissible queries are finite and every query uses two points that occur in no other admissible query.

Membership in NP is immediate: given a candidate $\mathcal Q'\subseteq\mathcal Q$, one can verify $|\mathcal Q'|\leq B$ and, for each of the at most $\binom{m}{2}$ pairs of hypotheses, check whether some query in $\mathcal Q'$ distinguishes the pair. By assumption, every relational response is polynomial-time computable. For NP-hardness, we reduce from \textsc{Minimum Test-Cover}.

Let $(U,\mathcal T,K)$ be an instance, where $ U=\{u_1,\ldots,u_m\}, \qquad \mathcal T=\{T_1,\ldots,T_r\}.$ A subcollection $\mathcal T'\subseteq\mathcal T$ is a test cover if for every pair $u_i\neq u_j$, there exists $T_\ell\in\mathcal T'$ containing exactly one of $u_i$ and $u_j$. 

First, we can observe that if some pair $u_i\neq u_j$ is not distinguished by any test in $\mathcal T$, then the \textsc{Test-Cover} instance is necessarily a NO instance. Such instances may therefore be mapped in polynomial time to any fixed NO instance of the target problem. Hence, for the remainder of the reduction, we may assume that every pair of objects is distinguished by at least one test. For every test $T_\ell$, introduce two fresh domain points $a_\ell,b_\ell$. All such points are distinct, so $\{a_\ell,b_\ell\} \cap \{a_{\ell'},b_{\ell'}\} = \varnothing,$ for $\ell\neq\ell'$.

For every object $u_j\in U$, define a hypothesis $h_j:\mathcal X\to\{0,1\}$ by $h_j(a_\ell)=1$ and $h_j(b_\ell) = 1-\mathds{1}\{u_j\in T_\ell\}. $ Since every pair of objects is distinguished by at least one test, the resulting hypotheses $h_1,\ldots,h_m$ are pairwise distinct. 

For every $T_\ell$, define the relational query $q_\ell=(a_\ell,b_\ell)$ and let $f_h(q_\ell) = h(a_\ell)-h(b_\ell).$ Then 

$$ f_{h_j}(q_\ell) = h_j(a_\ell)-h_j(b_\ell) = \mathds{1}\{u_j\in T_\ell\}.$$

Consequently, for every $i\neq j$, \begin{align*} f_{h_i}(q_\ell)\neq f_{h_j}(q_\ell) &\iff \mathds{1}\{u_i\in T_\ell\} \neq \mathds{1}\{u_j\in T_\ell\} \\ &\iff T_\ell \text{ distinguishes }u_i\text{ and }u_j. \end{align*}

Now set $B=K$. For any subcollection $\mathcal T'\subseteq\mathcal T$, define $\mathcal Q' = \{q_\ell:T_\ell\in\mathcal T'\}.$ The correspondence $T_\ell\leftrightarrow q_\ell$ is one-to-one, so \[ |\mathcal Q'|=|\mathcal T'|. \] Moreover, by the equivalence above, \[ \mathcal T' \text{ distinguishes every pair in }U \iff \mathcal Q' \text{ distinguishes every pair in }\mathcal H_0. \] Therefore, \[ (U,\mathcal T,K)\in\textsc{Test-Cover} \iff (\mathcal H_0,\mathcal Q,B)\in \textsc{Bounded-Relational-Extraction}. \] The construction uses $2r$ domain points, $m$ hypotheses, and $r$ queries, and its explicit representation can be constructed in $O(mr)$ time. Hence the reduction is polynomial. Finally, since every $q_\ell$ uses the fresh pair $\{a_\ell,b_\ell\}$, distinct queries have disjoint supports. Therefore the NP-completeness result holds even under this restriction. 
\end{proof}

%% file: appendix/manipfree.tex
\section{Proofs for Manipulation-free Regime}

\subsection{Lower Bound on Sample Complexity}\label{app:SClb}
\scLB*

\begin{proof}

\textbf{Step 1: Active bound.}

Let $\alpha_\epsilon = \min \left\{ \mathfrak s_k, \left\lfloor\frac1{2\epsilon}\right\rfloor \right\}.$
Let $S= \{\bx_1,\ldots,\bx_{\alpha_\epsilon}\}$ be a $k$-star with center $f_0$ and witnesses $f_1,\ldots,f_{\alpha_\epsilon}$, and define the fixed coupling $\nu = \frac1{\alpha_\epsilon} \sum_{i=1}^{\alpha_\epsilon} \delta_{\mathbf x_i}$.

We have, $\nu(\{\mathbf x_i\}) = \frac1{\alpha_\epsilon} \ge 2\epsilon > \epsilon.$

Consider any randomized auditor using at most $q$ queries.
We begin by observing :

$$\underset{f^* \in \cF_k}{\sup}\prob \big[\underset{\bX \sim \nu}{\prob} [\widehat f(\bX) \neq f^*(\bX)]>\epsilon \big] \geq \underset{i \in [\alpha_\epsilon]}{\max} \prob \big[\underset{\bX \sim \nu}{\prob} [\widehat f(\bX) \neq f_i(\bX)]>\epsilon \big] $$
We proceed by coupling the audit process with targets $f_0,f_1,\ldots,f_{\alpha_\epsilon}$ using the same unlabeled stream and the same internal randomness for the audit.

Let $Q_0$ denote the distinct star points queried under target $f_0$, and if $$\prob \big[\underset{\bX \sim \nu}{\prob} [\widehat f_0(\bX) \neq f_0(\bX)]>\epsilon \big] >\delta$$
$f_0$ is already hard.

Otherwise, for all $i \in [\alpha_\epsilon]$ with probability at least $1-\delta$, we have 
\begin{equation}\label{eq:xsc}
    \widehat f_0(\mathbf x_i)=f_0(\mathbf x_i)
\end{equation}

Let $E_0$ denotes that event, that is

$$E_0 = \Big\{  \widehat f_0(\bx_i) = f_0(\bx_i): \forall i \in [\alpha_\epsilon] \Big\}$$

We have

\begin{equation}\label{eq:hh}
    \prob|E_0] \ge 1- \delta
\end{equation}

If $\bx_i\notin Q_0,$ then all queries lie in some $\bx_j$, where $j \neq i$.

By the star definition, we know that $f_i(\mathbf x_j)=f_0(\mathbf x_j)$, and a simple induction over the interaction wxith the model owner therefore gives $\widehat f_i=\widehat f_0.$

Hence, on the event in \ref{eq:xsc}, every non-queried $\bx_i$ produces $\widehat f_i$ such that $\underset{\bX \sim \nu}{\prob} [\widehat f_i(\bX) \neq f_i(\bX)]>\epsilon$.

Let $F_i$ denote the event
$$F_i = \Big \{ \underset{\bX \sim \nu}{\prob} [\widehat f_i(\bX) \neq f_i(\bX)] \Big\}$$

We have shown that $$E_0 \cap \{\bx_i \in Q_0\} \subseteq F_i$$

Therefore, 

\begin{equation}\label{eq:gho}
    \mathds{1}_{E_0} \mathds{1}_{\bx_i \notin Q_0}  \le \mathds{1}_{F_i} 
\end{equation}

On the other hand, by definition $|Q_0| \le q$. Since there $\alpha_\epsilon$ star points, 

\begin{equation}\label{eq:wer}
    \underset{i \in [\alpha_\epsilon]}{\sum} \mathds{1}_{\bx_i \notin Q_0}    \ge \alpha_\epsilon -q
\end{equation}

\begin{align*}
    \underset{i \in [\alpha_\epsilon]}{\max}  \prob \big[\underset{\bX \sim \nu}{\prob} [\widehat f_i(\bX) \neq f_i(\bX)]>\epsilon \big] &\ge \frac{1}{\alpha_\epsilon} \underset{i \in [\alpha_\epsilon]}{\sum}  \prob \big[\underset{\bX \sim \nu}{\prob} [\widehat f_i(\bX) \neq f_i(\bX)]>\epsilon \big] \\
    &= \frac{1}{\alpha_\epsilon} \underset{i \in [\alpha_\epsilon]}{\sum}  \E[  \mathds{1}_{F_i}] \\
    &\ge  \frac{1}{\alpha_\epsilon} \underset{i \in [\alpha_\epsilon]}{\sum}  \E[   \mathds{1}_{E_0} \mathds{1}_{\bx_i \notin Q_0} ] \\
    &\ge \frac{\alpha_\epsilon - q }{\alpha_\epsilon} \prob[E_0]\\
    &\ge \frac{(\alpha_\epsilon - q)(1- \delta) }{\alpha_\epsilon}
\end{align*}

Where step 3 follows from \ref{eq:gho}, step 4 follows from \ref{eq:wer} , and the last step follows from \ref{eq:hh}.

The right part is bigger than $\delta$ when $\frac{ \alpha_\epsilon-q }{ \alpha_\epsilon }(1-\delta) > \delta,$

which is equivalent to $q > \frac{1-2\delta}{1-\delta}\alpha_\epsilon.$

\textbf{Step 2: Passive bound.}  Let $m=\shwartz(\cF)$, $S=(\bx_1,\ldots,\bx_m)$ DS-shattered CGQs,
and $B \subseteq \cF[S]$ an $m$-dimensional pseudo-cube. Let $\nu$ the coupling be defined as: $\nu = \frac1m \sum_{i=1}^m \delta_{\mathbf x_i}.$

For the lower bound we strengthen the learner by granting it direct
membership query access to any of the $m$ $\nu$-support points. We can note that a lower bound in this stronger model is also a lower bound for the
pool-based auditor.

Let the target $f^*$ be uniformly random from $B$. By Yao's minimax principle \citep{yao1977probabilistic}, it suffices to analyze a deterministic auditor.

Suppose the auditor has made $r$ distinct queries, with $r\le q$.
Condition on an arbitrary resulting interaction history, which we denote $T$.

Let $B_T = \{ \bb\in B: b\text{ is consistent with } T \}.$ Because the prior on $B$ is uniform and the auditor is deterministic, the conditional target is uniform on $B_T$.

Fix an non-queried coordinate $i\in[m].$ For every $\bb\in B_T$, the pseudo-cube property gives an $i$-neighbor
$\bb'\in B$ satisfying $b'_i\neq b_i$, and $b'_j=b_j$ for $ j\neq i$ and since coordinate $i$ has not been queried, $\bb'$ agrees with $\bb$ on every
queried coordinate hence $\bb'\in B_T$.

For $\bb =(b_1,b_2, \cdots, b_{i-1},b_i, b_{i+1}, \cdots,b_m)$, let $\pi_i(\bb) = (b_1,b_2, \cdots, b_{i-1}, b_{i+1}, \cdots,b_m)$ denote the projection removing the $i^{\text{th}}$ coordinate from $\bb$. For $g \in \pi_i(B_T)$, we define a fiber $F_g = \{\bb \in B_T: \pi_i(\bb) = g\}$

Partition $B_T$ into these fibers. Every such fiber contains at least two traces and within one fiber, all labels at coordinate $i$ are distinct;
if two traces had the same value at coordinate $i$ and agreed on all
other coordinates, they would be identical.

Therefore, for any label $\by_i$ predicted by the auditor at
$\bx_i$, in every fiber at most one trace has $b_i=\by_i.$

By definition of the pseudo-cube, every fiber has cardinality at least two and since conditioned on $T$ , the target $f^*$ is uniform on $B_T$, it is also uniform within each fiber conditioned on membership of that fiber, by total law of probability we obtain

$$\prob \left[ \widehat f(\bx_i)\neq f^*(\bx_i) \mid T \right] \ge \frac12.$$

There are at least $m-r\ge m-q$ non-qeried coordinates.
Hence
\begin{align*}
\E
\left[
\underset{\bX \sim \nu}{\prob} [\widehat f(\bX) \neq f^*(\bX)
\mid T
\right]
&=
\frac1m
\sum_{i=1}^m
\prob
\left(
\widehat f(\mathbf x_i)\neq f^*(\mathbf x_i)
\mid T
\right)
\\
&\ge
\frac{m-q}{2m}.
\end{align*}

Taking expectation over history,

\begin{equation}\label{eq:had}
    \E
\left[
\underset{\bX \sim \nu}{\prob} [\widehat f(\bX) \neq f^*(\bX)
\right]
\ge
\frac{m-q}{2m}.
\end{equation}
Now suppose the auditor were $(\epsilon,\delta)$-PAC for every target in
the pseudo-cube.
Averaging over the uniform prior on $B$ gives

$$ \prob \left\{ \underset{\bX \sim \nu}{\prob} [\widehat f(\bX) \neq f^*(\bX)]>\epsilon \right\} \le \delta$$
Since
\[
0\le \prob \left\{ \underset{\bX \sim \nu}{\prob} [\widehat f(\bX) \neq f^*(\bX)]>\epsilon \right\} \le \delta)\le1,
\]
we then have
\begin{align*}
\mathbb E
[\underset{\bX \sim \nu}{\prob} [\widehat f(\bX) \neq f^*(\bX)]]
&\le
\epsilon
\prob\{\underset{\bX \sim \nu}{\prob} [\widehat f(\bX) \neq f^*(\bX)]\le\epsilon\}
+
\prob\{\underset{\bX \sim \nu}{\prob} [\widehat f(\bX) \neq f^*(\bX)]>\epsilon\}
\\
&\le
\epsilon(1-\delta)+\delta
\\
&=
\epsilon+\delta(1-\epsilon).
\end{align*}
Combining with \ref{eq:had},
\[
\frac{m-q}{2m}
\le
\epsilon+\delta(1-\epsilon).
\]

Therefore $$q \ge m\left[1-2\epsilon-2\delta(1-\epsilon)\right]$$
\end{proof}

\newpage

\subsection{Proof of \textsc{ALeBi} Upper Bounds}\label{app:ubmanipfree}

\alebiUB*

\begin{proof}
We organize the proof as follows: we begin by showing the active terms in the upper bound by adapting the proof of \cite{hanneke2014theory} to the multi-class learning problem, since  the probe learning problem reduces to multiclass learning where the size of labels depend on the cardinal of groups $k$.

\paragraph{Step 1: The passive term.}

Here, the passive learner is allocated failure probability $\delta/2$.
Fix $n = \left\lceil \frac{\shwartz +\log(4/\delta) }{ \epsilon } \right\rceil .$

By Theorem~\ref{thm:pabbaraju-passive}, the passive term  derives directly by drawing $n$ CGQ queries $\cQ_{1:m} = (\bX_1,\ldots,\bX_n)$ from the fixed coupling $\nu$ and obtaining their probe labels via $\modelowner$, which guarantees:

By the choice of $n$, the multiclass passive algorithm proposed in \cite{pabbaraju2026optimal} outputs $\hat{f}_{\cQ_{1:n}}$ such that: 

$$\underset{\cQ_{1:n}}{\prob}[\underset{\bX \sim \nu}{\prob} [\widehat f_{\cQ_{1:n}}(\bX) \neq f^*(\bX)] > \epsilon] \le \frac{\delta}{2} < \delta$$

In the following, fix $S \triangleq \{\bx_1, \bx_2, \cdots, \bx_n\}$ as the sample set used for passive learning, with $n = \left\lceil \frac{\shwartz +\log(4/\delta) }{ \epsilon } \right\rceil $.

\paragraph{Step 2: First active term.}

We assume that pool of CGQ queries is index by some set $I$. 
If $I\subseteq[N]$ denotes the set of indices that $\auditor$ have seen in the history of interaction with $\modelowner$, the corresponding version space is given by $\cV_I = \left\{ f\in\cF_k: f(\bx_j)=f^\star(\bx_j), \forall j\in I \right\},$ where $f^\star = f_{h^\star}$, $h^\star$ being the black-box model under audit.

For a nonempty set $I$, we say that $i\in I$
is an effective (informative) index if it exhibits disagreement within the multi-context class with respect to $f^ \star$, meaning there exists $f_i\in\cF_k$ such that for all $j\in I\setminus\{i\}$, $f_i(\bx_j) = f^\star(\bx_j)$ while $f_i(\bx_i)
\neq
f^\star(\bx_i).$ Let $E(I)$ denote such set of effective coordinates.

It is straightforward to see that $|E(I)| \le \min\{\mathfrak s_k,|I|\}$

Let $f_0, \{f_i\}_{i \in I}$ be such that they exhibit this disagreement within $\cF_k$, where $f_0 = f^\star$

For $i \neq j\in E(I)$, we have $f_i(\bx_j)=f^\star(\bx_j)$ while $f_i(\bx_i)\neq f^\star(\bx_i)$

  If $\bx_i=\bx_j$ for $i\neq j$, then a
witness for coordinate $i$ would have to both agree and disagree with
$f^\star$ at the same CGQ. Therefore the CGQs indexed by the effective set $E(I)$ must be distinct.

which shows that $\{\bx_i:i\in E(I)\}
$ is a multi-star set for $\cF_k$ centered at $f_0 = f^\star$.
By definition of $\mathfrak s_k$, $|E(I)|\le\mathfrak s_k$.

For an integer $n$ \footnote{$n$ here is arbitrary, we'll later use $n$ as the sample complexity for passive learning probes.}, let $s_n \triangleq \min\{\mathfrak s_k,n\}.$
For $I\subseteq[n]$, let $N(I)$ be the cardinal of CGQ queries when the
coordinates of $I$ are processed in a uniformly random order, and let
$\Psi(I) \triangleq \E \left[ e^{\lambda N(I)} \,\middle|\, S \right], $ where $ \lambda\ge0$.

We show by induction on $m=|I|$ that the following holds $$\Psi(I)\le\prod_{r=1}^m\left[1+(e^\lambda-1)\min\left\{1,\frac{s_n}{r}\right\}\right]$$

For $m=0$ the resulting identity is trivial.

Let $|I|=m\ge1$.  In a uniformly random ordering of $I$, the last
coordinate is uniformly distributed over $I$.

Conditionned on coordinate $i$ being last.  After all coordinates in
$I\setminus\{i\}$ have been processed, the current version space is $$ \left\{ f: f(\bx_j)=f^\star(\bx_j) \ \forall j\in I\setminus\{i\} \right\}$$
Therefore the final CGQ cooreponding to index $i$ is queried if and only if $i\in E(I)$.  Hence

$$N(I) = N(I\setminus\{i\}) + \mathds 1\{i\in E(I)\}$$

Averaging over the uniformly random last coordinate yields

$$\Psi(I) = \frac1m \sum_{i\in I} e^{\lambda\mathbf 1\{i\in E(I)\}} \Psi(I\setminus\{i\}).$$

Applying the induction hypothesis to each
$I\setminus\{i\}$ yields

$$\Psi(I) \le \Big(\prod_{r=1}^{m-1} \left[ 1+ (e^\lambda-1) \min\left\{1,\frac{s_n}{r}\right\} \right]\Big)  \Big( \frac1m \sum_{i\in I} e^{\lambda\mathbf 1\{i\in E(I)\}}\Big) $$

On the other hand, 
$$ \frac1m \sum_{i\in I} e^{\lambda\mathbf 1\{i\in E(I)\}} = 1+ (e^\lambda-1)\frac{|E(I)|}{m}$$
and since $|E(I)| \le \min\{\mathfrak s_k,|I|\}$, we have $\frac{|E(I)|}{m}\le\min\left\{1,\frac{s_n}{m}\right\}$
we conclude the induction.

For $I=[n]$ with the fact that $1+t\le e^t,$ we have $$\mathbb E \left[ e^{\lambda N_n} \,\middle|\, S \right] \le \exp \left( (e^\lambda-1)\sum_{r=1}^{n} \min\left\{1,\frac{s_n}{r}\right\} \right),$$
where $N_n \triangleq N([n])$.

For $\lambda=1$, and by Markov's inequality,
\[
\begin{aligned}
\prob
\left\{
N_n>b
\,\middle|\,
S
\right\}
&=
\prob
\left\{
e^{N_n}>e^b
\,\middle|\,
S
\right\}
\\
&\le
\exp
\left(
(e-1)\min\left\{1,\frac{s_n}{r}\right\} -b
\right).
\end{aligned}
\]

Taking $b=(e-1)\min\left\{1,\frac{s_n}{r}\right\}+\log\frac{2}{\delta}$, for every realized pool $S$,
\[
\prob
\left\{
N_n>b
\,\middle|\,
S
\right\}
\le
\frac{\delta}{2}
\]

The auditor $\auditor$ runs $k$-ALeBi for a budget of $b$ under the condition that if it is about to make query number $b+1$, it halts. Thus the truncated procedure makes at most
$b$ oracle queries on every execution.

If truncation does not occur, $k$-ALeBi has reconstructed every response in the original i.i.d. sample $S$ used for the probe passive leaning, and hence has recovered exactly the target probe.
The passive term in \textbf{Step 1} is then applied to this completed labeled sample and requires no additional oracle queries.

The failure events are therefore $Z_{\text{alebi}} = \{\text{$k$-ALeBi is truncated}\}$ for interactive setting

and $Z_{\mathrm{pass}} = \left\{ \underset{\bX \sim \nu}{\prob} [\widehat f_{S}(\bX) \neq f^*(\bX)] >\epsilon \right\}$ for passive setting.

In step 1, we have shown that $\prob(Z_{\mathrm{pass}})\le\frac{\delta}{2}$
And in step 2, we have shown that: $\prob(Z_{\mathrm{alebi}})\le\frac{\delta}{2}$
Therefore, by the union bound,
$$\underset{\cQ_{1:n}}{\prob}[\underset{\bX \sim \nu}{\prob} [\widehat f_{\cQ_{1:n}}(\bX) \neq f^*(\bX)] > \epsilon]  < \delta$$

Finally $1\le s_n\le n,$
then
\[
\begin{aligned}
\sum_{r=1}^{n} \min\left\{1,\frac{s_n}{r}\right\}
&=
\sum_{m=1}^{s_n}1
+
\sum_{m=s_n+1}^N\frac{s_n}{m}
\\
&\le
s_n
+
s_n\log\frac{n}{s_n}
\\
&=
s_n\log\frac{en}{s_n}.
\end{aligned}
\]

In particular, for $\mathfrak s_k\le n$, this yields a sample complexityb of 

\begin{equation}\label{eq:azkfgehz}
    (e-1)\mathfrak s_k
\log\frac{e n}{\mathfrak s_k}
+
\log\frac{2}{\delta}
+
1
\end{equation}

with $n= \left\lceil \frac{\shwartz +\log(4/\delta) }{ \epsilon } \right\rceil .$

\paragraph{Step 3: Second active term.}
We next obtain a deterministic exact-completion procedure on the same finite pool using the same fixed coupling over groups marginals. Let $\cV_S = \left\{ (f(\bx_1),\ldots,f(\bx_n)): f\in\cF_k \right\}$ denote the version space consistent with $S$.

Recall the definitions of specifying sets and extended teaching dimension in DEfinition \ref{def:xtd}. For an arbitrary labeling $g\in\mathcal Y_k^n,$ $S$ is a specifying set for $g$ with respect to $\cV_S$ if $\left| \{ v\in \cV_S:v|_S=g|_S \} \right| \le1.$ The extended teaching dimension is
\[
\operatorname{XTD}(\cV_S)
=
\max_{g\in\mathcal Y_k^n}
\min
\left\{
|S|:
S\text{ specifies }g\text{ with respect to }\cV_S
\right\}.
\]

We first establish the relation between XTD and the multiclass star
number that is needed in the algorithm:

$$\operatorname{XTD}(\cV_S)
\le
\min\{
\mathfrak s_k(\cV_S)+1,n
\}
\le
\min\{
\mathfrak s_k+1,n
\}$$

The bound by $n$ is trivial since all $n$ coordinates specify any
labeling with respect to $\cV_S$.

To prove the first inequality, fix an arbitrary $g\in\mathcal Y_k^n$
and let $S=\{i_1,\ldots,i_m\}$ be a specifying set of minimum cardinality.

If $m=0$ the bound terms become trivial. For every $r\in[m]$, minimality implies that $S\setminus\{i_r\}$
is not specifying.  Consequently there exist at least two distinct
traces in $\cV_S$ agreeing with $g$ on
$S\setminus\{i_r\}$.

Because $S$ itself is specifying, at most one of those traces can
also agree with $g$ at $i_r$.  We can therefore choose
$v_r\in \cV_S$ such that $v_r(i_r)\neq g(i_r)$
and for all $\ell\neq r$, $v_r(i_\ell)=g(i_\ell)$.

If $m\ge2$, let $v_1$ as a center and the
$m-1$ indices $i_2,\ldots,i_m.$

For every $r\ge2$, $v_r(i_r)\neq g(i_r)=v_1(i_r),$
while for every $\ell\ge2$ with $\ell\neq r$, $v_r(i_\ell)=g(i_\ell)=v_1(i_\ell).$

Hence $ i_2,\ldots,i_m$ is a multi-star set for $\cV_S$, centered at $v_1$, and consequently, $m-1\le\mathfrak s(\cV_S),$
Therefore $m\le\mathfrak s(\cV_S)+1.$

The same inequality is trivial for $m=1$.  Since $g$ was arbitrary $\operatorname{XTD}(\cV_S)
\le
\mathfrak s(\cV_S)+1.$ Finally, a star in $\cV_S$ lifts to a star on the corresponding CGQs in
$\mathcal F_k$, and therefore $\mathfrak s(\cV_S)\le\mathfrak s_k$

This proves 

$$\operatorname{XTD}(\cV_S)
\le
\min\{
\mathfrak s_k+1,n
\}$$

\medskip
\noindent

\begin{algorithm}[H]
\caption{\textsc{Probes Membership-Halving}}
\label{alg:rel-mh}
\begin{algorithmic}[1]

\Require Budget $m$, finite trace $\cV_S$.

\State Intialization:
\[
\cV \gets \cV_S.
\]

\While{$|V|>1$}

    \For{$i=1,\ldots,m$}

        \State Choose a maximizing response
        \[
        \phi_{\cV}(i)
        \in
        \arg\max_{y\in\mathcal Y_k}
        |\{v\in \cV:v(i)=y\}|.
        \]

        \State 
        \[
        a_i
        :=
        |\{v\in \cV:v(i)=\phi_{\cV}(i)\}|.
        \]

    \EndFor

    \State Choose a smallest specifying set $S$ for $\phi_{\cV}$ with respect to $V$.

    \State Choose
    \[
    i^\star
    \in
    \arg\max_{i\in S}
    (|V|-a_i).
    \]

    \State Query
    \[
    y^\star
    \gets
    f^\star(\mathbf X_{i^\star}).
    \]

    \State Update
    \[
    \cV
    \gets
    \{v\in \cV:v(i^\star)=y^\star\}.
    \]

\EndWhile

\State Let $v^\star$ denote the unique surviving trace in $\cV$.

\State \Return
\[
\left(
(\mathbf X_i,v^\star(i))
\right)_{i=1}^{n}.
\]

\end{algorithmic}
\end{algorithm}
Next we show that Algorithm \ref{alg:rel-mh} uses at most $m = 1+\operatorname{XTD}(\cV_S)\log|\cV_S|$ CGQs.

Let $V\subseteq \cV_S$ denote the current trace version space.
With the initialization step $\cV=\cV_S,$ for every index $i\in[n]$, Algorithm \ref{alg:rel-mh} chooses a maximizing response $ \phi_{\cV}(i)
\in
\arg\max_{y\in\mathcal Y_k}
|\{v\in \cV:v(i)=y\}|$
and let $a_i = |\{v\in \cV:v(i)=\phi_{\cV}(i)\}|.$

Since restricting a class cannot increase XTD, $\operatorname{XTD}(\cV)
\le \operatorname{XTD}(\cV_S).$

Hence there exists a set $S\subseteq[n]$ such that $
|S|\le\operatorname{XTD}(\cV_S)$
that specifies the labeling $\phi_{\cV}$.

At most one trace in $V$ agrees with $\phi_{\cV}$ on all coordinates in
$S$ thus every other trace disagrees with $\phi_{\cV}$ somewhere in
$S$, yielding $\sum_{i\in S}
(|\cV|-a_i)
\ge
|\cV|-1$

Thus there exists $i^\star\in S$ such that

\begin{equation}\label{eq:kuusd}
    |\cV|-a_{i^\star}
\ge
\frac{|\cV|-1}{\operatorname{XTD}(\cV_S)}.
\end{equation}

Algorithm \ref{alg:rel-mh} queries $y^\star = f^\star(\bx_{i^\star})$
and retain $\cV' = \{v\in \cV:v(i^\star)=y^\star\}.$

We distinguish between two cases;

If $y^\star=\phi_{\cV}(i^\star),$
then $|\cV'|=a_{i^\star}$
and inequality \ref{eq:kuusd} gives
\[
|\cV'|-1
\le
\left(
1-\frac1{\operatorname{XTD}(\cV_S)}
\right)(|\cV|-1)
\]

On the other hand, if $y^\star\neq\phi_{\cV}(i^\star)$
then the surviving response cell isn't a maximizing cell.
Since a maximizing cell has cardinality at least that of every other
cell $|\cV'|
\le
\frac{|\cV|}{2}$

When $\operatorname{XTD}(\cV_S)\ge2$,
\[
\frac{|\cV|}{2}-1
\le
\left(
1-\frac1{\operatorname{XTD}(\cV_S)}
\right)(|\cV|-1)
\]

If $\operatorname{XTD}(\cV_S)=1$, one specifying coordinate leaves at most one surviving
trace, so exact identification occurs after one query.

Thus, whenever $|\cV|>1$,
\[
|\cV'|-1
\le
\left(
1-\frac1{\operatorname{XTD}(\cV_S)}
\right)(|V|-1)
\]

After $q$ queries,
\[
|\cV_q|-1
\le
\left(
1-\frac1{\operatorname{XTD}(\cV_S)}
\right)^q
(|\cV_S|-1)
\le
e^{-q/\operatorname{XTD}(\cV_S)}
(|\cV_S|-1)
\]

If $q>\operatorname{XTD}(\cV_S)\log(|\cV_S|-1),$
then the right-hand side is strictly smaller than $1$.
Since $|\cV_q|-1$ is a nonnegative integer, it must equal zero.
Therefore $\cV_q$ contains exactly one trace.

Since the target trace is maintained via consistency with version space, that unique trace is exactly $(f^\star(\bx_1),\ldots,f^\star(\bx_n)).$

Hence the sample complexity of Algorithm \ref{alg:rel-mh} is bounded by $1+\operatorname{XTD}(\cV_S)\log|\cV_S|$, and since$ \operatorname{XTD}(\cV_S)
\le
\min\{
\mathfrak s_k+1,n
\}$, the sample complexity is bounded by $1+\min\{\mathfrak s_k+1,n\}\log|\cV_S|$

If $\Pi_{\mathcal F_k}(n)$ denotes the growth function over n CGQs, using the fact that $|\cV_S|
\le
\Pi_{\cF_k}(n),$ uniformly over every possible pool, we get a sample complecity of 

\begin{equation}\label{eq:jhfd}
    \cO(1+
\min\{\mathfrak s_k+1,n\}
\log\Pi_{\mathcal F_k}(n))
\end{equation}

ALgorithm \ref{alg:rel-mh} has no truncation failure as it reconstructs the
CGQ pool exactly. Thus the only failure event is that of
the passive multiclass learner whose probability is already at most $\delta/2$ from \textbf{step 1}.

A sharper result presented below is derived to improve the dependence on the extended teaching dimension by a logarithmic factor.

\begin{lemma}[Refined bounds for Algorithm \ref{alg:rel-mh}] \label{thm:refined-membership-halving} For every finite trace class $\cV \subseteq\mathcal Y^n$  there exists a deterministic query procedure that exactly identifies every target $v^\star\in \cV$ using $\cO\!\left( \frac{\operatorname{XTD}(\cV)} {1\vee\log\operatorname{XTD}(\cV)} \log|\cV| \right)$ queries. \end{lemma}

\begin{proof}

For arbitrary label spaces $\operatorname{XTD}(\cV_S)
\le
\mathfrak s(\cV_S)$
and restricting on consitency with $S$ doesn't increase the star number $\mathfrak s(\cV_S)
\le
s_k.$
Since $\operatorname{XTD}(\cV_S)\le n$, then $\operatorname{XTD}(\cV_S)\le\bar s_k.$

We now apply the refined multiclass Membership Halving procedure. At the beginning of interaction, let $\cV$ be the current version space and
let $\phi_{\cV}(i)
\in
\arg\max_y
|\{v\in \cV:v(i)=y\}|$
be its maximizing labeling.

Fix a specifying set $S$ for $\phi_{\cV}$ of size at most $\operatorname{XTD}(\cV_S)$, query the remaining coordinate maximizing $|\{v\in V:v(i)\neq\phi_{\cV}(i)\}|,$
keeping $\phi_{\cV}$ fixed throughout the phase.

Suppose a phase starts with $n$ traces and the first contradiction to
the fixed plurality labeling occurs on its $q$-th query.

Let $A_j$ be the set of traces eliminated by a maximizing-consistent answer at
the $j$-th queried coordinate.

The sets $A_1,\ldots,A_q$
are pairwise disjoint, and by construction of algorithm \ref{alg:rel-mh}, we end up constraining a chain $|A_1|\ge|A_2|\ge\cdots\ge|A_q|.
$
Therefore $q|A_q|
\le
\sum_{j=1}^q|A_j|
\le n.$

After the contradiction, the surviving exact response cell is a
subset of $A_q$, therefore $|\cV'|
\le
\frac nq.$

If no contradiction occurs on the entire specifying set, the target
trace is uniquely identified.

Suppose  algorithm \ref{alg:rel-mh} performs $P$ phases.  For phase $p$, let
$n_{p-1}$
be the size of the version space at the beginning of the phase,
let $n_p$
be its size at the end of the phase, and let $q_p$
be the number of queries made during that phase.Then there exists a universal
constant $c>0$ such that
\begin{equation}\label{eqnnnn}
    \log\frac{n_{p-1}}{n_p}
\ge
c\,\operatorname{Log}(q_p),
\end{equation}
where $\operatorname{Log}(t) \triangleq 1\vee\log t.$

Moreover, since the phase queries only coordinates from a specifying
set,
\[
q_p
\le
\operatorname{XTD}(\cV_{p-1})
\le
\operatorname{XTD}(\cV_S)
\]

The function $t\longmapsto
\frac{t}{\operatorname{Log}(t)}$
is nondecreasing for $t\ge1$. Hence, from $q_p\le \operatorname{XTD}(\cV_S)$,
\[
\frac{q_p}{\operatorname{Log}(q_p)}
\le
\frac{\operatorname{XTD}(\cV_S)}{\operatorname{Log}(\operatorname{XTD}(\cV_S))}.
\]
Equivalently,
\[
q_p
\le
\frac{\operatorname{XTD}(\cV_S)}{\operatorname{Log}(\operatorname{XTD}(\cV_S))}
\operatorname{Log}(q_p).
\]

With the result of \ref{eqnnnn},
\[
q_p
\le
\frac{1}{c}
\frac{\operatorname{XTD}(\cV_S)}{\operatorname{Log}(\operatorname{XTD}(\cV_S))}
\log\frac{n_{p-1}}{n_p}.
\]

Summing over all $P$ phases gives sample complexity of 
\[
\begin{aligned}
\sum_{p=1}^{P}q_p &\le
\frac{1}{c}
\frac{\operatorname{XTD}(\cV_S)}{\operatorname{Log}(\operatorname{XTD}(\cV_S))}
\sum_{p=1}^{P}
\log\frac{n_{p-1}}{n_p}
\\
&=
\frac{1}{c}
\frac{\operatorname{XTD}(\cV_S)}{\operatorname{Log}(\operatorname{XTD}(\cV_S))}
\log\frac{n_0}{n_P}.
\end{aligned}
\]
The algorithm starts from $n_0=|\cV_S|$
and terminates when a unique trace remains, that is $n_P=1.$
Therefore a sample complexity of
$\frac{\operatorname{XTD}(\cV_S)}
{\operatorname{Log}\operatorname{XTD}(\cV_S)}
\log|\cV_S|$ suffices

\end{proof}

Applying lemma~\ref{thm:refined-membership-halving} with $|\cV_S| \le \Pi_{\mathcal F_k}(n), $ and therefore a set of size $ O\!\left( \frac{ \min\{\mathfrak s_k,n\} }{ 1\vee \log\min\{\mathfrak s_k,n\} } \log\Pi_{\mathcal F_k}(n) \right)$ suffices. 

On the other hand, a binary point $v=(v_1,\ldots,v_k)$ from the cube $\{0,1\}^k$ induces probe response $T(v)=(v_i-v_j)_{i<j}$

If for two points $v,w$, one has $T(v)=T(w)$, then for all $i<j$, $v_i-w_i=v_j-w_j$, making $v-w$ constant.  The only distinct pair
with this property is $(v,w) =(0^k,1^k)$. Hence the CGQ response alphabet has size $2^k-1.$

We have $\shwartz(\cV_S)\le\min\{\shwartz(\cF),n\}$, applying the multiclass Sauer lemma to the finite class $\cV_S\subseteq[2^k-1]^N$ yields
\[
|\cV_S|
\le
\sum_{j=0}^{\min\{\shwartz,n\}}
\binom nj
(2^k-1-1)^j
=
\sum_{j=0}^{\min\{\shwartz,n\}}
\binom nj
(2^k-2)^j.\]

Consequently,
for $\min\{\shwartz,n\}\ge1$,
\[
\log|\cV_S|
\le
\min\{\shwartz,n\}
\log
\left(
\frac{
en(2^k-2)
}{
\min\{\shwartz,n\}
}
\right)
\]

Implement this result over \ref{eq:jhfd} gives the sample complexity of 
\[
\min\{\mathfrak s_k,n\}
\min\{\shwartz,n\}
\log
\left(
\frac{
en(2^k-2)
}{
\min\{\shwartz,n\}
}
\right)
\]

With $n= \left\lceil \frac{\shwartz +\log(4/\delta) }{ \epsilon } \right\rceil .$

Combining the bounds derived in step1,2 and 3 yields the result.

\end{proof}

\newpage

%% file: appendix/manip.tex
\newpage
\section{Proofs for Fairness-aware Manipulation Regime} \label{app:robust-proofs}

We begin by showing the following lemma

\begin{lemma}[Accuracy of one neighborhood simulation]
\label{lem:local-majority}
Under Assumptions~\ref{ass:local-replicability} and~\ref{ass:attack},
\[
\underset{q \sim \nu}{\prob}\!\left(
\widehat f(q)\neq f^{\star}(q)
\right)
\le
\exp\!\left(
-\frac{R(1-2\beta)^2}{2}
\right).
\]
\end{lemma}

\begin{proof}
Fix $q$ a CGQ query and define $Z_r \triangleq
\mathds 1\{\widetilde Y_r\neq f^{\star}(q)\}.$
\paragraph{Step 1: bound the conditional error probability.}
Let $m_r \triangleq
P\!\left(
f^{\star}(q'_r)\neq f^{\star}(q)
\,\middle|\,
\bbH_{r-1}
\right)$
Since $q'_r$ is a fresh draw from $P_Q(\cdot\mid J(q))$,
Assumption~\ref{ass:local-replicability} gives $m_r\le\rho.$

If the clean local response differs from $f^{\star}(q)$, the prob error is bounded by $1$. Otherwise an error can occur only through
an attack, whose conditional probability is at most $p$. Hence
\[
\begin{aligned}
\mathbb E[Z_r\mid\bbH_{r-1}]
&\le
m_r+(1-m_r)p \\
&\le
\rho+(1-\rho)p
=
\beta
\end{aligned}
\]

\paragraph{Step 2: Failure mode vs the number of errors.}
If $\sum_{r=1}^R Z_r<\frac R2$
then strictly more than half of the responses equal $f^{\star}(q)$.
Therefore $f^{\star}(q)$ is the unique mode. Thus
\[
\{\widehat Y(q)\neq f^{\star}(q)\}
\subseteq
\left\{
\sum_{r=1}^R Z_r\ge\frac R2
\right\}
\]

\paragraph{Step 3: Concentration of bad events.}
Fix $\mu_r:=\mathbb E[Z_r\mid\bbH_{r-1}]$ and $D_r:=Z_r-\mu_r.$
We have 
\[
\mathbb E[D_r\mid\bbH_{r-1}]=0,
\]
and, conditionally on $\bbH_{r-1}$, $D_r$ lies in an interval of
length one. The conditional Hoeffding lemma therefore gives
\[
\mathbb E\!\left[
e^{\lambda D_r}
\,\middle|\,
\bbH_{r-1}
\right]
\le
e^{\lambda^2/8}.
\]
Iterating this bound and applying Chernoff yields
\begin{equation}\label{eq:robi}
    \prob\!\left(
\sum_{r=1}^R D_r\ge t
\right)
\le
\exp\!\left(-\frac{2t^2}{R}\right)
\end{equation}

Since $\mu_r\le\beta$,
\[
\left\{
\sum_{r=1}^R Z_r\ge\frac R2
\right\}
\subseteq
\left\{
\sum_{r=1}^R D_r
\ge
R\left(\frac12-\beta\right)
\right\}.
\]
Replacing $t=R\left(\frac12-\beta\right)$
in \ref{eq:robi} yields
\[
\prob\!\left(
\sum_{r=1}^R Z_r\ge\frac R2
\right)
\le
\exp\!\left(
-\frac{R(1-2\beta)^2}{2}
\right).
\]
Combining this with step 2 proves the result.
\end{proof}
\robustalebiub*

\begin{proof}
Fix $p>0$ and set $N=N_0(\epsilon,\delta/2).$

Consider the $t$-th CGQ requested by \textsc{$k$-ALeBi}, conditionned on
all previous CGQ responses being uncorrupted. By
Lemma~\ref{lem:local-majority} and the definition of $R$,
\[
\prob\!\left(
\widehat Y_t\neq f^{\star}(q_t)
\,\middle|\,
\text{all previous CGQs are clean}
\right)
\le
\frac{\delta}{2N}
\]
There are at most $N$ requested CGQs. Therefore
\[
\prob\!\left(
\text{at least one CGQ response is corrupted}
\right)
\le
\frac{\delta}{2}
\]

When running \textsc{Robust $k$-ALeBi} and \textsc{$k$-ALeBi} with the same unlabeled sample and
the same internal randomization, on the event that every CGQ
response is correct, and both executions receive the same interaction history with $\modelowner$
and therefore return the same output.

Let $F_0$ be the failure event of this clean execution. By assumption, we have $\prob(F_0)\le\frac{\delta}{2}.$
Hence the failure event of the robust algorithm is contained in $\{\text{there exists CGQ response that was corrupted}
\}
\cup F_0.$
A union bound proves the $(\epsilon,\delta)$-PAC guarantee of learning the probe.

Finally, by definition of $N$, at most $N$ clean CGQ requests are sent to $\modelowner$, each using exactly $R$ oracle attacks. This proves the upper bounds sample complexity. Each accepted local CGQ
requires at most $1/\gamma$ unlabeled draws in expectation, which gives
a sample complexity of $\left\lceil
\frac{2N_0(\epsilon,\delta/2)}{\gamma(1-2\beta)^2}
\log
\frac{
2N_0(\varepsilon,\delta/2)
}{\delta}
\right\rceil$.

\end{proof}

\robustLB*

\begin{proof}

We first prove the lower bound when the $\modelowner$ is honest.

We fix a fairness-effective star of size $m$ with center $f_0$, witnesses
$f_1,\ldots,f_m$, and support $q_1,\ldots,q_m$, and let $\nu
=
\frac1m
\sum_{j=1}^m\delta_{q_j}$.

Since $f_0$ is zero on the entire support, $\mu(f_0)=0.$
For every $i$, $f_i$ is zero except at $q_i$. Since
$f_i(q_i)\neq\mathbf 0$ and individual coordinates belong to
$\{-1,0,1\}$, at least one coordinate has one in abs. Thus $\mu(f_i)=\frac1m.$
By the choice of $m$, $\frac1m\ge4\varepsilon.$

Now fix $i$ and couple executions under $f_0$ and $f_i$ using the same unlabeled sample and the same internal randomness. Until $q_i$ is queried, the two histories are identical. Let $A_i$ be the event that $q_i$ is never queried. On $A_i$, both
executions return the same estimate. But $\frac1m\ge4\varepsilon$, thus $\epsilon$-accuracy
intervals around $0$ and $1/m$ are disjoint. Therefore the
common estimate cannot be correct under both targets.

Since each execution fails with probability at most $\delta$, $\nu_{f_0}(A_i)\le2\delta.$
Thus $$\nu_{f_0}(q_i\text{ is queried})
\ge
1-2\delta$$

Summing over queries, $\mathbb E_{f_0}
\left[
|\{\text{distinct queried support points}\}|
\right]
\ge
m(1-2\delta).$
The number of distinct queried CGQs is at most the total number of
oracle requests, proving the desired result.

To prove the resulting lower bound in the fairness aware adversarial regime, one can observe that the first term follows from the same indistinguishability argument as in
Lemma~\ref{lem:local-majority}. We prove the attack-dependent term.

\paragraph{Step 1: construct two attacked environments.}
Fix $i\in[m]$. Under $P_0$, the target is $f_0$. On $J_i$, the clean
response is $\mathbf 0$, and the attack returns $f_i(q_i)$ with probability
$p$.

Under $P_i$, the target is $f_i$. On $J_i$, the clean response is
$f_i(q_i)$, and the attack returns $\mathbf 0$ with probability $p$.
Outside $J_i$, use the same observation law under both environments.

After identifying the two possible responses on $J_i$ with $\{0,1\}$,
one observation from $C_i$ has law
\[
\operatorname{Bern}(p)
\quad\text{under }P_0,
\qquad
\operatorname{Bern}(1-p)
\quad\text{under }P_i.
\]

\paragraph{Step 2: compute the history divergence.}
Let $N_i$ be the number of oracle requests in $J_i$. Since the two
environments differ only there, the chain rule for adaptive KL divergence
gives

\begin{equation}\label{eq:iuzehf}
    D_{\rm KL}(P_0\|P_i)
=
\mathbb E_0[N_i]D_{\rm KL} \left( \operatorname{Bern}(p) \,\middle\|\, \operatorname{Bern}(1-p) \right)
\end{equation}

\paragraph{Step 3: Distinguish the environments.}
As proved previously, $\mu(f_0)=0,$ and $\mu(f_i)=\frac1m\ge4\varepsilon.$

Thresholding $\widehat\mu$ at $1/(2m)$ gives a binary test between $P_0$
and $P_i$ with both error probabilities at most $\delta$. By data
processing,
\[
D_{\rm KL}(P_0\|P_i)
\ge
D_{\rm KL}
\left(
\operatorname{Bern}(\delta)
\,\middle\|\,
\operatorname{Bern}(1-\delta)
\right)
\]
With \ref{eq:iuzehf} gives
\[
\mathbb E_0[N_i]
\ge
\frac{D_{\rm KL}
\left(
\operatorname{Bern}(\delta)
\,\middle\|\,
\operatorname{Bern}(1-\delta)
\right)}{D_{\rm KL} \left( \operatorname{Bern}(p) \,\middle\|\, \operatorname{Bern}(1-p) \right)}.
\]

Summing over cells for $i=1,\ldots,m$,
\[
\mathbb E_0[T]
\ge
m\frac{D_{\rm KL}
\left(
\operatorname{Bern}(\delta)
\,\middle\|\,
\operatorname{Bern}(1-\delta)
\right)}{D_{\rm KL} \left( \operatorname{Bern}(p) \,\middle\|\, \operatorname{Bern}(1-p) \right)}.
\]
Combining this with the indistinguishability lower bound proves
$$\mathbb E[T]
\ge
m
\max\left\{
1-2\delta,
\frac{D_{\rm KL}
\left(
\operatorname{Bern}(\delta)
\,\middle\|\,
\operatorname{Bern}(1-\delta)
\right)}{D_{\rm KL}
\left(
\operatorname{Bern}(p)
\,\middle\|\,
\operatorname{Bern}(1-p)
\right)}
\right\}$$

$$\mathbb E[T]
\ge
m
\max\left\{
1-2\delta,
\frac{(1- 2 \delta) \log \frac{1- \delta}{\delta}}{(1- 2 p) \log \frac{1- p}{p}}
\right\}$$.
\end{proof}

As the corruption increases and $p$ gets closer to $1/2$,
\[
D_{\rm KL} \left( \operatorname{Bern}(p) \,\middle\|\, \operatorname{Bern}(1-p) \right)
=
2(1-2p)^2
+
O((1-2p)^4),
\]
while for $\delta$ close to $O$
\[
D_{\rm KL}
\left(
\operatorname{Bern}(\delta)
\,\middle\|\,
\operatorname{Bern}(1-\delta)
\right)
=
\Theta(\log(1/\delta))
\]
THus, in the high-noise regime and high-confidence regime, the sample complexity becomes comparable to the one derived in Theorem \ref{th:SCLB} $\Omega\!\left(
\min\left\{
\mathfrak s_\mu(\mathcal F),
\frac1\varepsilon
\right\}
\frac{\log(1/\delta)}
{(1-2p)^2}
\right)$

%% file: appendix/malliscious.tex
\newpage
\section{Auxiliary Definitions and Lemmas}

\begin{theorem}[Distirbution-free realizable multiclass learning \small{\cite{pabbaraju2026optimal}}]
\label{thm:pabbaraju-passive}
Let $\cF_k$ be a multiclass hypothesis class of Daniely-Shwartz dimension $\shwartz$.
Then there exists a learning algorithm $\cA$ such that, a sample complexity of $\left(\frac{\shwartz+\log(1/\delta)}{\epsilon}\right)$ is sufficient to learn $\cF_k$ in the distirbution-free setting under realizability.
\end{theorem}

\begin{definition}[Extended Teaching Dimension]
\label{def:xtd}
Let $\cV\subseteq \mathcal Y^m$
be a finite hypothesis class over the finite domain $[m]=\{1,\ldots,m\}.$ For an arbitrary labeling $g\in\mathcal Y^m,$
a set $S\subseteq[m]$ is said to specify $g$ with respect to $\cV$ when
\[
\left|
\left\{
v\in \cV:
v|_S=g|_S
\right\}
\right|
\le 1.
\]

The extended teaching dimension of $V$ is
\[
\operatorname{XTD}(\cV)
\triangleq
\max_{g\in\mathcal Y^m}
\min
\left\{
|S|:
S\subseteq[m]
\text{ specifies }g\text{ with respect to } \cV
\right\}.
\]
\end{definition}

\begin{lemma}[Freedman's inequality \small{\cite{freedman1975tail, tropp2011freedman}} ]
\label{lem:freedman}
Suppose that $D_t\le b$ almost surely for every $t$. Then, for every $s\geq 0$ and $v>0$,
\[
\prob\left(
S_T\geq s,\;
V_T\leq v
\right)
\leq
\exp\left(
-\frac{s^2}{2(v+bs/3)}
\right).
\]
More generally, the original Freedman inequality is uniform over
time..
\end{lemma}

\begin{definition}[Minimum Test Cover \small{\citep{crowston2012parameterized,garey1990guide}}]\label{def:testcover}
Let $V = \{v_1,\ldots,v_n\}$ be a set of objects and let
$\cT = \{T_1,\ldots,T_m\}$ be a collection of subsets of $V$,
called tests. A test $T \in \mathcal{T}$ separates two distinct items
$v_i,v_j \in V$ if exactly one of them belongs to $T$, i.e., $|\{v_i,v_j\}\cap T| = 1.$
A subcollection $\mathcal{T}' \subseteq \mathcal{T}$ is called a
\emph{test cover} if every pair of distinct items in $V$ is separated
by at least one test in $\mathcal{T}'$.

The \emph{Minimum Test Cover} problem consists of finding a test cover
$\mathcal{T}'$ of minimum cardinality: $\min_{\mathcal{T}' \subseteq \mathcal{T}} |\mathcal{T}'|$
subject to
\[
    \forall\, v_i,v_j\in V,\ i\neq j,\quad
    \exists\,T\in\mathcal{T}' :
    |\{v_i,v_j\}\cap T|=1.
\]
\end{definition}

\begin{theorem}
The Minimum Test Cover problem is NP-hard.
\end{theorem}

%% file: appendix/experiments.tex
\section{Additional Experiments Results and Details}\label{secapp:exp}

\subsection{Fairness-aware Adversarial Regime}\label{app:robust}

The finite-pool implementation follows the neighborhood mechanism above,
but constructs the cells using a concrete target-independent pseudometric.

Let $G$ be a fixed reference panel of candidate probe hypotheses, selected
without knowledge of the deployed target. Define
\[
d_G(q,q')
:=
\frac1{|G|}
\sum_{f\in G}
\mathbf 1\{f(q)\neq f(q')\}
\]
For every candidate CGQ $q$, its cell $J(q)$ is fixed before any oracle
interaction and consists of the $256$ nearest distinct CGQs under $d_G$.
Neither the current version space nor observed target responses are used
to construct these cells.

For $p>0$, the implementation in the experimental section uses
\[
\beta=p+\rho-p\rho,
\qquad
\rho=0.05,
\qquad
\delta=0.05,
\]
and
$$R(p)
=
\left\lceil
\frac{2}{(1-2\beta)^2}
\log\frac{2N}{\delta}
\right\rceil,$$
where $N$ bounds the number of clean CGQ decisions allowed by budget?

When \textsc{Robust $k$-ALeBi} requests $q_t$, it also samples $R(p)$ distinct
CGQs from $J(q_t)$ without replacement and receives response $\widetilde Y_{t,1},\ldots,\widetilde Y_{t,R}.$ Then , $\widehat Y_t
=
\operatorname{mode}
\{
\widetilde Y_{t,1},\ldots,\widetilde Y_{t,R}
\},$
and uses this as the estimated response of $q_t$.

If the remaining exposure budget cannot pay for all $R(p)$ local queries,
 \textsc{Robust $k$-ALeBi} abstains rather than using a smaller block. When $p=0$,
the cell size is consummed, so \textsc{Robust $k$-ALeBi} coincides 
with  \textsc{ $k$-ALeBi}.

The implemented attacks use independent Bernoulli-$p$ attack,
which are a special case of Assumption~\ref{ass:attack}. When an attack
occurs, the experiment chooses an incorrect relational response that
increases the current absolute statistical-parity estimation error.

\subsection{Auditors Baselines}\label{app:baselines}

The supplied \cite{yan2022active} implementation \textsc{DiamAudit} performs a
search in the current linear version space for hypotheses with high and low statistical-parity values and queries a point on which they disagree. The Yan curves use the original version-space average statistical-parity
error computed by a hit-and-run evaluation protocol on the model class. For \textsc{$k$-ALeBi}, we use a Monte-Carlo approximation of the same linear-model family. A  CGQ contains one individual from each group and
therefore may cost up to two exposed individual predictions per coordinate $i \le k$. The horizontal axis measures these individual predictions exposed allowing model extraction attacks.

The original implementation of \cite{yan2022active} baseline is not run on German Credit. The supplied
implementation assumes a binary protected attribute, whereas our German
Credit setup contains four \texttt{status\_and\_sex} groups, for a total of $k=6$ groups. 

For the plotted comparison where $k=2$, \textsc{$k$-ALeBi} and \textsc{DiamAudit} are evaluated
using the same statistical-parity unfairness $\mu
=
\left|
P(h=1\mid A=1)
-
P(h=1\mid A=0)
\right|$ with absolute value estimation error $|\widehat\mu-\mu^\star|$. The original saved \cite{yan2022active} endpoint appears only in the validation table because it retains the evaluation convention of their implementation.

 All our computations are performed on an 11th Gen Intel® Core™ i7-1185G7 processor (3.00
GHz, 8 cores) with 32.0 GiB of RAM.